\documentclass[letterpaper]{article} 
\usepackage{aaai2026}  
\usepackage{times}  
\usepackage{helvet}  
\usepackage{courier}  
\usepackage[hyphens]{url}  
\usepackage{graphicx} 
\usepackage{natbib}  
\usepackage{caption} 
\usepackage{algorithm}
\usepackage{algorithmic}

\usepackage{amsfonts}
\usepackage{amsmath}
\usepackage{comment}
\usepackage{amssymb}
\usepackage{booktabs}
\usepackage{multirow}
\usepackage{dsfont}
\usepackage{subfigure}
\usepackage{xcolor}
\usepackage{amsthm}
\theoremstyle{plain}
\newtheorem{lemma}{Lemma}
\newtheorem{theorem}{Theorem}
\newtheorem{assumption}{Assumption}
\newtheorem{remark}{Remark}

\usepackage{newfloat}
\usepackage{listings}
\DeclareCaptionStyle{ruled}{labelfont=normalfont,labelsep=colon,strut=off} 
\floatstyle{ruled}
\newfloat{listing}{tb}{lst}{}
\floatname{listing}{Listing}
\title{Ordered Local Momentum \\ for Asynchronous Distributed Learning under Arbitrary Delays}
\author{
    Chang-Wei Shi\textsuperscript{\rm 1},
    Shi-Shang Wang\textsuperscript{\rm 1},
    Wu-Jun Li\textsuperscript{\rm 1,\rm 2}\thanks{Corresponding author.} \\
}
\affiliations{
\textsuperscript{\rm 1} National Key Laboratory for Novel Software Technology, School of Computer Science,
Nanjing University\\
\textsuperscript{\rm 2} Pazhou Laboratory (Huangpu), Guangdong, China \\
    \{shicw, wangs\}@smail.nju.edu.cn, liwujun@nju.edu.cn
}

\usepackage{bibentry}

\begin{document}

\def\0{{\bf 0}}
\def\w{{\bf w}}
\def\u{{\bf u}}
\def\v{{\bf v}}
\def\g{{\bf g}}
\def\h{{\bf h}}
\def\y{{\bf y}}
\def\RB{{\mathbb R}}
\def\EB{{\mathbb E}}
\def\DM{{\mathcal D}}

\maketitle

\begin{abstract}
Momentum SGD~(MSGD) serves as a foundational optimizer in training deep models due to momentum's key role in accelerating convergence and enhancing generalization.
Meanwhile, asynchronous distributed learning is crucial for training large-scale deep models, especially when the computing capabilities of the workers in the cluster are heterogeneous. 
To reduce communication frequency, local updates are widely adopted in distributed learning. However, how to implement asynchronous distributed MSGD with local updates remains unexplored.
To solve this problem, we propose a novel method, called \underline{or}dered \underline{lo}cal \underline{mo}mentum~(\mbox{OrLoMo}), for asynchronous distributed learning. 
In OrLoMo, each worker runs MSGD locally. Then the local momentum from each worker will be aggregated by the server in order based on its global iteration index. To the best of our knowledge, \mbox{OrLoMo} is the first method to implement asynchronous distributed MSGD with local updates. We prove the convergence of OrLoMo for non-convex problems under arbitrary delays. Experiments validate that OrLoMo can outperform its synchronous counterpart and other asynchronous methods.
\end{abstract}

\section{Introduction}
Stochastic gradient descent~(SGD)~\cite{Robbins&Monro:1951} and its variants \cite{article, kingma2014adam} are commonly used in training deep models. At each iteration, SGD uses a stochastic gradient as an approximation of the  full gradient to update the model parameter. In practice, momentum SGD~(MSGD)~\cite{article} is widely adopted by incorporating momentum into SGD to accelerate convergence and enhance generalization. Furthermore, momentum has been a core component in many popular adaptive optimizers, such as Adam~\cite{kingma2014adam} and \mbox{AMSgrad}~\cite{reddi2018convergence}.

As both models and datasets in deep learning continue to grow rapidly, the demand of training computation has surpassed the capacity of single devices. Distributed learning~\cite{yang2019survey, verbraeken2020survey} addresses this challenge by distributing computation across multiple workers. Existing distributed learning methods primarily fall into two categories: synchronous distributed learning~(SDL) methods and asynchronous distributed learning~(ADL) methods.
SDL methods~\cite{DBLP:conf/iclr/LinHM0D18, DBLP:conf/nips/VogelsKJ19,DBLP:conf/icml/YuJY19, DBLP:conf/iclr/WangTBR20} require all workers to synchronize at each communication round. Thus, SDL methods face the challenge of idling, as the fast workers must wait for the slow workers. This challenge is especially severe in clusters with heterogeneous computing capabilities.
ADL methods~\cite{DBLP:conf/nips/AgarwalD11, lian2018asynchronous, DBLP:conf/icml/YangL21, shiordered} mitigate this by allowing workers to proceed independently since only information from a subset of workers is required to be aggregated per round. 
Representative ADL methods include asynchronous SGD~(ASGD) and its variants~\cite{DBLP:conf/nips/AgarwalD11, DBLP:conf/nips/LianHLL15, DBLP:conf/icml/ZhengMWCYML17,  DBLP:conf/alt/ArjevaniSS20, pmlr-v267-maranjyan25b}. 
While ADL methods avoid idle time, they introduce convergence challenges due to the delay of updates.
Existing convergence analyses for ADL methods typically rely on additional assumptions about the delays~(e.g., maximum delay bounds), which can't fully capture the actual behavior of ADL methods in practice. The convergence proof for ADL methods under arbitrary delays was a long-standing open problem until recently. The works in \cite{DBLP:conf/nips/KoloskovaSJ22, DBLP:conf/nips/MishchenkoBEW22} establish a convergence guarantee for ASGD under arbitrary delays and prove that ASGD achieves the same convergence rate as synchronous SGD in terms of gradient computations.
The work in \cite{shiordered} proposes a momentum variant of ASGD called OrMo-DA, and provides a convergence proof for OrMo-DA under arbitrary delays.

To reduce communication frequency, local updates are widely adopted in distributed learning. The most representative methods using local updates are local SGD~\cite{lin2019don,stich2019local} and its variants \cite{DBLP:conf/nips/0001DKD19, DBLP:conf/icml/KarimireddyKMRS20, NEURIPS2024_0f986451}. Local SGD is also the key component of FedAvg in federated learning \cite{mcmahan2017communication}.
Local SGD enables each worker to execute multiple SGD iterations instead of one iteration before communicating with the server or the other workers, thereby reducing communication frequency.
Several works in \cite{lin2019don, stich2019local, DBLP:conf/nips/HaddadpourKMC19, DBLP:conf/aaai/YuYZ19} establish the convergence guarantee for local SGD. The work in \cite{DBLP:conf/icml/YuJY19} proposes parallel restarted SGD with momentum~(PRSGDm) as a momentum variant of local SGD and proves its convergence with a linear speedup property. 
Most existing works about local SGD and its variants focus on SDL methods, which face the challenge of idling while waiting for slow workers.
Several works \cite{xie2019asynchronous, DBLP:conf/aistats/NguyenMZYR0H22, DBLP:journals/tmc/WangLJZ25} have studied asynchronous local SGD~(AL-SGD), which implements asynchronous distributed SGD with local updates. While momentum is widely recognized as crucial for deep model training, no research has studied how to implement asynchronous distributed MSGD with local updates. Furthermore, the convergence guarantee of asynchronous distributed MSGD with local updates remains an open problem, particularly under arbitrary delays caused by the heterogeneous computing capabilities in the cluster.

In this paper, we propose a novel method, called \underline{or}dered \underline{lo}cal \underline{mo}mentum~(OrLoMo), for asynchronous distributed learning. Our contributions are outlined as follows:
\begin{itemize}
    \item In OrLoMo, each worker runs MSGD locally. Then the local momentum from each worker will be aggregated by the server in order based on its global iteration index. To the best of our knowledge, OrLoMo is the first method to implement asynchronous distributed MSGD with local updates.
    \item We prove the convergence of OrLoMo for non-convex problems under arbitrary delays. Our convergence analysis doesn't rely on the maximum delay or bounded gradient assumptions which are commonly used in existing analyses of ADL methods.
    \item Experiments validate that OrLoMo can outperform its synchronous counterpart and other asynchronous methods.

\end{itemize}
\section{Preliminary}
In this section, we introduce the problem formulation and then present two classic algorithms, including momentum SGD~(MSGD) and asynchronous local SGD~(AL-SGD).
\subsection{Problem Formulation}
Many machine learning model training tasks can be formulated as the following optimization problem:
\begin{align}\label{equation: object}
	\min_{\w\in \RB^d} F(\w) = \EB_{\xi \sim \DM} f(\w;\xi),
\end{align}
where $\w$ is the $d$-dimensional model parameter, $\xi$ denotes a training instance sampled from distribution $\DM$ and $f(\w;\xi)$ is the loss function evaluated at parameter $\w$ for instance $\xi$. $\nabla f(\w;\xi)$ denotes the stochastic gradient. $\|\cdot\|$ denotes the $L_2$ norm. $[n]$ denotes $\{0,1,\ldots,n-1\}$, where $n \in \mathbb{N}^*$.

In this paper, we focus on the widely used parameter server~(PS) framework~\cite{DBLP:conf/nips/LiASY14}. In the PS framework, the workers are responsible for sampling training instances and computing stochastic gradients, and the server is responsible for storing and updating the model parameters. Without loss of generality, we assume that each worker performs stochastic gradient computation by sampling a single training instance per iteration. The analysis of mini-batch sampling follows a similar approach.

\subsection{Momentum SGD}
The widely used momentum SGD~(MSGD)~\citep{article} can be formulated as follows:
\begin{equation}
\begin{aligned}
\u_{t+1} &= \beta \u_{t} + \gamma \g_t, \\
\w_{t+1} &= \w_t  - \u_{t+1}, \label{eq:mo-update-MSGD}
\end{aligned}
\end{equation}
where $\g_t = \nabla f \left(\w_t;\xi\right), \u_0 = \0, \beta\in [0,1), t \in [T]$ and $T$ denotes the number of iterations.
$\beta$ is the momentum coefficient. $\gamma$ is the learning rate.
 When $\beta = 0$, MSGD degenerates to vanilla SGD. $\u_t$ represents Polyak's momentum, which is a weighted sum of gradients in a strictly ordered manner. Specifically, gradients with smaller iteration indexes are assigned progressively smaller weights. For example, $\u_4 = \beta^3 \gamma \g_0 + \beta^2 \gamma \g_1 + \beta^1 \gamma \g_2 + \beta^0 \gamma \g_3$.

\subsection{Asynchronous Local SGD}
Algorithm~\ref{alg:AsynlocalSGD} outlines the asynchronous local SGD~(AL-SGD) method. AL-SGD serves as the foundational case of FedBuff \cite{DBLP:conf/aistats/NguyenMZYR0H22} without federated learning components, including the buffer, partial client participation, and differential privacy mechanism. 
Our work focuses on distributed learning in data-center settings.
When $S=1$, AL-SGD degenerates to  ASGD~\cite{DBLP:conf/nips/MishchenkoBEW22}.

\begin{algorithm}[!t]
    \caption{AL-SGD}
    \begin{algorithmic}[1]
 \label{alg:AsynlocalSGD}   
    \STATE \textbf{Input}: number of workers $K$, number of global iterations $T$, number of local iterations $S$, global learning rate $\eta_t$, local learning rate $\gamma$;
    \STATE \underline{\textbf{Server:}}  
    \STATE \textbf{Initialization}: global parameter $\w_0$;  
    \STATE Send $\w_{0}$ to all workers;
    \FOR {$t=0$  \textbf{to}  $T-1$}
    \STATE  Receive $\Delta \w_{ite (t,k_t)}^{k_t}$ from some worker $k_t$;
    \STATE $\w_{t+1} = \w_{t}   - \eta_t {\Delta \w}_{ite (t,k_t)}^{k_t}$;
    \STATE Send $\w_{t+1}$ back to worker $k_t$;
    \ENDFOR
    \STATE Notify all workers to stop;
    \STATE \underline{\textbf{Worker $k:$}} $(k \in [K])$  
    \REPEAT
    \STATE Wait until receiving $\w_{t'}$ from the server;
    \STATE Initialize local variables $\tilde{\w}_{t',0}^k=\w_{t'}$;  
    \FOR {$s=0$  \textbf{to}  $S-1$}
    \STATE Randomly sample $\xi \sim \DM$ and then compute the stochastic gradient $\g_{t',s}^{k} = \nabla f(\tilde{\w}_{t',s}^{k}; \xi)$;
    \STATE $\tilde{\w}_{t',s+1}^{k} = \tilde{\w}_{t',s}^{k} - \gamma \g_{t',s}^{k}$;   
    \ENDFOR  
    \STATE $\Delta \w_{t'}^k = \tilde{\w}_{t',0}^k- \tilde{\w}_{t',S}^k$;    
    \STATE Send $\Delta \w_{t'}^k$ to the server;  
    \UNTIL{receive server's notification to stop}
    \end{algorithmic}
    \end{algorithm}

For each worker $k$ ($k \in [K]$), after receiving the latest global parameter from the server, the worker conducts $S$ SGD iterations locally, and then sends the local parameter update $\Delta \w_{t'}^k$ to the server. 
For the server, after receiving the local parameter update ${\Delta \w}_{ite (t,k_t)}^{k_t}$ from some worker $k_t$, the server will update the global parameter using ${\Delta \w}_{ite (t,k_t)}^{k_t}$ and then send the latest global parameter $\w_{t+1}$ back to the worker $k_t$ immediately. 
To distinguish them from worker-side local iterations, the server-side iterations are termed global iterations.
Here, $k_t$ denotes the index of the worker from which the local parameter update is used at global iteration $t$, and $ite(t, k_t)$ denotes the global iteration index of ${\Delta \w}_{ite (t,k_t)}^{k_t}$. The function $ite(t, k)$ denotes the global iteration index of the latest parameter sent to worker $k$ before global iteration $t$. $ite(t,k)$ can be formulated as follows:
\begin{align*}
ite(t+1,k)= \left\{
\begin{aligned}
 &t+1   &\ k=k_t, \\
 &ite(t,k)  &\ k \neq k_t, \\ 
\end{aligned}
\right.
\end{align*}
where $k \in [K]$, $t \in [T]$ and $ite(0,k)=0, \forall k \in [K]$. We define $\tau_t = t - ite(t,k_t)$ as the delay of  ${\Delta \w}_{ite (t,k_t)}^{k_t}$.

In AL-SGD, upon receiving a local parameter update from any worker, the server immediately updates the global parameter and then sends it back to the worker. This asynchronous nature avoids straggler bottlenecks by not waiting for slower workers. In contrast, synchronous local SGD \cite{stich2019local} requires the server to aggregate the information from all workers before performing global parameter updates and broadcasting. Consequently, its training speed is constrained by the slowest worker in the cluster.

\section{Methodology}
In this section, we introduce our proposed method OrLoMo, including the algorithm details and convergence analysis.  
\subsection{OrLoMo}
The details of OrLoMo are presented in Algorithm~\ref{alg:OrLoMo}.

For each worker, upon receiving the latest global parameter $\w_{t'}$ from the server, the local parameter $\tilde{\w}_{t',0}^k$ and local momentum $\tilde{\u}_{t',0}^k$ are initialized as $\tilde{\w}_{t',0}^k=\w_{t'}, \tilde{\u}_{t',0}^k=\0.$ Then, the worker performs $S$ MSGD updates locally.
After completing $S$ local updates, the worker sends both local momentum $\Delta \u_{t'}^k = \tilde{\u}_{t',S}^k$ and local parameter update $\Delta \w_{t'}^k = \tilde{\w}_{t',0}^k- \tilde{\w}_{t',S}^k$ back to the server.  Here, $t'$ denotes the global iteration index of both $\Delta \u_{t'}^k$ and $\Delta \w_{t'}^k$.

\begin{algorithm}[!t]
    \caption{OrLoMo}
    \begin{algorithmic}[1]
 \label{alg:OrLoMo}   
    \STATE \textbf{Input}: number of workers $K$, number of global iterations $T$, number of local iterations $S$, global learning rate $\eta_t$, local learning rate $\gamma$, momentum coefficient $\beta \in [0,1)$;
    \STATE \underline{\textbf{Server:}}  
    \STATE \textbf{Initialization}: global parameter $\w_0$, global momentum $\u_0 = \0$;  
    \STATE Send $\w_{0}$ to all workers;
    \FOR {$t=0$  \textbf{to}  $T-1$}
    \IF{$\lceil \frac{t}{K}\rceil > \lceil \frac{t-1}{K}\rceil$}
    \STATE $\w_{t+\frac{1}{2}} = \w_{t}-\beta \u_{t}$, $\u_{t+\frac{1}{2}}=\beta \u_t$;
    \ELSE 
    \STATE $\w_{t+\frac{1}{2}} = \w_t$, $\u_{t+\frac{1}{2}}= \u_t$; 
    \ENDIF    
    \STATE  Receive $\Delta \u_{ite (t,k_t)}^{k_t}$ and $\Delta \w_{ite (t,k_t)}^{k_t}$ from some worker $k_t$;
    \STATE $\u_{t+1} = \u_{t+\frac{1}{2}} + \beta^{\lceil \frac{t}{K}\rceil - \lceil \frac{ite (t,k_t)}{K}\rceil}\eta_t {\Delta \u}_{ite (t,k_t)}^{k_t}$;
    \STATE $\w_{t+1} = \w_{t+\frac{1}{2}}   - \eta_t{\Delta \w}_{ite (t,k_t)}^{k_t}- \frac{\beta (1-\beta^{\lceil \frac{t}{K}\rceil - \lceil \frac{ite (t,k_t)}{K}\rceil})}{1-\beta} \eta_t{\Delta \u}_{ite (t,k_t)}^{k_t}$;
    \STATE Send $\w_{t+1}$ back to worker $k_t$;
    \ENDFOR
    \STATE Notify all workers to stop;
    \STATE \underline{\textbf{Worker $k:$}} $(k \in [K])$  
    \REPEAT
    \STATE Wait until receiving $\w_{t'}$ from the server;
    \STATE Initialize local variables $\tilde{\u}_{t',0}^k=\0, \tilde{\w}_{t',0}^k=\w_{t'}$;  
    \FOR {$s=0$  \textbf{to}  $S-1$}
    \STATE Randomly sample $\xi \sim \DM$ and then compute the stochastic gradient $\g_{t',s}^{k} = \nabla f(\tilde{\w}_{t',s}^{k}; \xi)$;
    \STATE $\tilde{\u}_{t',s+1}^k = \beta\tilde{\u}_{t',s}^{k} + \gamma \g_{t',s}^{k}$;
    \STATE $\tilde{\w}_{t',s+1}^{k} = \tilde{\w}_{t',s}^{k} - \tilde{\u}_{t',s+1}^k$;   
    \ENDFOR  
    \STATE $\Delta \u_{t'}^k = \tilde{\u}_{t',S}^k$;    
    \STATE $\Delta \w_{t'}^k = \tilde{\w}_{t',0}^k- \tilde{\w}_{t',S}^k$;    
    \STATE Send $\Delta \u_{t'}^k$ and $\Delta \w_{t'}^k$ to the server;  
    \UNTIL{receive server's notification to stop}
    \end{algorithmic}
    \end{algorithm}

All the local momentums computed in OrLoMo can be given by: 
\begin{align*}
\Delta \u_0^0,\Delta \u_0^1, \cdots, \Delta \u_0^{K-1}, \Delta \u_1^{k_0},\Delta \u_2^{k_1},\Delta \u_3^{k_2}, \cdots.
\end{align*}
We group these local momentums in order based on their global iteration indexes. Each group contains $K$ local momentums. The $i$-th group ($i \geq 1$) in OrLoMo is defined as:
\begin{align*}
&\left\{\Delta \u_{(i-1)K+1}^{k_{(i-1)K}},\Delta \u_{(i-1)K+2}^{k_{(i-1)K+1}}, \cdots, \Delta \u_{iK}^{k_{iK-1}}\right\},
\end{align*}
and the $0$-th group is defined as $\left\{\Delta \u_{0}^{0},\Delta \u_{0}^{1}, \cdots, \Delta \u_{0}^{K-1}\right\}$. Note that local momentums typically arrive at the server out of order due to hardware factors~(e.g., workers’ computing capabilities). The core of OrLoMo is that the server aggregates local momentums into the global momentum in order based on their global iteration indexes.

The global momentum in OrLoMo is the weighted sum of local momentums received by the server, scaled by their corresponding global learning rates. For simplicity, we refer to local momentum scaled by its global learning rate as scaled local momentum. The weights of scaled local momentums in the global momentum are determined by their global iteration indexes.
Before global iteration $t$, the global parameter $\w_t$ has been sent to some worker. Thus, the global iteration index of the latest local momentum in $\u_{t+1}$ can be $t$ at most. The weight of the scaled local momentum with global iteration index $t$ is defined to be $\beta^0$ in $\u_{t+1}$, which belongs to the $\lceil \frac{t}{K} \rceil$-th group. The scaled local momentums from the $i$-th group is defined to weight $\beta^{\lceil \frac{t}{K} \rceil-i}$ in $\u_{t+1}$, where $0 \leq i \leq \lceil \frac{t}{K} \rceil$. The scaled local momentums from one group share the same weight in the global momentum. 

At global iteration $t$, the server receives local information from some worker $k_t$, including the local momentum $\Delta \u_{ite(t,k_t)}^{k_t}$ and local parameter update $\Delta \w_{ite(t,k_t)}^{k_t}$. 
We explain the update rules on the server at iteration $t$.
\begin{itemize}
\item $\w_{t+\frac{1}{2}} = \w_{t}-\beta \u_{t}$, $\u_{t+\frac{1}{2}}=\beta \u_t$, when $\lceil \frac{t}{K} \rceil > \lceil \frac{t-1}{K} \rceil$.

Since the global parameter $\w_t$ has been sent to worker $k_{t-1}$, $\Delta \u_{t}^{k_{t-1}}$ may arrive at the server at iteration $t$. If $\lceil \frac{t}{K} \rceil > \lceil \frac{t-1}{K} \rceil$, the global momentum is multiplied by $\beta$ and used to update the global parameter. In this way, the global momentum gets ready to accommodate $\Delta \u_{t}^{k_{t-1}}$, which belongs to the new~($\lceil \frac{t}{K} \rceil$-th) group. 

\item Update the global momentum: $\u_{t+1} = \u_{t+\frac{1}{2}} + \beta^{\lceil \frac{t}{K}\rceil - \lceil \frac{ite (t,k_t)}{K}\rceil}\eta_t{\Delta \u}_{ite (t,k_t)}^{k_t}$.

Since the scaled local momentum from the $\lceil \frac{t}{K} \rceil$-th group weights $\beta^0$ in the global momentum, the weight of the scaled local momentum $\eta_t{\Delta \u}_{ite (t,k_t)}^{k_t}$ (belonging to the $\lceil \frac{ite (t,k_t)}{K}\rceil$-th group) should be $\beta^{\lceil \frac{t}{K} \rceil - \lceil \frac{ite (t,k_t)}{K}\rceil}$. OrLoMo updates the global momentum by adding $\eta_t \Delta{\u}^{k_t}_{ite (t,k_t)}$ with this weight. 

\item Update the global parameter: $\w_{t+1} = \w_{t+\frac{1}{2}}   - \eta_t{\Delta \w}_{ite (t,k_t)}^{k_t}- \frac{\beta (1-\beta^{\lceil \frac{t}{K}\rceil - \lceil \frac{ite (t,k_t)}{K}\rceil})}{1-\beta} \eta_t{\Delta \u}_{ite (t,k_t)}^{k_t}$. 

We analyze the contributions of the newly arrived
${\Delta \u}_{ite (t,k_t)}^{k_t}$ and ${\Delta \w}_{ite (t,k_t)}^{k_t}$ to the global parameter update, separately.
\begin{itemize}
    \item ${\Delta \w}_{ite (t,k_t)}^{k_t}$ is scaled by the global learning rate $\eta_t$ and directly applied to update the global parameter. 
    \item Local momentums from the same group as ${\Delta \u}_{ite (t,k_t)}^{k_t}$ (received by the server at previous iterations) have contributed to the global parameter update with a total coefficient $-(\beta^1+\beta^2+ \cdots+\beta^{\lceil \frac{t}{K} \rceil - \lceil \frac{ite(t, k_t)}{K}\rceil})=-\frac{\beta (1-\beta^{\lceil \frac{t}{K}\rceil - \lceil \frac{ite (t,k_t)}{K}\rceil})}{1-\beta}$, scaled by their corresponding global learning rates. To ensure consistency, the coefficient of ${\Delta \u}_{ite (t,k_t)}^{k_t}$ for the global parameter update is set as $-\frac{\beta (1-\beta^{\lceil \frac{t}{K}\rceil - \lceil \frac{ite (t,k_t)}{K}\rceil})}{1-\beta}$, scaled by its corresponding global learning rate $\eta_t$.
\end{itemize}
\end{itemize}
\begin{figure}[!t] 
  \centering
    \subfigure[The global momentum $\u_{t+1}$. ]{
    \begin{minipage}[b]{1\linewidth} \label{fig:moment1}
      \includegraphics[width=1\linewidth]{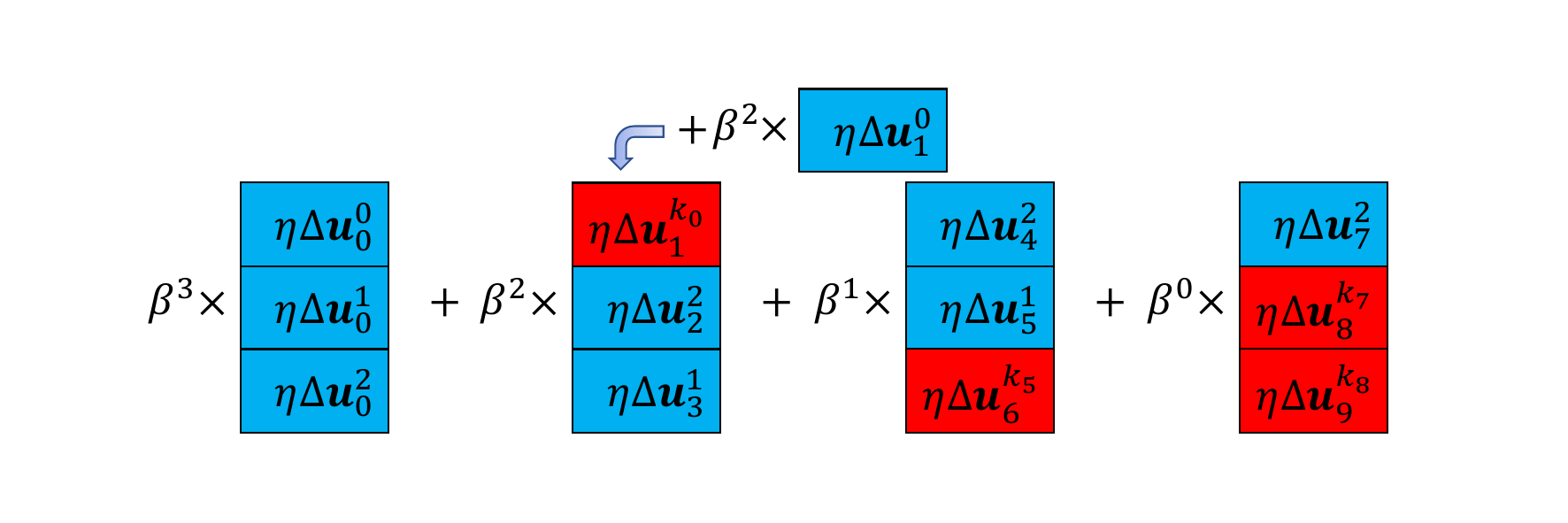}
      \end{minipage}}
  \subfigure[$\w_{t+1}-\w_0$ contributed by the local momentums.]{
    \begin{minipage}[b]{1\linewidth}\label{fig:moment2}
      \includegraphics[width=1\linewidth]{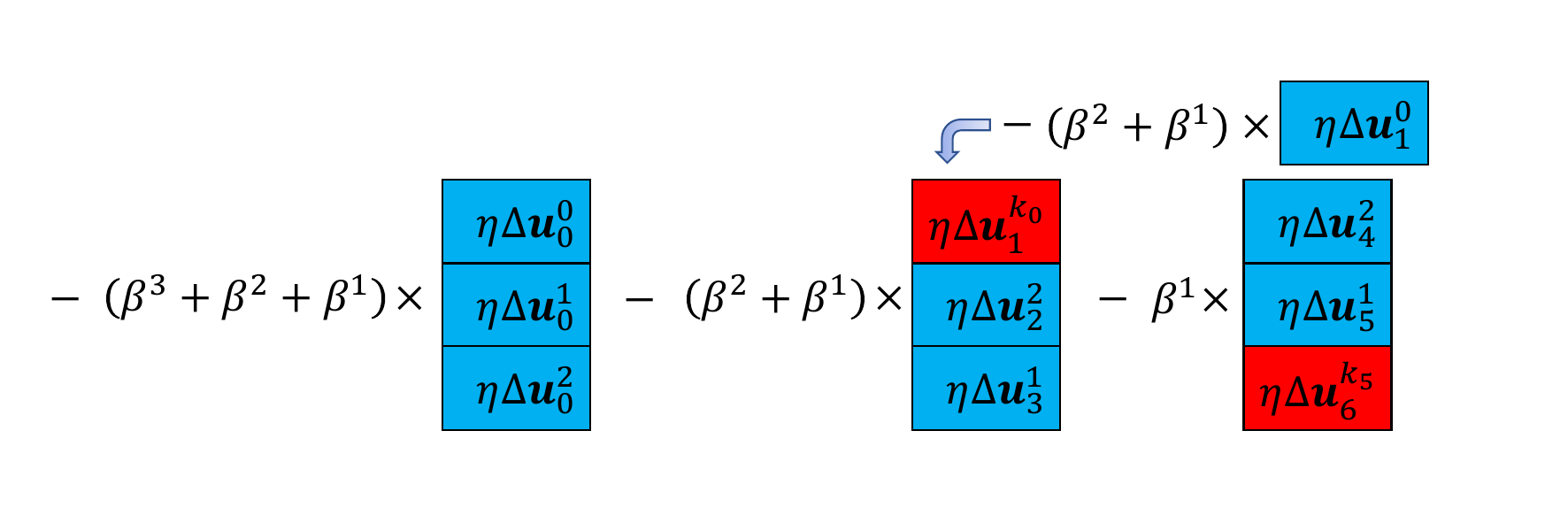}
      \end{minipage}}
  \subfigure[$\w_{t+1}-\w_0$ contributed by the local parameter updates.]{
    \begin{minipage}[b]{1\linewidth}\label{fig:moment3}
      \includegraphics[width=1\linewidth]{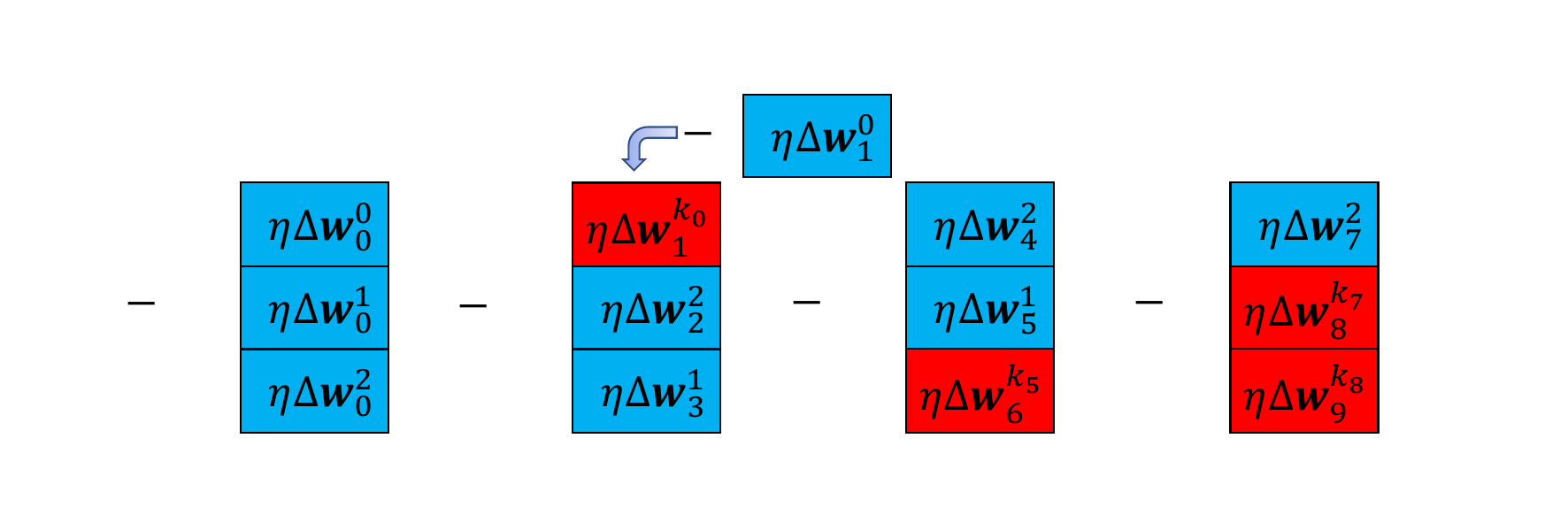}
      \end{minipage}}
\caption{Update demonstration when $t=8, K=3$.} \label{fig: workflow}
\end{figure}   

We use an example shown in Figure \ref{fig: workflow} to illustrate the update process of OrLoMo at iteration $8$ when $K=3$. For simplicity, we consider the global learning rate as constant here, i.e., $\eta_t \equiv \eta$. The local momentums or local parameter updates shown in red indicate those that have not arrived at the server. In this case, $k_0 = 0, k_2=k_4=k_5=1$ and $ k_1=k_3=k_6=k_7=2$. The global iteration indexes of the server's received information at global iterations $0, 1, \cdots, 7$ are 0, 0, 0, 2, 3, 5, 4, 7.
At iteration 8, the local momentum $\Delta \u_1^0$ and the local parameter update $\Delta \w_1^0$ from worker $0$ arrive at the server, i.e., $k_8=0$ and $ite(8,k_8)=1$. The scaled local momentum $\eta \Delta \u_1^0$ is added into the global momentum with the weight of $\beta^{\lceil \frac{8}{3} \rceil-\lceil \frac{1}{3} \rceil}=\beta^2$ as shown in Figure \ref{fig:moment1}. Then we get that $\u_9=\beta^3\times(\eta\Delta \u_0^0+\eta\Delta \u_0^1+\eta\Delta \u_0^2)+\beta^2\times(\eta\Delta \u_1^0+\eta\Delta \u_2^2+\eta\Delta \u_3^1)+\beta^1\times(\eta\Delta \u_4^2+\eta\Delta \u_5^1)+\beta^0\times\eta\Delta \u_7^2$.
The local momentum ${\Delta \u}_1^0$ belongs to the same group as $\Delta \u_2^2$, which was received by the server at iteration 3. $\eta \Delta \u_2^2$ has contributed to the parameter update with the coefficient $-(\beta^2+\beta^1)$. Thus, we set the coefficient of $\eta\Delta \u_1^0$ as $-(\beta^2+\beta^1)$ when updating the global parameter as shown in Figure \ref{fig:moment2}.
The local parameter update ${\Delta \w}_1^0$ is scaled by the learning rate $\eta$ and then directly used to update the global parameter as shown in Figure \ref{fig:moment3}. 

\begin{remark}
When $\beta=0$, OrLoMo degenerates to AL-SGD in Algorithm \ref{alg:AsynlocalSGD}. 
Compared with AL-SGD, OrLoMo requires workers to send not only local parameter updates but also local momentums to the server. This inclusion of local momentums is typical in synchronous local momentum methods, e.g., \cite{DBLP:conf/icml/YuJY19}. Since OrLoMo operates with low communication frequency, the extra overhead doesn't significantly impact the training speed, which is validated in our experiments.
\end{remark}
\begin{remark}
When $S=1$, OrLoMo degenerates to OrMo-DA in \cite{shiordered}.
In OrLoMo, each worker runs $S$ MSGD steps locally.
For comparison, we implement a baseline method named local OrMo-DA in the experiments, where workers execute $S$ SGD steps locally instead. 
Experimental results confirm that local momentum in OrLoMo significantly improves convergence performance, especially when the number of local iterations $S$ is large.
\end{remark}
\subsection{Convergence Analysis}
In this subsection, we prove the convergence of OrLoMo  for non-convex problems under arbitrary delays. The proof details are included in the appendix. 

We adopt the following standard assumptions in distributed learning. Notably, our convergence analysis doesn't require any assumptions about the delays, e.g., the maximum delay $\tau_t = t- ite(t, k_t) \leq \tau_{max}, \forall t \in [T]$, making our results valid for arbitrary delays. Furthermore, unlike existing ADL methods \cite{DBLP:conf/icml/YangL21, wang2024fadas}, we do not require the bounded gradient assumption, i.e., $\EB_{\xi \sim \DM} \|\nabla f(\w; \xi)\|^2 \leq G^2, \forall \w \in \RB^d$.

  \begin{assumption}\label{assu:gradient}
  The stochastic gradient $\nabla f(\w; \xi)$ has unbiased expectation and bounded variance for $\forall \w \in \RB^d$:
    \begin{align*}
     &\EB_{\xi \sim \DM} \left[\nabla f(\w; \xi)\right] = \nabla F(\w),\\
     &\EB_{\xi \sim \DM}\|\nabla f(\w; \xi)-\nabla F(\w)\|^2 \leq \sigma^2.  
    \end{align*}
    \end{assumption}
    Assumption \ref{assu:gradient} indicates that the data across all workers are independent and identically distributed~(i.i.d.). It's commonly satisfied in the data-center scenario~\cite{DBLP:conf/nips/DeanCMCDLMRSTYN12}, where all workers have access to the full dataset.  
  \begin{assumption}\label{assu:smooth}
     $F(\w)$ is $L$-smooth~($L>0$): 
     \begin{align*}
     F(\w) \leq F(\w') &+ \nabla F(\w')^T(\w - \w') + \frac{L}{2}\|\w - \w'\|^2,
     \end{align*} 
     holds for $\forall \w,\w'\in \RB^d$.
  \end{assumption}
  \begin{assumption}\label{assu:lowerbounded}
     $F(\w)\geq F^*, \forall \w \in \mathbb{R}^d$.
  \end{assumption}

Firstly, we define two notations $next(t,k)$ and $\hat{\eta}_{t,k}$. 
$next(t,k)$ denotes the index of the next global iteration at which the server will receive the information from worker $k$ after global iteration $t$ (including $t$). $\hat{\eta}_{t,k}$ is the global learning rate at iteration $next(t,k)$. They are formulated as:
$$next(t,k) = \min\{j \geq t: k_j = k\},$$ 
$$\hat{\eta}_{t,k} =  \eta_{next(t,k)},$$
where $k \in [K]$, $t \in [T]$.

Secondly, we introduce two auxiliary variables $\hat{\u}_t$ and $\hat{\w}_t $ corresponding to $\u_t$ and $\w_t$ respectively, where $t \geq 1$: 
\begin{align*} 
& \hat{\u}_{t+1} = \left\{
\begin{aligned}
& \beta \hat{\u}_{t}+\hat{\eta}_{t,k_{t-1}}{\Delta \u}_{t}^{k_{t-1}} & K \mid \left(t-1\right), \\
& \hat{\u}_{t}+\hat{\eta}_{t,k_{t-1}}{\Delta \u}_{t}^{k_{t-1}} & K \nmid \left(t-1\right),
\end{aligned}
\right. 
\end{align*} for $t \geq 1$ and $\hat{\u}_{1} =  \sum_{k \in [K]}  \hat{\eta}_{0,k}{\Delta \u}_{0}^{k}$, 
\begin{align*} 
& \hat{\w}_{t+1} = \left\{
\begin{aligned}
& \hat{\w}_{t}-\beta \hat{\u}_{t}- {\hat{\eta}}_{t,k_{t-1}}{\Delta \w}_t^{k_{t-1}}& K \mid \left(t-1\right), \\
& \hat{\w}_{t}- {\hat{\eta}}_{t,k_{t-1}}{\Delta \w}_t^{k_{t-1}} & K \nmid \left(t-1\right),
\end{aligned}
\right. 
\end{align*} for $t \geq 1$ and $\hat{\w}_{1} = \w_0 - \sum_{k \in [K]} {\hat{\eta}}_{0,k} {\Delta \w}_0^k$.

For simplicity, we define $\Delta \h_0^k =\gamma \sum_{s=0}^{S-1}\g_{0,s}^{k}$ for $k \in [K]$, and
$\Delta \h_{t}^{k_{t-1}}=\gamma\sum_{s=0}^{S-1}\g_{t,s}^{k_{t-1}}$ for $t \geq 1$. 
It's easy to verify that: $\frac{1}{1-\beta}\Delta \h_0^k = \frac{\beta}{1-\beta}{\Delta \u}_{0}^{k}+{\Delta \w}_0^k$ for $ k \in [K],$ and 
$\frac{1}{1-\beta}\Delta \h_{t}^{k_{t-1}} = \frac{\beta}{1-\beta}{\Delta \u}_{t}^{k_{t-1}}+{\Delta \w}_t^{k_{t-1}}$ for $t \geq 1$.

Then, we define another auxiliary variable $\hat{\y}_t$ corresponding to $\hat{\w}_t$: $\hat{\y}_1 = \w_0-\frac{1}{1-\beta}\sum_{k \in [K]} \hat{\eta}_{0,k}\Delta \h_0^k$, and  
\begin{align} \label{def:yhat}
\hat{\y}_{t+1} = \hat{\y}_t - \frac{\hat{\eta}_{t,k_{t-1}}}{1-\beta}\Delta \h_{t}^{k_{t-1}}, t \geq 1.
\end{align}

Our proof begins with rigorously tracking the differences between these variables in Lemmas~\ref{local-ormo-lemma:ugap}, \ref{local-ormo-lemma:wgap}, and \ref{lemma:ywgapPenalty}.
\begin{lemma}~\label{local-ormo-lemma:ugap}
  For any $t \geq 0$, the difference between $\u_{t+1}$ and $\hat{\u}_{t+1}$ can be expressed as follows:
  \begin{equation}
  \begin{aligned}\label{local-ormo-eq:ugap}
    & \hat{\u}_{t+1} - \u_{t+1} \\
    & =  \sum_{k \in [K], k \neq k_t} \beta^{\lceil \frac{t}{K} \rceil - \lceil \frac{ite(t, k)}{K}\rceil} \hat{\eta}_{ite(t, k),k}{\Delta \u}_{ite(t, k)}^k.
  \end{aligned}
  \end{equation}
\end{lemma}

\begin{lemma}~\label{local-ormo-lemma:wgap}
  For any $t \geq 0$, the difference between $\w_{t+1}$ and $\hat{\w}_{t+1}$ can be expressed as follows:
  \begin{equation}
  \begin{aligned} \label{local-ormo-eq:wgapPenalty}
    \hat{\w}_{t+1} &- \w_{t+1}=  -\sum_{k \in [K], k \neq k_t} \bigg[{\hat{\eta}}_{ite(t, k),k}{\Delta \w}_{ite(t, k)}^k \bigg.\\
    &\bigg.+ \frac{\beta(1-\beta^{\lceil \frac{t}{K} \rceil - \lceil \frac{ite(t, k)}{K}\rceil})}{1-\beta}{\hat{\eta}}_{ite(t, k),k}{\Delta \u}_{ite(t, k)}^k\bigg].
  \end{aligned}
  \end{equation}
\end{lemma}

\begin{lemma} \label{lemma:ywgapPenalty}
  For any $t \geq 1$, the difference between $\hat{\y}_t$ and $\hat{\w}_t$ can be expressed as:
    $\hat{\y}_t - \hat{\w}_t = -\frac{\beta}{1-\beta}\hat{\u}_t.$
\end{lemma}

\begin{lemma}   With Assumption~\ref{assu:gradient}, the gap between the local and global parameters can be bounded for $t \geq 1$:
      \begin{align*}
      \sum_{s=0}^{S-1}\mathbb{E}\left\|\tilde{\w}_{t,s}^{k_{t-1}} - \w_t\right\|^2 &\leq \frac{\gamma^2S(S-1)\sigma^2}{(1-\beta)^2}\\
      +&\frac{\gamma^2S(S-1)}{(1-\beta)^2}
      \sum_{s=0}^{S-1}\mathbb{E}\left\| \nabla F(\tilde{\w}_{t,s}^{k_{t-1}})\right\|^2.
      \end{align*}
      \end{lemma}

\begin{theorem} \label{thm:OrLoMo}
  With Assumptions~\ref{assu:gradient},~\ref{assu:smooth} and~\ref{assu:lowerbounded}, letting $16LS\gamma \leq (1-\beta)^2$ and
\begin{equation} \label{OrLoMo-penalty} 
\eta_t = 
\begin{cases} 
\frac{1}{K} & \tau_t \leq 2K, \\ 
\frac{1}{\tau_t} & \tau_t > 2K, 
\end{cases} 
\end{equation}
OrLoMo in Algorithm \ref{alg:OrLoMo} has the following convergence rate:
\begin{align*}
\mathbb{E}\left\|\nabla F(\bar{\w}_T)\right\|^2 & \leq \frac{4K(1-\beta)\left(F(\w_0)- F^*\right)}{\gamma ST} \\
&~~~~+\frac{4\gamma L\sigma^2}{K(1-\beta)^2}
+\frac{\gamma^2 L^2(S-1)\sigma^2}{(1-\beta)^2}, 
  \end{align*}
where $ \mathbb{E}\left\|\nabla F(\bar{\w}_T)\right\|^2 =  \frac{\sum_{k \in [K]}\hat{\eta}_{0,k}\left\|\nabla F(\w_0)\right\|^2}{\sum_{k \in [K]}\hat{\eta}_{0,k}+\sum_{t=1}^{T-1}\hat{\eta}_{t,k_{t-1}}}+\frac{\sum_{t=1}^{T-1}\hat{\eta}_{t,k_{t-1}}\mathbb{E}\left\|\nabla F(\w_t)\right\|^2}{\sum_{k \in [K]}\hat{\eta}_{0,k}+\sum_{t=1}^{T-1}\hat{\eta}_{t,k_{t-1}}}$.
\end{theorem}
    \begin{table*}[!t] \small
  \centering
    \setlength{\tabcolsep}{1mm}{  
  \begin{tabular}{c|c|c|cccc|cccc}
    \toprule   
      \multicolumn{3}{c|}{Datasets} & \multicolumn{4}{c|}{CIFAR10}& \multicolumn{4}{c}{CIFAR100}  \\ 
    \midrule
     & Workers & Local Iterations & PRSGDm & AL-SGD &  local OrMo-DA & OrLoMo & PRSGDm & AL-SGD &  local OrMo-DA & OrLoMo \\
    \midrule
     \multirow{6}{*}{\rotatebox{90}{homogeneous}}&\multirow{3}{*}{8}& 8 & \textbf{92.18{\scriptsize$\pm$0.10}} & 91.46{\scriptsize$\pm$0.33} & 92.02{\scriptsize$\pm$0.23} & 92.11{\scriptsize$\pm$0.08} & \textbf{69.89{\scriptsize$\pm$0.08}} & 67.46{\scriptsize$\pm$0.48} & 69.04{\scriptsize$\pm$0.41} & 69.43{\scriptsize$\pm$0.23}  \\ 
     & & 16 & 91.67{\scriptsize$\pm$0.14} & 91.01{\scriptsize$\pm$0.09} & 91.36{\scriptsize$\pm$0.20} & \textbf{92.11{\scriptsize$\pm$0.26}} & 69.09{\scriptsize$\pm$0.31} & 67.05{\scriptsize$\pm$0.32} & 68.29{\scriptsize$\pm$0.46} & \textbf{69.49{\scriptsize$\pm$0.23}}  \\
     & & 32 & 91.37{\scriptsize$\pm$0.19} & 90.82{\scriptsize$\pm$0.16} & 90.02{\scriptsize$\pm$0.09} & \textbf{91.85{\scriptsize$\pm$0.12}} & 68.29{\scriptsize$\pm$0.24} & 67.17{\scriptsize$\pm$0.33} & 65.15{\scriptsize$\pm$0.20} & \textbf{68.95{\scriptsize$\pm$0.12}}  \\ \cmidrule(r){2-11} 
     &\multirow{3}{*}{16}& 8 & \textbf{92.00{\scriptsize$\pm$0.12}} & 90.26{\scriptsize$\pm$0.16} & 91.14{\scriptsize$\pm$0.07} & 91.39{\scriptsize$\pm$0.09} & \textbf{69.03{\scriptsize$\pm$0.12}}  & 64.67{\scriptsize$\pm$0.36} & 67.30{\scriptsize$\pm$0.48} & 68.92{\scriptsize$\pm$0.36}  \\
     & & 16 & \textbf{91.29{\scriptsize$\pm$0.13}} & 89.85{\scriptsize$\pm$0.28} & 89.08{\scriptsize$\pm$0.43} & 91.27{\scriptsize$\pm$0.18} & 66.93{\scriptsize$\pm$0.18} & 64.06{\scriptsize$\pm$0.25} & 64.94{\scriptsize$\pm$0.26} & \textbf{68.20{\scriptsize$\pm$0.60}}  \\
     & & 32 & 90.10{\scriptsize$\pm$0.02} & 89.21{\scriptsize$\pm$0.20} & 77.96{\scriptsize$\pm$0.87} & \textbf{91.05{\scriptsize$\pm$0.09}} & 64.78{\scriptsize$\pm$0.37} & 62.15{\scriptsize$\pm$0.71} & 49.09{\scriptsize$\pm$0.60} & \textbf{65.65{\scriptsize$\pm$0.25}}  \\ 
    \midrule 
     \multirow{6}{*}{\rotatebox{90}{heterogeneous}}&\multirow{3}{*}{8}& 8 & 92.06{\scriptsize$\pm$0.13} & 91.13{\scriptsize$\pm$0.19} & 92.18{\scriptsize$\pm$0.13} & \textbf{92.34{\scriptsize$\pm$0.10}} & 69.60{\scriptsize$\pm$0.23} & 67.97{\scriptsize$\pm$0.40} & \textbf{69.73{\scriptsize$\pm$0.29}} & 69.44{\scriptsize$\pm$0.30}  \\
     && 16 & 91.66{\scriptsize$\pm$0.13} & 91.15{\scriptsize$\pm$0.14} & 91.64{\scriptsize$\pm$0.07} & \textbf{91.76{\scriptsize$\pm$0.24}} & 68.57{\scriptsize$\pm$0.16} & 67.38{\scriptsize$\pm$0.23} & 68.62{\scriptsize$\pm$0.31} & \textbf{69.25{\scriptsize$\pm$0.21}}  \\
 && 32 & 91.61{\scriptsize$\pm$0.35} & 90.88{\scriptsize$\pm$0.23} & 90.83{\scriptsize$\pm$0.11} & \textbf{91.96{\scriptsize$\pm$0.12}} & 68.40{\scriptsize$\pm$0.30} & 67.06{\scriptsize$\pm$0.43} & 67.09{\scriptsize$\pm$0.44} & \textbf{69.25{\scriptsize$\pm$0.19}}  \\
    \cmidrule(r){2-11}
     &\multirow{3}{*}{16}& 8 & \textbf{91.88{\scriptsize$\pm$0.23}} & 90.02{\scriptsize$\pm$0.03} & 91.45{\scriptsize$\pm$0.14} & 91.61{\scriptsize$\pm$0.20} & \textbf{69.29{\scriptsize$\pm$0.18}} & 64.83{\scriptsize$\pm$0.30} & 67.79{\scriptsize$\pm$0.21} & 68.65{\scriptsize$\pm$0.34}  \\
     && 16 & 91.08{\scriptsize$\pm$0.05} & 89.85{\scriptsize$\pm$0.13} & 90.23{\scriptsize$\pm$0.14} & \textbf{91.50{\scriptsize$\pm$0.02}} & 67.34{\scriptsize$\pm$0.18} & 64.20{\scriptsize$\pm$0.27} & 66.07{\scriptsize$\pm$0.14} & \textbf{68.26{\scriptsize$\pm$0.27}}  \\
     && 32 & 90.41{\scriptsize$\pm$0.19} & 89.31{\scriptsize$\pm$0.34} & 83.45{\scriptsize$\pm$1.14} & \textbf{90.73{\scriptsize$\pm$0.29}} & 64.92{\scriptsize$\pm$0.40} & 61.99{\scriptsize$\pm$0.24} & 55.42{\scriptsize$\pm$0.51} & \textbf{66.19{\scriptsize$\pm$0.24}}  \\
    \bottomrule 
  \end{tabular}}
    \caption{Test accuracy results of SqueezeNet model.} \label{table: accu-squeezenet}
\end{table*}

\begin{figure*}[!t]
  \centering
  \subfigure[homogeneous~(CIFAR10)]{
    \begin{minipage}[b]{0.24\textwidth}
      \includegraphics[width=1\linewidth]{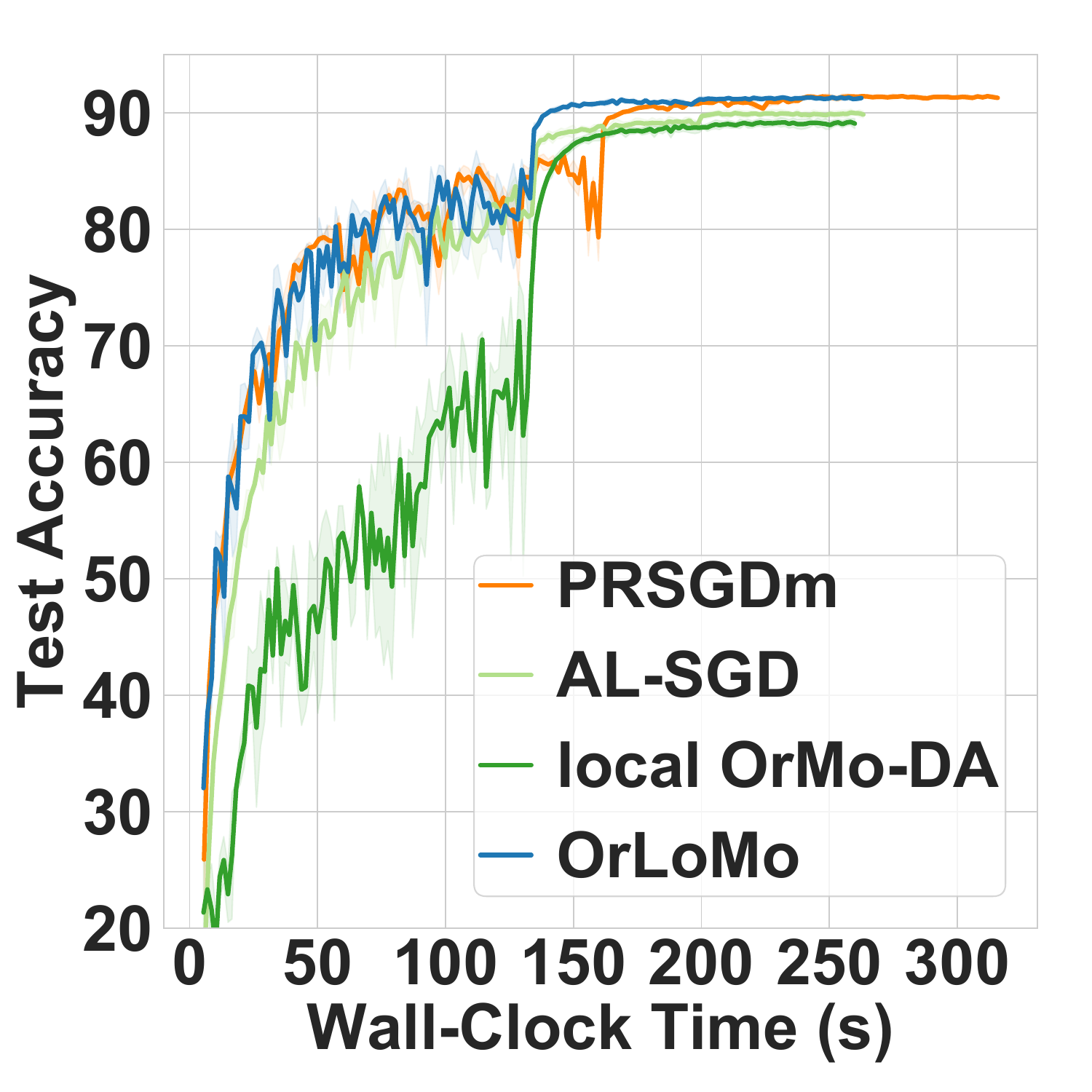}
      \end{minipage}}
  \subfigure[heterogeneous~(CIFAR10)]{
    \begin{minipage}[b]{0.24\textwidth}
      \includegraphics[width=1\linewidth]{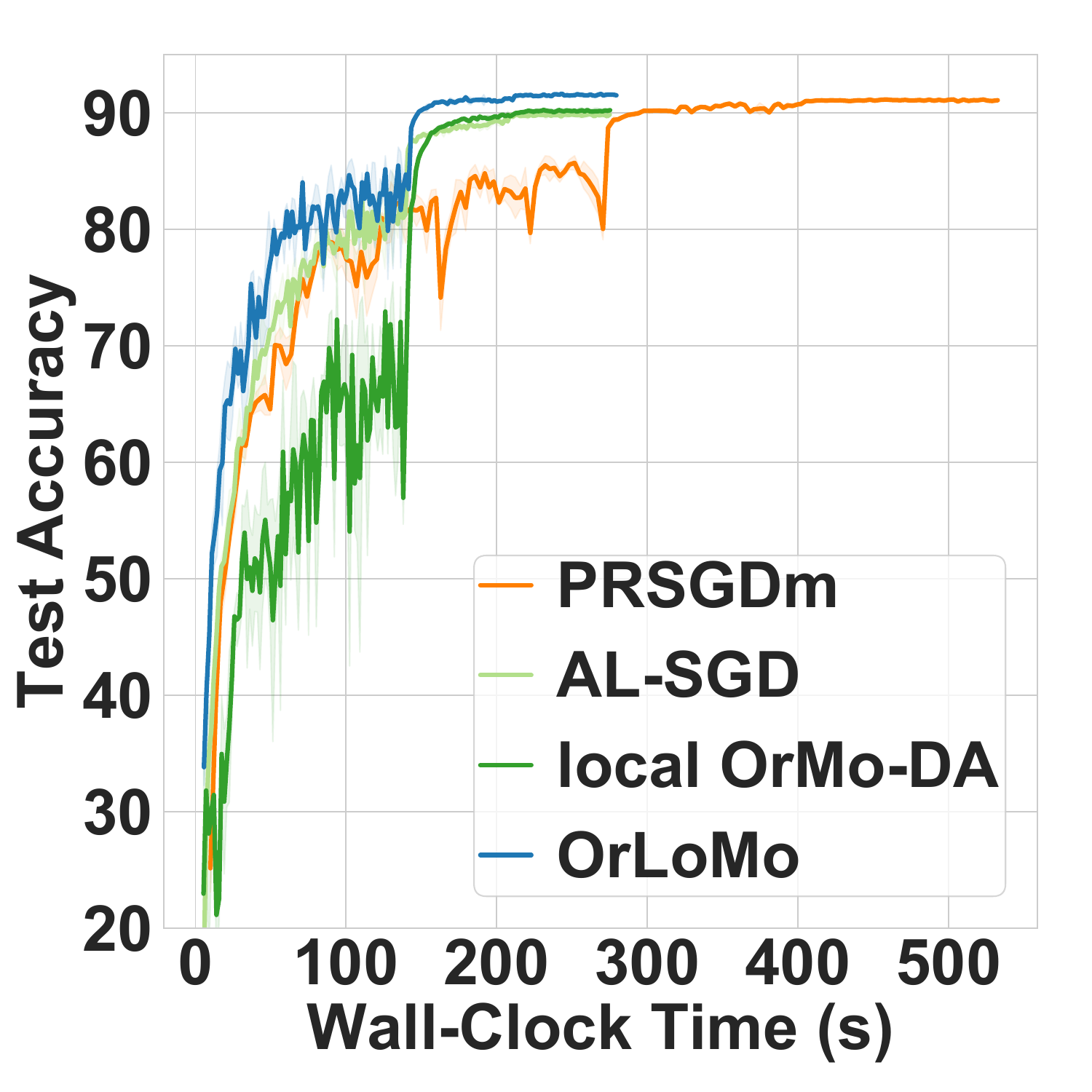}
      \end{minipage}}
  \subfigure[homogeneous~(CIFAR100)]{
    \begin{minipage}[b]{0.24\textwidth}
      \includegraphics[width=1\linewidth]{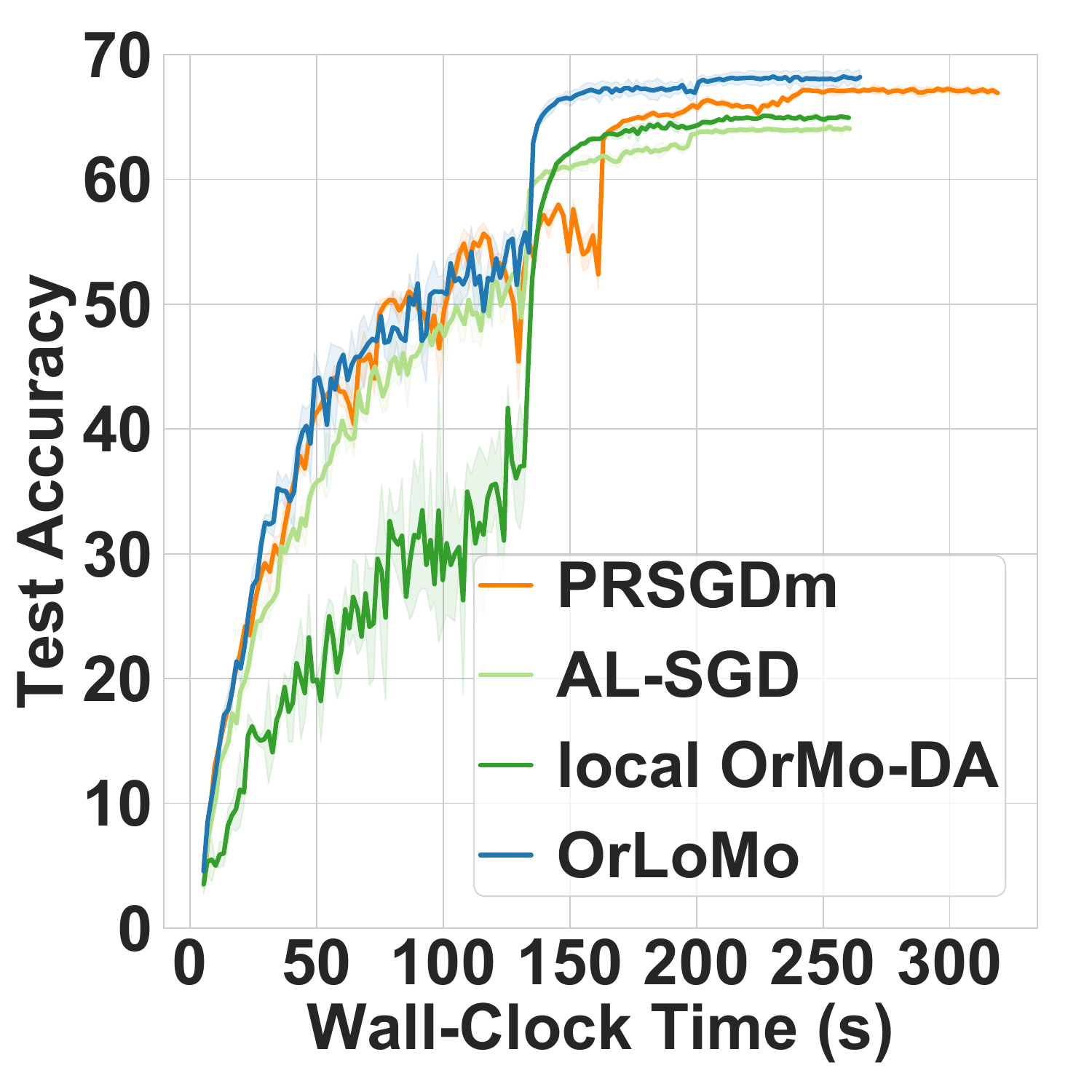}
      \end{minipage}}
  \subfigure[heterogeneous~(CIFAR100)]{
    \begin{minipage}[b]{0.24\textwidth}
      \includegraphics[width=1\linewidth]{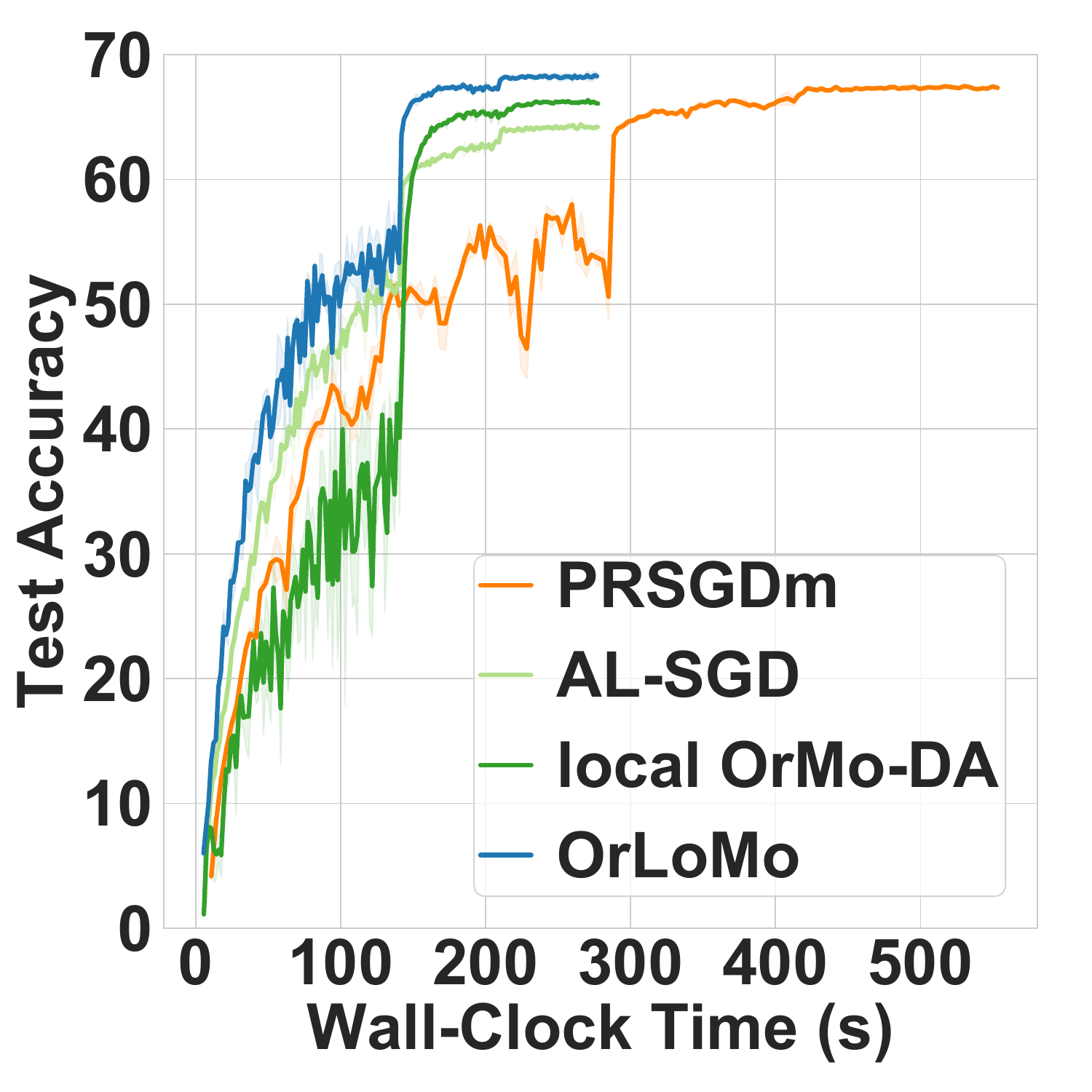}
      \end{minipage}}
\caption{Test accuracy curves of SqueezeNet model when $K=16, S=16$.}\label{fig:squeezenet-accu}
\end{figure*}   
\begin{remark}
The global learning rate in (\ref{OrLoMo-penalty}) implements a mechanism that penalizes both the local momentums and the local parameter updates with extremely large delays ($\tau_t > 2K$). This mechanism is critical for ensuring convergence of asynchronous methods under arbitrary delays, as shown in \cite{DBLP:conf/nips/MishchenkoBEW22, shiordered}.
\end{remark}
\begin{remark}
Theorem 1 in~\cite{DBLP:conf/icml/YuJY19} establishes the convergence rate of PRSGDm as $\mathcal{O}(\frac{1}{\gamma \tilde{T}}+\frac{\gamma}{K}+\gamma^2 S)$, with $\tilde{T}$ denoting its number of iterations. For fair comparison between OrLoMo and PRSGDm, we unify the metric to total gradient computations $C$.
PRSGDm's convergence rate is $\mathcal{O}(\frac{K}{\gamma C}+\frac{\gamma }{K}+\gamma^2 S )$ since $C=K\tilde{T}$.
 OrLoMo's convergence rate is $\mathcal{O}(\frac{K}{\gamma C}+\frac{\gamma}{K}+\gamma^2 S)$ since $C=ST$.
Our analysis proves identical gradient computation complexity between OrLoMo and its synchronous counterpart PRSGDm.
\end{remark}

\begin{remark}
If the local learning rate $\gamma$ is set as $\min\{\frac{(1-\beta)^2}{16LS}, \frac{K(1-\beta)^{\frac{3}{2}}\left(F(\w_0)- F^*\right)^{\frac{1}{2}}}{(LST)^{\frac{1}{2}}\sigma}, \frac{(1-\beta)\left[4K (F(\w_0)- F^*)\right]^{\frac{1}{3}}}{\left[S(S-1)T\right]^{\frac{1}{3}}\left(L\sigma \right)^{\frac{2}{3}}} \}$, the convergence rate in Theorem \ref{thm:OrLoMo} becomes $\mathcal{O}\left(\sqrt{\frac{L\sigma^2}{ST}} + \frac{(KL\sigma)^{\frac{2}{3}} (S-1)^{\frac{1}{3}}}{(ST)^{\frac{2}{3}}}+\frac{KL}{T}\right)$.
By letting $S=1$ and $\gamma=\min\{\frac{(1-\beta)^2}{16L}, \frac{K(1-\beta)^{\frac{3}{2}}\left(F(\w_0)- F^*\right)^{\frac{1}{2}}}{(LT)^{\frac{1}{2}}\sigma}\}$ in Theorem \ref{thm:OrLoMo}, we get the convergence rate $\mathcal{O}\left(\sqrt{\frac{L\sigma^2}{T}}+\frac{KL}{T}\right)$ for OrLoMo, which matches the convergence results for OrMo-DA in Theorem 2 in \cite{shiordered}.
\end{remark}

\section{Experiments} \label{expe}
In this section, we evaluate the performance of OrLoMo and other baseline methods. 
All methods are implemented based on the PS framework.
The experiments are conducted using NVIDIA RTX 2080 Ti GPUs. The distributed platform is implemented using Docker containerization.

  \begin{table*}[!t] \small
  \centering
  \setlength{\tabcolsep}{1mm}{  
  \begin{tabular}{c|c|c|cccc|cccc}
    \toprule   
      \multicolumn{3}{c|}{Datasets} & \multicolumn{4}{c|}{CIFAR10}& \multicolumn{4}{c}{CIFAR100}  \\ 
    \midrule
     & Workers & Local Iterations & PRSGDm & AL-SGD &  local OrMo-DA & OrLoMo & PRSGDm & AL-SGD &  local OrMo-DA & OrLoMo \\
    \midrule
     \multirow{6}{*}{\rotatebox{90}{homogeneous}}&\multirow{3}{*}{8}& 8 & 91.65{\scriptsize$\pm$0.11} & 90.99{\scriptsize$\pm$0.11} & \textbf{91.89{\scriptsize$\pm$0.05}} & 91.76{\scriptsize$\pm$0.19} & \textbf{67.78{\scriptsize$\pm$0.08}} & 65.20{\scriptsize$\pm$0.22}  & 67.27{\scriptsize$\pm$0.12} & 67.58{\scriptsize$\pm$0.21}   \\ 
     & & 16 & \textbf{91.68{\scriptsize$\pm$0.17}} & 90.72{\scriptsize$\pm$0.18} & 90.90{\scriptsize$\pm$0.11} & 91.49{\scriptsize$\pm$0.14} & \textbf{66.73{\scriptsize$\pm$0.32}} & 64.82{\scriptsize$\pm$0.21}  & 66.30{\scriptsize$\pm$0.09} & 66.32{\scriptsize$\pm$0.23}    \\
     & & 32 & 91.14{\scriptsize$\pm$0.15} & 90.69{\scriptsize$\pm$0.06} & 88.64{\scriptsize$\pm$0.26} & \textbf{91.26{\scriptsize$\pm$0.13}} & 65.43{\scriptsize$\pm$0.06}  & 64.74{\scriptsize$\pm$0.34} & 63.45{\scriptsize$\pm$0.18} & \textbf{66.08{\scriptsize$\pm$0.35}}  \\ \cmidrule(r){2-11} 
     &\multirow{3}{*}{16}& 8 & 91.36{\scriptsize$\pm$0.17} & 89.58{\scriptsize$\pm$0.10} & 90.71{\scriptsize$\pm$0.23} & \textbf{91.37{\scriptsize$\pm$0.13}} &  65.80{\scriptsize$\pm$0.12} & 62.26{\scriptsize$\pm$0.47} & 66.00{\scriptsize$\pm$0.16} & \textbf{66.18{\scriptsize$\pm$0.10}}  \\ 
     & & 16 & 90.64{\scriptsize$\pm$0.17} & 89.43{\scriptsize$\pm$0.18} & 88.69{\scriptsize$\pm$0.02} & \textbf{91.32{\scriptsize$\pm$0.14}} &  64.00{\scriptsize$\pm$0.38} & 60.80{\scriptsize$\pm$0.24} & 61.78{\scriptsize$\pm$0.15} & \textbf{65.23{\scriptsize$\pm$0.43}}  \\ 
     & & 32 & 89.72{\scriptsize$\pm$0.12} & 88.81{\scriptsize$\pm$0.17} & 76.42{\scriptsize$\pm$3.15} & \textbf{90.13{\scriptsize$\pm$0.08}} & 61.91{\scriptsize$\pm$0.24}  & 59.65{\scriptsize$\pm$0.44} & 48.48{\scriptsize$\pm$0.65} & \textbf{62.07{\scriptsize$\pm$0.19}}  \\ 
    \midrule 
     \multirow{6}{*}{\rotatebox{90}{heterogeneous}}&\multirow{3}{*}{8}& 8 & 91.90{\scriptsize$\pm$0.09} & 91.08{\scriptsize$\pm$0.28} &  91.88{\scriptsize$\pm$0.05} & \textbf{91.97{\scriptsize$\pm$0.31}} & 67.23{\scriptsize$\pm$0.24} & 65.02{\scriptsize$\pm$0.04} & \textbf{67.83{\scriptsize$\pm$0.12}} & 67.62{\scriptsize$\pm$0.42}   \\ 
     && 16 & 91.61{\scriptsize$\pm$0.13} & 90.84{\scriptsize$\pm$0.06} & 91.33{\scriptsize$\pm$0.11} & \textbf{91.76{\scriptsize$\pm$0.15}}  & 66.45{\scriptsize$\pm$0.33} & 64.71{\scriptsize$\pm$0.44} & 66.40{\scriptsize$\pm$0.16} & \textbf{66.48{\scriptsize$\pm$0.14}}  \\
     && 32 & 90.99{\scriptsize$\pm$0.02} & 90.94{\scriptsize$\pm$0.24} & 89.61{\scriptsize$\pm$0.38} & \textbf{91.44{\scriptsize$\pm$0.16}} & 65.43{\scriptsize$\pm$0.06} & 64.74{\scriptsize$\pm$0.34} & 63.45{\scriptsize$\pm$0.18} & \textbf{66.08{\scriptsize$\pm$0.35}} \\ 
    \cmidrule(r){2-11}
     &\multirow{3}{*}{16}& 8 & \textbf{91.28{\scriptsize$\pm$0.33}} & 89.58{\scriptsize$\pm$0.13} & 91.11{\scriptsize$\pm$0.10} & 91.27{\scriptsize$\pm$0.14} & 65.80{\scriptsize$\pm$0.12} & 62.26{\scriptsize$\pm$0.47} & 66.00{\scriptsize$\pm$0.16} & \textbf{66.18{\scriptsize$\pm$0.10}} \\ 
     && 16 & 90.72{\scriptsize$\pm$0.06} & 89.28{\scriptsize$\pm$0.09} & 89.44{\scriptsize$\pm$0.30} & \textbf{91.12{\scriptsize$\pm$0.15}} & 64.32{\scriptsize$\pm$0.08} & 61.02{\scriptsize$\pm$0.15} & 62.56{\scriptsize$\pm$0.64} & \textbf{65.44{\scriptsize$\pm$0.30}} \\ 
     && 32 & 89.76{\scriptsize$\pm$0.20} & 89.20{\scriptsize$\pm$0.48} & 82.26{\scriptsize$\pm$0.95} & \textbf{90.15{\scriptsize$\pm$0.23}} & 61.73{\scriptsize$\pm$0.23} & 59.60{\scriptsize$\pm$0.24} & 51.27{\scriptsize$\pm$0.78} & \textbf{61.88{\scriptsize$\pm$0.07}} \\ 
    \bottomrule 
  \end{tabular}}
  \caption{Test accuracy results of ResNet20 model.} \label{table: accu-resnet20}
\end{table*}
\begin{figure*}[!t]
  \centering
  \subfigure[homogeneous~(CIFAR10)]{
    \begin{minipage}[b]{0.24\textwidth}
      \includegraphics[width=1\linewidth]{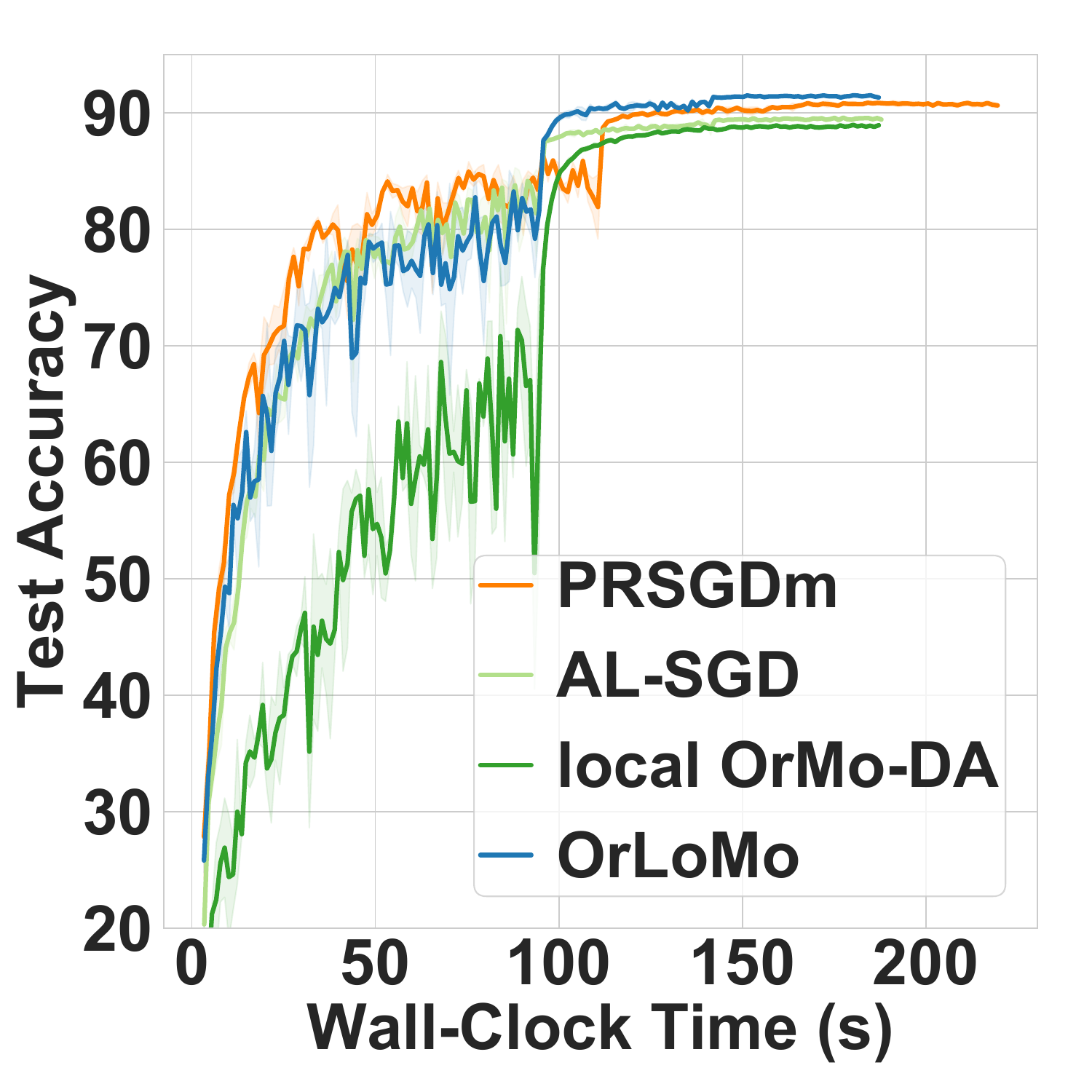}
      \end{minipage}}
  \subfigure[heterogeneous~(CIFAR10)]{
    \begin{minipage}[b]{0.24\textwidth}
      \includegraphics[width=1\linewidth]{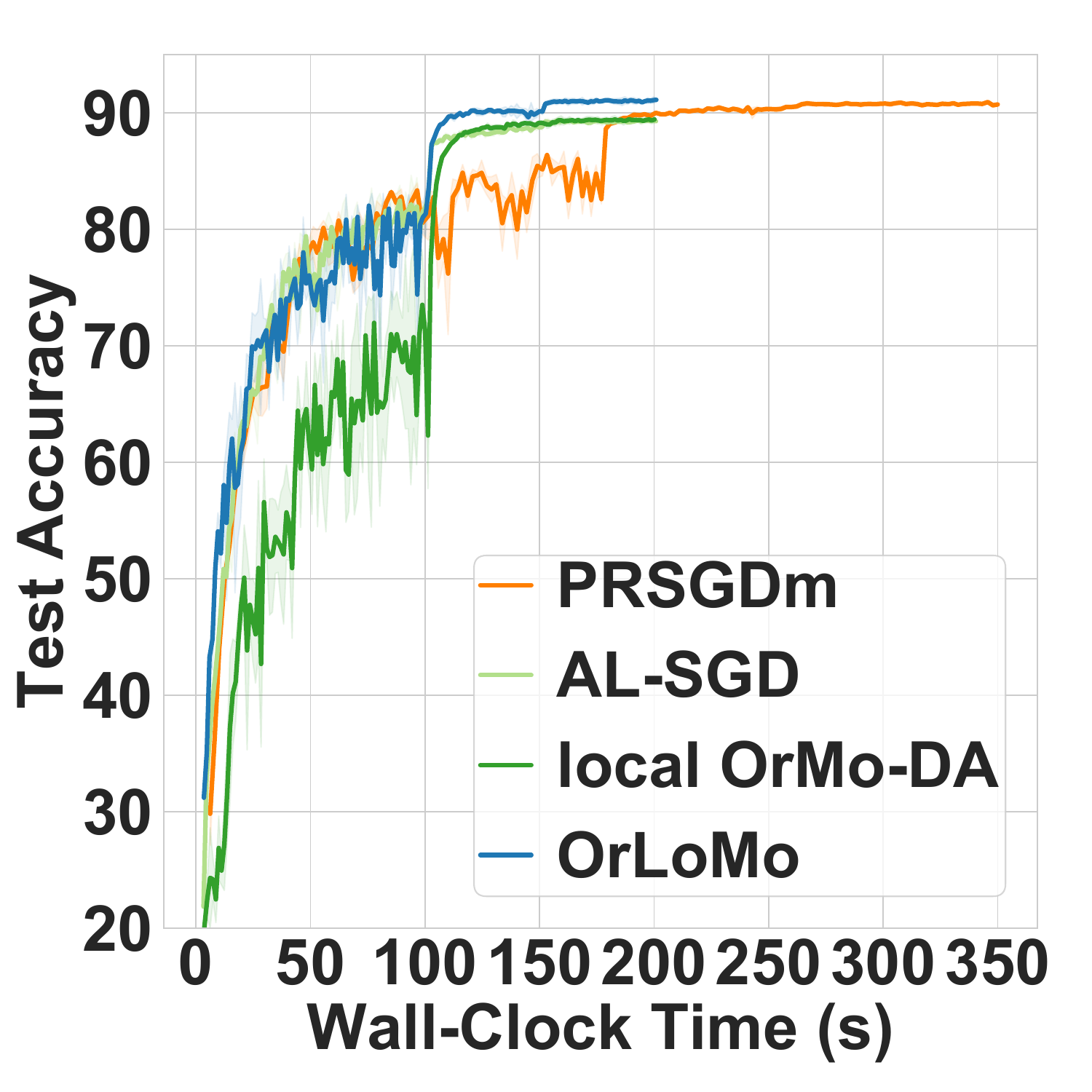}
      \end{minipage}}
  \subfigure[homogeneous~(CIFAR100)]{
    \begin{minipage}[b]{0.24\textwidth}
      \includegraphics[width=1\linewidth]{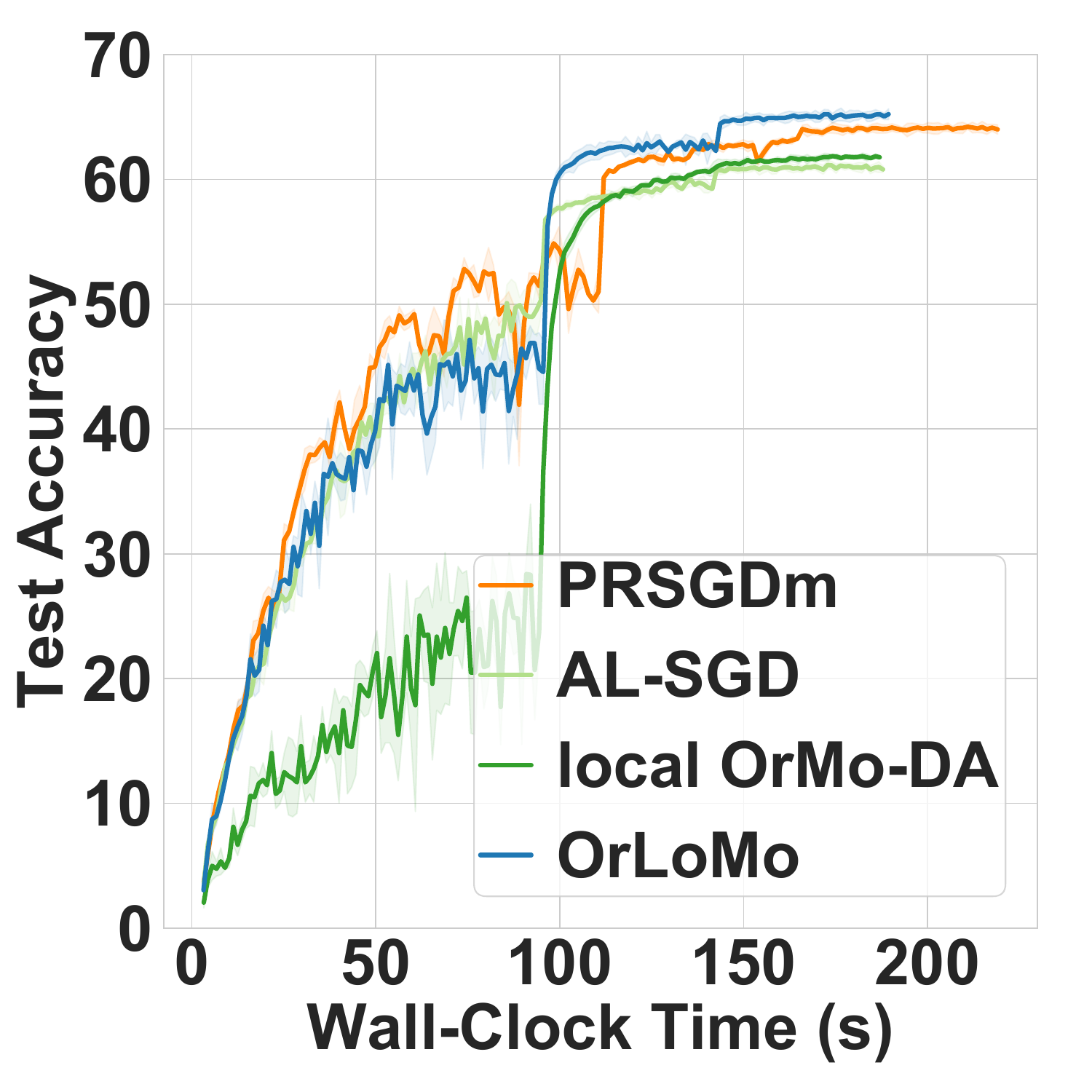}
      \end{minipage}}
  \subfigure[heterogeneous~(CIFAR100)]{
    \begin{minipage}[b]{0.24\textwidth}
      \includegraphics[width=1\linewidth]{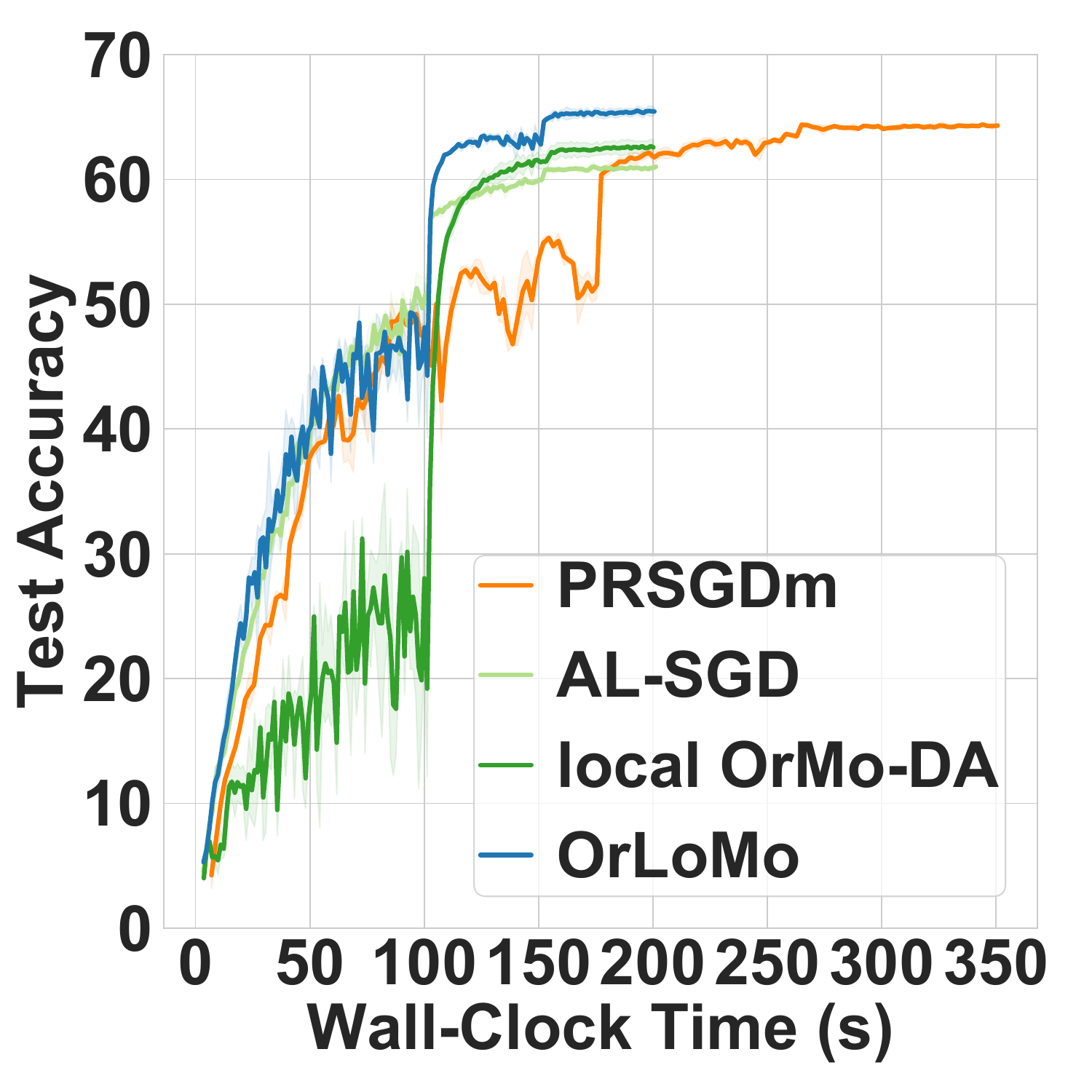}
      \end{minipage}}
\caption{Test accuracy curves of ResNet20 model when $K=16, S=16$.}\label{fig:resnet20-accu}
\end{figure*}   

  \begin{table}[!t] \small
  \centering
  \begin{tabular}{c|ccccc}
    \toprule   
     \emph{Slow/Fast} &  1 & 2 & 5 & 10 & 50  \\
    \midrule
    Maximum Delay&  34 & 207 & 714  & 1905  &  9449   \\
     Test Accuracy&  91.38 & 91.42 & 91.23  &  91.62  &  91.44   \\ 
    \bottomrule 
    \end{tabular}
      \caption{Test accuracy results of OrLoMo under different heterogeneous scenarios on ResNet20 model and CIFAR10 dataset with $K=16, S=8$.} \label{tab: extremedelay}
\end{table}
The baseline methods include PRSGDm \cite{DBLP:conf/icml/YuJY19}, AL-SGD, and local OrMo-DA. In local OrMo-DA, each worker performs SGD steps locally, while the server follows the same update rule as OrMo-DA \cite{shiordered}, as detailed in the appendix. AL-SGD, local OrMo-DA, and OrLoMo are ADL methods, while PRSGDm is an SDL method. We evaluate these methods by training SqueezeNet \cite{NIPS2015_ae0eb3ee} and ResNet20 \cite{he2016deep} models on CIFAR10 and CIFAR100 datasets~\cite{Krizhevsky2009LearningML}. Additional experimental results of ResNet32 on Tiny-ImageNet are provided in the appendix. The models are trained with 160 epochs.
The number of workers $K$ is set to 8 and 16. The batch size for each local iteration is set to 64. The number of local iterations $S$ is set to 8, 16, and 32. 
The local learning rate is multiplied by 0.1 at the 80-th and 120-th epochs.
The momentum coefficient is set to 0.9. The weight decay is set to 0.001. Each experiment is repeated 3 times. Two scenarios are considered in our experiments:
\begin{itemize}
    \item Homogeneous: All the workers have similar computational capabilities.
    \item Heterogeneous: $25\%$ of all workers are designated as slow workers, with their average gradient computation time being twice as long as that of the other workers.
\end{itemize}

Table~\ref{table: accu-squeezenet} and Table~\ref{table: accu-resnet20} show the test accuracy results. Figure~\ref{fig:squeezenet-accu} and Figure~\ref{fig:resnet20-accu} present the test accuracy curves with respect to wall-clock time. Experimental results on training loss are presented in the appendix.
Compared with AL-SGD and local OrMo-DA, OrLoMo can achieve better test accuracy. 
OrLoMo's communication overhead introduced by local momentums incurs negligible training time impact, maintaining comparable efficiency to both AL-SGD and local OrMo-DA.
As shown in Figure~\ref{fig:squeezenet-accu} and Figure~\ref{fig:resnet20-accu}, OrLoMo can be 2 times faster than PRSGDm in the heterogeneous scenario since the training speed of PRSGDm is hindered by slow workers. In contrast, slow workers have a limited impact on OrLoMo's training speed. 
Even in the homogeneous scenario, OrLoMo remains faster than PRSGDm. 
This advantage arises because the computation time of each worker varies within a certain range, and some workers must wait for the others to complete their gradient computations in PRSGDm. 

To assess OrLoMo's robustness to arbitrary delays (particularly extreme maximum delays), we conduct experiments under different heterogeneous scenarios. We designate one worker as the slow worker while keeping the others as fast workers. 
Table~\ref{tab: extremedelay} shows test accuracy results of OrLoMo when $K=16,$ $S=8$. 
The \emph{Slow/Fast} ratio quantifies the disparity in average gradient computation time between the slow worker and the fast workers, with higher values indicating more severe maximum delays.
Our experimental results confirm that OrLoMo sustains robust performance under extremely large delays. This aligns with our theoretical analysis, which remains valid under arbitrary delays.
\section{Conclusion}
In this paper, we propose a novel method, called OrLoMo, for asynchronous distributed learning. To the best of our knowledge, OrLoMo is the first method to implement asynchronous distributed MSGD with local updates. We prove the convergence of OrLoMo for non-convex problems under arbitrary delays. Empirical results demonstrate that OrLoMo can achieve state-of-the-art performance.

\section{Acknowledgments}
This work is supported by  National Key R\&D Program of China~(No.2020YFA0713900), the Key Major Project of Pengcheng
 Laboratory~(No.PCL2024A06) and NSFC Project
~(No.12326615).
\bibliography{aaai2026}

\onecolumn
\section{Appendix}

\subsection{Convergence Analysis}
First, we introduce an auxiliary variable $\hat{\u}_t$ corresponding to $\u_t$ for each iteration $t \geq 1$: 
\begin{align} 
& \hat{\u}_{t+1} = \left\{
\begin{aligned}
& \beta \hat{\u}_{t}+\hat{\eta}_{t,k_{t-1}}{\Delta \u}_{t}^{k_{t-1}} & K \mid \left(t-1\right), \\
& \hat{\u}_{t}+\hat{\eta}_{t,k_{t-1}}{\Delta \u}_{t}^{k_{t-1}} & K \nmid \left(t-1\right),
\end{aligned}
\right. 
\end{align} and $\hat{\u}_{1} =  \sum_{k \in [K]}  \hat{\eta}_{0,k}{\Delta \u}_{0}^{k}$.

\begin{lemma}
  For any $t \geq 0$, the difference between $\u_{t+1}$ and $\hat{\u}_{t+1}$ can be expressed as follows:
  \begin{align} \label{local-ormo-eq:ugap-appendix}
    \hat{\u}_{t+1} - \u_{t+1} =  \sum_{k \in [K], k \neq k_t} \beta^{\lceil \frac{t}{K} \rceil - \lceil \frac{ite(t, k)}{K}\rceil} \hat{\eta}_{ite(t, k),k}{\Delta \u}_{ite(t, k)}^k.
  \end{align}
\end{lemma}
\begin{proof} 
  For $t = 0$, $$\hat{\u}_{1} =  \sum_{k \in [K]} \hat{\eta}_{0,k} {\Delta \u}_{0}^{k}, \u_1 = \eta_0 {\Delta \u}_{0}^{k_0},$$ then we have
  \begin{align*}
    \hat{\u}_1 - \u_1 =   \sum_{k \in [K], k \neq k_0} \beta^{\lceil \frac{0}{K} \rceil - \lceil \frac{0}{K}\rceil}\hat{\eta}_{0,k}  {\Delta \u}_{0}^{k}
    = \sum_{k \in [K], k \neq k_0} \beta^{\lceil \frac{0}{K} \rceil - \lceil \frac{ite(0, k)}{K}\rceil} \hat{\eta}_{0,k} {\Delta \u}_{0}^{k}.
  \end{align*}

Given that (\ref{local-ormo-eq:ugap-appendix}) is valid at $t = t'-1$, we demonstrate its validity for $t = t'$, where $t'$ is an arbitrary integer satisfying $t' \geq 1$. 

We begin our discussion by dividing it into two cases depending on whether $t'-1$ is divisible by $K$.

\begin{itemize}
    \item $K \mid (t'-1)$
  \begin{align*}
    \hat{\u}_{t'+1} &= \beta \hat{\u}_{t'}+{\hat{\eta}}_{t', k_{t'-1}}{\Delta \u}_{t'}^{k_{t'-1}}, \\
    \u_{t'+1}&=\beta \u_{t'} +  \beta^{\lceil \frac{t'}{K} \rceil - \lceil \frac{ite(t', k_{t'})}{K}\rceil}{\eta}_{t'}{\Delta \u}_{ite(t', k_{t'})}^{k_{t'}}\\
    &\overset{(a)}{=} \beta \u_{t'} +  \beta^{\lceil \frac{t'}{K} \rceil - \lceil \frac{ite(t', k_{t'})}{K}\rceil}{\hat{\eta}}_{ite(t', k_{t'}), k_{t'}}{\Delta \u}_{ite(t', k_{t'})}^{k_{t'}}\\
    \hat{\u}_{t'+1} - \u_{t'+1} &= \beta (\hat{\u}_{t'} - \u_{t'}) + {\hat{\eta}}_{t', k_{t'-1}}{\Delta \u}_{t'}^{k_{t'-1}} - \beta^{\lceil \frac{t'}{K} \rceil - \lceil \frac{ite(t', k_{t'})}{K}\rceil}{\hat{\eta}}_{ite(t', k_{t'}), k_{t'}}{\Delta \u}_{ite(t', k_{t'})}^{k_{t'}}\\  
    &\overset{(b)}{=} \sum_{k \in [K], k \neq k_{t'-1}} \beta^{\lceil \frac{t'}{K} \rceil - \lceil \frac{ite(t'-1, k)}{K}\rceil} {\hat{\eta}}_{ite(t'-1, k), k}{\Delta \u}_{ite(t'-1, k)}^{k}\\
    &~~~~- \beta^{\lceil \frac{t'}{K} \rceil - \lceil \frac{ite(t', k_{t'})}{K}\rceil}{\hat{\eta}}_{ite(t', k_{t'}), k_{t'}}{\Delta \u}_{ite(t', k_{t'})}^{k_{t'}}+{\hat{\eta}}_{t', k_{t'-1}}{\Delta \u}_{t'}^{k_{t'-1}}
  \end{align*}
   Equation (a) holds because ${\eta}_{t'}={\eta}_{next\left(ite\left(t', k_{t'}\right), k_{t'}\right)}={\hat{\eta}}_{ite(t', k_{t'}), k_{t'}}$, where the first equality follows from the identity
   $next\left(ite\left(t', k_{t'}\right), k_{t'}\right)=t'$ and the second equality uses the definition $\hat{\eta}_{t,k} =  \eta_{next(t,k)}$.
   
   Equation (b) holds because $\lceil \frac{t'}{K} \rceil=\lceil \frac{t'-1}{K} \rceil + 1$, when $K \mid (t'-1)$.
    
  \item $K \nmid (t'-1)$ 
  \begin{align*}
    \hat{\u}_{t'+1} &= \hat{\u}_{t'}+{\hat{\eta}}_{t', k_{t'-1}}{\Delta \u}_{t'}^{k_{t'-1}}, \\
    \u_{t'+1}&= \u_{t'} +  \beta^{\lceil \frac{t'}{K} \rceil - \lceil \frac{ite(t', k_{t'})}{K}\rceil}{\hat{\eta}}_{ite(t', k_{t'}), k_{t'}}{\Delta \u}_{ite(t', k_{t'})}^{k_{t'}}\\
    \hat{\u}_{t'+1} - \u_{t'+1} 
    &= (\hat{\u}_{t'} - \u_{t'}) + {\hat{\eta}}_{t', k_{t'-1}}{\Delta \u}_{t'}^{k_{t'-1}}- \beta^{\lceil \frac{t'}{K} \rceil - \lceil \frac{ite(t', k_{t'})}{K}\rceil}{\hat{\eta}}_{ite(t', k_{t'}), k_{t'}}{\Delta \u}_{ite(t', k_{t'})}^{k_{t'}}\\
    &\overset{(c)}{=} \sum_{k \in [K], k \neq k_{t'-1}} \beta^{\lceil \frac{t'}{K} \rceil - \lceil \frac{ite(t'-1, k)}{K}\rceil} {\hat{\eta}}_{ite(t'-1, k), k}{\Delta \u}_{ite(t'-1, k)}^{k}\\
    &~~~~ - \beta^{\lceil \frac{t'}{K} \rceil - \lceil \frac{ite(t', k_{t'})}{K}\rceil}{\hat{\eta}}_{ite(t', k_{t'}), k_{t'}}{\Delta \u}_{ite(t', k_{t'})}^{k_{t'}}+ {\hat{\eta}}_{t', k_{t'-1}}{\Delta \u}_{t'}^{k_{t'-1}}
  \end{align*}
    Equation (c) holds because $\lceil \frac{t'}{K} \rceil=\lceil \frac{t'-1}{K} \rceil$, when $K \nmid (t'-1)$.
  \end{itemize}
  
  Since $ite(t', k) = ite(t'-1,k), \forall k \neq k_{t'-1}$, we can get the following equation for both cases above:
  \begin{align*}
    \hat{\u}_{t'+1} - \u_{t'+1}  
    &= \sum_{k \in [K], k \neq k_{t'-1}} \beta^{\lceil \frac{t'}{K} \rceil - \lceil \frac{ite(t', k)}{K}\rceil} {\hat{\eta}}_{ite(t', k),k}{\Delta \u}_{ite(t', k)}^{k} \\
    &~~~~ - \beta^{\lceil \frac{t'}{K} \rceil - \lceil \frac{ite(t', k_{t'})}{K}\rceil} {\hat{\eta}}_{ite(t', k_{t'}), k_{t'}}{\Delta \u}_{ite(t', k_{t'})}^{k_{t'}} +  {\hat{\eta}}_{t', k_{t'-1}}{\Delta \u}_{t'}^{k_{t'-1}}
    \end{align*}

  Then, we divide our discussion into two cases, considering  whether $k_{t'} = k_{t'-1}$.

  \begin{itemize}
      \item $k_{t'}=k_{t'-1}$
  \begin{align*}
    \hat{\u}_{t'+1} - \u_{t'+1} &\overset{(d)}{=} \sum_{k \in [K], k \neq k_{t'}} \beta^{\lceil \frac{t'}{K} \rceil - \lceil \frac{ite(t', k)}{K}\rceil} {\hat{\eta}}_{ite(t', k),k}{\Delta \u}_{ite(t', k)}^{k}.
  \end{align*}    
  Equation (d) holds because $ite(t', k_{t'}) = t'$ when $k_{t'}=k_{t'-1}$.
  \item $k_{t'}\neq k_{t'-1}$    
    \begin{align*}
    \hat{\u}_{t'+1} - \u_{t'+1}  &= \sum_{k \in [K], k \neq k_{t'-1}} \beta^{\lceil \frac{t'}{K} \rceil - \lceil \frac{ite(t', k)}{K}\rceil}{\hat{\eta}}_{ite(t', k),k}{\Delta \u}_{ite(t', k)}^{k}  +{\hat{\eta}}_{t', k_{t'-1}} {\Delta \u}_{t'}^{k_{t'-1}} \\
    &~~~~- \beta^{\lceil \frac{t'}{K} \rceil - \lceil \frac{ite(t', k_{t'})}{K}\rceil} {\hat{\eta}}_{ite(t', k_{t'}), k_{t'}}{\Delta \u}_{ite(t', k_{t'})}^{k_{t'}}\\
    &= \sum_{k \in [K], k \neq k_{t'-1}, k \neq k_{t'}} \beta^{\lceil \frac{t'}{K} \rceil - \lceil \frac{ite(t', k)}{K}\rceil} {\hat{\eta}}_{ite(t', k),k}{\Delta \u}_{ite(t', k)}^{k}+ {\hat{\eta}}_{t', k_{t'-1}} {\Delta \u}_{t'}^{k_{t'-1}}\\
    &= \sum_{k \in [K], k \neq k_{t'-1}, k \neq k_{t'}} \beta^{\lceil \frac{t'}{K} \rceil - \lceil \frac{ite(t', k)}{K}\rceil} {\hat{\eta}}_{ite(t', k),k} {\Delta \u}_{ite(t', k)}^{k} \\
    &~~~~+ \beta^{\lceil \frac{t'}{K} \rceil - \lceil \frac{ite(t', k_{t'-1})}{K}\rceil}{\hat{\eta}}_{ite(t', k_{t'-1}),k_{t'-1}}{\Delta \u}_{ite(t', k_{t'-1})}^{k_{t'-1}} \\
    &= \sum_{k \in [K], k \neq k_{t'}} \beta^{\lceil \frac{t'}{K} \rceil - \lceil \frac{ite(t', k)}{K}\rceil}{\hat{\eta}}_{ite(t', k),k}{\Delta \u}_{ite(t', k)}^{k}. 
  \end{align*}
  \end{itemize}

  We can conclude that $\hat{\u}_{t+1} - \u_{t+1} = \sum_{k \in [K], k \neq k_t} \beta^{\lceil \frac{t}{K} \rceil - \lceil \frac{ite(t, k)}{K}\rceil} {\hat{\eta}}_{ite(t, k),k}{\Delta \u}_{ite(t, k)}^k$ holds for $\forall t \geq 0$.
\end{proof}

Secondly, we introduce an auxiliary variable $\hat{\w}_t $ corresponding to $\w_t$  for each iteration $t \geq 1$:  
\begin{align*} 
& \hat{\w}_{t+1} = \left\{
\begin{aligned}
& \hat{\w}_{t}-\beta \hat{\u}_{t}- {\hat{\eta}}_{t,k_{t-1}}{\Delta \w}_t^{k_{t-1}}& K \mid \left(t-1\right), \\
& \hat{\w}_{t}- {\hat{\eta}}_{t,k_{t-1}}{\Delta \w}_t^{k_{t-1}} & K \nmid \left(t-1\right),
\end{aligned}
\right. 
\end{align*} and $\hat{\w}_{1} = \w_0 - \sum_{k \in [K]} {\hat{\eta}}_{0,k} {\Delta \w}_0^k$.

\begin{lemma}
  For any $t \geq 0$, the difference between $\w_{t+1}$ and $\hat{\w}_{t+1}$ can be expressed as follows:
  \begin{align} \label{local-ormo-eq:wgapPenalty-appendix}
    \hat{\w}_{t+1} - \w_{t+1}=  -\sum_{k \in [K], k \neq k_t} \left[{\hat{\eta}}_{ite(t, k),k}{\Delta \w}_{ite(t, k)}^k + \frac{\beta(1-\beta^{\lceil \frac{t}{K} \rceil - \lceil \frac{ite(t, k)}{K}\rceil})}{1-\beta}{\hat{\eta}}_{ite(t, k),k}{\Delta \u}_{ite(t, k)}^k\right].
  \end{align}
\end{lemma}
\begin{proof}
  For $t = 0$, $$\hat{\w}_{1} = \w_0 - \sum_{k \in [K]} {\hat{\eta}}_{0,k} {\Delta \w}_0^k, \w_1 = \w_0 -  {\hat{\eta}}_{0,k_0}{\Delta \w}_{0}^{k_0},$$ then we have
  \begin{align*}
    \hat{\w}_1 - \w_1 
    &= -\sum_{k \in [K], k \neq k_0} {\hat{\eta}}_{0,k} {\Delta \w}_0^k \\
    &= -\sum_{k \in [K], k \neq k_0} {\hat{\eta}}_{ite(0,k),k}{\Delta \w}_{ite(0,k)}^{k}-\sum_{k \in [K], k \neq k_0}\frac{\beta(1-\beta^{\lceil \frac{0}{K} \rceil - \lceil \frac{ite(0, k)}{K}\rceil})}{1-\beta}{\hat{\eta}}_{ite(0, k), k}{\Delta \u}_{ite(0, k)}^{k}
  \end{align*}

Given that (\ref{local-ormo-eq:wgapPenalty-appendix}) is valid at $t = t'-1$, we  demonstrate its validity for $t = t'$, where $t'$ is an arbitrary integer satisfying $t' \geq 1$. 

We begin our discussion by dividing it into two cases depending on whether $t'-1$ is divisible by $K$.

\begin{itemize}
    \item $K \mid (t'-1)$
  \begin{align*}
    \hat{\w}_{t'+1} &= \hat{\w}_{t'} - \beta \hat{\u}_{t'} - {\hat{\eta}}_{t', k_{t'-1}}{\Delta \w}_{t'}^{k_{t'-1}}, \\
    \w_{t'+1} &= \w_{t'}  - \beta \u_{t'} -{\hat{\eta}}_{ite(t', k_{t'}), k_{t'}}{\Delta \w}_{ite(t', k_{t'})}^{k_{t'}}- \frac{\beta (1-\beta^{\lceil \frac{t'}{K} \rceil - \lceil \frac{ite(t', k_{t'})}{K}\rceil})}{1-\beta}{\hat{\eta}}_{ite(t', k_{t'}), k_{t'}}{\Delta \u}_{ite(t', k_{t'})}^{k_{t'}}\\
    \hat{\w}_{t'+1} - \w_{t'+1} 
    &= \hat{\w}_{t'} - \w_{t'} - \beta \left(\hat{\u}_{t'} - \u_{t'}\right) - {\hat{\eta}}_{t', k_{t'-1}}{\Delta \w}_{t'}^{k_{t'-1}}+ {\hat{\eta}}_{ite(t', k_{t'}), k_{t'}}{\Delta \w}_{ite(t', k_{t'})}^{k_{t'}}\\
    &~~~~+\frac{\beta (1-\beta^{\lceil \frac{t'}{K} \rceil - \lceil \frac{ite(t', k_{t'})}{K}\rceil})}{1-\beta}{\hat{\eta}}_{ite(t', k_{t'}), k_{t'}} {\Delta \u}_{ite(t', k_{t'})}^{k_{t'}} \\
    &=-\sum_{k \in [K], k \neq k_{t'-1}}{\hat{\eta}}_{ite(t'-1, k), k} \left[{\Delta \w}_{ite(t'-1, k)}^{k} + \frac{\beta(1-\beta^{\lceil \frac{t'-1}{K} \rceil - \lceil \frac{ite(t'-1, k)}{K}\rceil})}{1-\beta}{\Delta \u}_{ite(t'-1, k)}^{k}\right]\\
    &~~~~- \sum_{k \in [K], k \neq k_{t'-1}}\beta^{\lceil \frac{t'-1}{K} \rceil - \lceil \frac{ite(t'-1, k)}{K}\rceil+1}{\hat{\eta}}_{ite(t'-1, k), k}{\Delta \u}_{ite(t'-1, k)}^{k}\\
    &~~~~+\frac{\beta (1-\beta^{\lceil \frac{t'}{K} \rceil - \lceil \frac{ite(t', k_{t'})}{K}\rceil})}{1-\beta}{\hat{\eta}}_{ite(t', k_{t'}), k_{t'}}{\Delta \u}_{ite(t', k_{t'})}^{k_{t'}}- {\hat{\eta}}_{t', k_{t'-1}}{\Delta \w}_{t'}^{k_{t'-1}}\\
    &~~~~+ {\hat{\eta}}_{ite(t', k_{t'}), k_{t'}}{\Delta \w}_{ite(t', k_{t'})}^{k_{t'}}\\
    &=-\sum_{k \in [K], k \neq k_{t'-1}}{\hat{\eta}}_{ite(t'-1, k), k} \left[{\Delta \w}_{ite(t'-1, k)}^{k} + \frac{\beta(1-\beta^{\lceil \frac{t'}{K} \rceil - \lceil \frac{ite(t'-1, k)}{K}\rceil})}{1-\beta}{\Delta \u}_{ite(t'-1, k)}^{k}\right]\\
    &~~~~+\frac{\beta (1-\beta^{\lceil \frac{t'}{K} \rceil - \lceil \frac{ite(t', k_{t'})}{K}\rceil})}{1-\beta}{\hat{\eta}}_{ite(t', k_{t'}), k_{t'}}{\Delta \u}_{ite(t', k_{t'})}^{k_{t'}}- {\hat{\eta}}_{t', k_{t'-1}}{\Delta \w}_{t'}^{k_{t'-1}}\\
    &~~~~+ {\hat{\eta}}_{ite(t', k_{t'}), k_{t'}}{\Delta \w}_{ite(t', k_{t'})}^{k_{t'}}
  \end{align*}    
      \item $K \nmid (t'-1)$
  \begin{align*}
    \hat{\w}_{t'+1} &= \hat{\w}_{t'} - {\hat{\eta}}_{t', k_{t'-1}}{\Delta \w}_{t'}^{k_{t'-1}}, \\
    \w_{t'+1} &= \w_{t'} -{\hat{\eta}}_{ite(t', k_{t'}), k_{t'}}{\Delta \w}_{ite(t', k_{t'})}^{k_{t'}}- \frac{\beta (1-\beta^{\lceil \frac{t'}{K} \rceil - \lceil \frac{ite(t', k_{t'})}{K}\rceil})}{1-\beta}{\hat{\eta}}_{ite(t', k_{t'}), k_{t'}}{\Delta \u}_{ite(t', k_{t'})}^{k_{t'}}\\
    \hat{\w}_{t'+1} - \w_{t'+1}
    &= \hat{\w}_{t'} - \w_{t'}  - {\hat{\eta}}_{t', k_{t'-1}}{\Delta \w}_{t'}^{k_{t'-1}}+ {\hat{\eta}}_{ite(t', k_{t'}), k_{t'}}{\Delta \w}_{ite(t', k_{t'})}^{k_{t'}}\\
    &~~~~+\frac{\beta (1-\beta^{\lceil \frac{t'}{K} \rceil - \lceil \frac{ite(t', k_{t'})}{K}\rceil})}{1-\beta}{\hat{\eta}}_{ite(t', k_{t'}), k_{t'}} {\Delta \u}_{ite(t', k_{t'})}^{k_{t'}}  \\
    &=-\sum_{k \in [K], k \neq k_{t'-1}}{\hat{\eta}}_{ite(t'-1, k), k} \left[{\Delta \w}_{ite(t'-1, k)}^{k} + \frac{\beta(1-\beta^{\lceil \frac{t'-1}{K} \rceil - \lceil \frac{ite(t'-1, k)}{K}\rceil})}{1-\beta}{\Delta \u}_{ite(t'-1, k)}^{k}\right]\\
    &~~~~+\frac{\beta (1-\beta^{\lceil \frac{t'}{K} \rceil - \lceil \frac{ite(t', k_{t'})}{K}\rceil})}{1-\beta}{\hat{\eta}}_{ite(t', k_{t'}), k_{t'}}{\Delta \u}_{ite(t', k_{t'})}^{k_{t'}}- {\hat{\eta}}_{t', k_{t'-1}}{\Delta \w}_{t'}^{k_{t'-1}}\\
    &~~~~+ {\hat{\eta}}_{ite(t', k_{t'}), k_{t'}}{\Delta \w}_{ite(t', k_{t'})}^{k_{t'}}\\
    &=-\sum_{k \in [K], k \neq k_{t'-1}}{\hat{\eta}}_{ite(t'-1, k), k} \left[{\Delta \w}_{ite(t'-1, k)}^{k} + \frac{\beta(1-\beta^{\lceil \frac{t'}{K} \rceil - \lceil \frac{ite(t'-1, k)}{K}\rceil})}{1-\beta}{\Delta \u}_{ite(t'-1, k)}^{k}\right]\\
    &~~~~+\frac{\beta (1-\beta^{\lceil \frac{t'}{K} \rceil - \lceil \frac{ite(t', k_{t'})}{K}\rceil})}{1-\beta}{\hat{\eta}}_{ite(t', k_{t'}), k_{t'}}{\Delta \u}_{ite(t', k_{t'})}^{k_{t'}}- {\hat{\eta}}_{t', k_{t'-1}}{\Delta \w}_{t'}^{k_{t'-1}}\\
    &~~~~+ {\hat{\eta}}_{ite(t', k_{t'}), k_{t'}}{\Delta \w}_{ite(t', k_{t'})}^{k_{t'}}
  \end{align*} 
\end{itemize}
  
  Since $ite(t', k) = ite(t'-1,k), \forall k \neq k_{t'-1}$, we can get the following equation for both cases above:
    \begin{align*}
    \hat{\w}_{t'+1} - \w_{t'+1} &=-\sum_{k \in [K], k \neq k_{t'-1}}{\hat{\eta}}_{ite(t', k), k} \left[{\Delta \w}_{ite(t', k)}^{k} + \frac{\beta(1-\beta^{\lceil \frac{t'}{K} \rceil - \lceil \frac{ite(t', k)}{K}\rceil})}{1-\beta}{\Delta \u}_{ite(t', k)}^{k}\right]\\
    &~~~~+\frac{\beta (1-\beta^{\lceil \frac{t'}{K} \rceil - \lceil \frac{ite(t', k_{t'})}{K}\rceil})}{1-\beta}{\hat{\eta}}_{ite(t', k_{t'}), k_{t'}}{\Delta \u}_{ite(t', k_{t'})}^{k_{t'}}- {\hat{\eta}}_{t', k_{t'-1}}{\Delta \w}_{t'}^{k_{t'-1}}\\
    &~~~~+ {\hat{\eta}}_{ite(t', k_{t'}), k_{t'}}{\Delta \w}_{ite(t', k_{t'})}^{k_{t'}}
  \end{align*} 
  
  Then, we divide our discussion into two cases, considering  whether $k_{t'} = k_{t'-1}$.

  \begin{itemize}
      \item $k_{t'}=k_{t'-1}$
 \begin{align*}
    \hat{\w}_{t'+1} - \w_{t'+1} 
    =-\sum_{k \in [K], k \neq k_{t'}}{\hat{\eta}}_{ite(t', k), k} \left[{\Delta \w}_{ite(t', k)}^{k}+ \frac{\beta(1-\beta^{\lceil \frac{t'}{K} \rceil - \lceil \frac{ite(t', k)}{K}\rceil})}{1-\beta}{\Delta \u}_{ite(t', k)}^{k}\right]
    \end{align*}

\item $k_{t'} \neq k_{t'-1}$
      \begin{align*}
    \hat{\w}_{t'+1} - \w_{t'+1} 
    &=-\sum_{k \in [K], k \neq k_{t'-1}, k \neq k_{t'}}{\hat{\eta}}_{ite(t', k), k} \left[{\Delta \w}_{ite(t', k)}^{k} + \frac{\beta(1-\beta^{\lceil \frac{t'}{K} \rceil - \lceil \frac{ite(t', k)}{K}\rceil})}{1-\beta}{\Delta \u}_{ite(t', k)}^{k}\right]\\
    &~~~~-{\hat{\eta}}_{ite(t', k_{t'}), k_{t'}}{\Delta \w}_{ite(t', k_{t'})}^{k_{t'}} -{\hat{\eta}}_{ite(t', k_{t'}), k_{t'}}\frac{\beta(1-\beta^{\lceil \frac{t'}{K} \rceil - \lceil \frac{ite(t', k_{t'})}{K}\rceil})}{1-\beta}{\Delta \u}_{ite(t', k_{t'})}^{k_{t'}}\\
    &~~~~+\frac{\beta (1-\beta^{\lceil \frac{t'}{K} \rceil - \lceil \frac{ite(t', k_{t'})}{K}\rceil})}{1-\beta}{\hat{\eta}}_{ite(t', k_{t'}), k_{t'}}{\Delta \u}_{ite(t', k_{t'})}^{k_{t'}}- {\hat{\eta}}_{t', k_{t'-1}}{\Delta \w}_{t'}^{k_{t'-1}}\\
    &~~~~+ {\hat{\eta}}_{ite(t', k_{t'}), k_{t'}}{\Delta \w}_{ite(t', k_{t'})}^{k_{t'}}\\
    &=-\sum_{k \in [K], k \neq k_{t'-1}, k \neq k_{t'}}{\hat{\eta}}_{ite(t', k), k} \left[{\Delta \w}_{ite(t', k)}^{k} + \frac{\beta(1-\beta^{\lceil \frac{t'}{K} \rceil - \lceil \frac{ite(t', k)}{K}\rceil})}{1-\beta}{\Delta \u}_{ite(t', k)}^{k}\right]\\
    &~~~~- {\hat{\eta}}_{t', k_{t'-1}}{\Delta \w}_{t'}^{k_{t'-1}}\\
    &=-\sum_{k \in [K], k \neq k_{t'}}{\hat{\eta}}_{ite(t', k), k} \left[{\Delta \w}_{ite(t', k)}^{k} + \frac{\beta(1-\beta^{\lceil \frac{t'}{K} \rceil - \lceil \frac{ite(t', k)}{K}\rceil})}{1-\beta}{\Delta \u}_{ite(t', k)}^{k}\right]
  \end{align*} 
  \end{itemize}
  We can conclude that   
  \begin{align*}
    \hat{\w}_{t+1} - \w_{t+1}
    = -\sum_{k \in [K], k \neq k_{t}}{\hat{\eta}}_{ite(t, k), k} \left[{\Delta \w}_{ite(t, k)}^k + \frac{\beta(1-\beta^{\lceil \frac{t}{K} \rceil - \lceil \frac{ite(t, k)}{K}\rceil})}{1-\beta}{\Delta \u}_{ite(t, k)}^k\right]
    \end{align*}
\end{proof}

For the sake of simplicity, we define $\Delta \h_{t}^{k_{t-1}}=\gamma\sum_{s=0}^{S-1}\g_{t,s}^{k_{t-1}}$  for $t \geq 1$, and
$\Delta \h_0^k =\gamma \sum_{s=0}^{S-1}\g_{0,s}^{k}$ for $k \in [K]$.
It's easy to verify that:
$$\frac{1}{1-\beta}\Delta \h_{t}^{k_{t-1}} = \frac{\beta}{1-\beta}{\Delta \u}_{t}^{k_{t-1}}+{\Delta \w}_t^{k_{t-1}},$$ for $t \geq 1$ and $
\frac{1}{1-\beta}\Delta \h_0^k = \frac{\beta}{1-\beta}{\Delta \u}_{0}^{k}+{\Delta \w}_0^k, $ for $ k \in [K].$

Then, we define another auxiliary variable $\hat{\y}_t$ corresponding to $\hat{\w}_t$ for each iteration $t \geq 1$:  
$$\hat{\y}_{t+1} = \hat{\y}_t - \frac{\hat{\eta}_{t,k_{t-1}}}{1-\beta}\Delta \h_{t}^{k_{t-1}},$$ and $\hat{\y}_1 = \w_0-\frac{1}{1-\beta}\sum_{k \in [K]} \hat{\eta}_{0,k}\Delta \h_0^k$ .

\begin{lemma} 
  For any $t \geq 1$, the difference between $\hat{\y}_t$ and $\hat{\w}_t$ can be expressed as follows:
  \begin{align*}
    \hat{\y}_t - \hat{\w}_t = -\frac{\beta}{1-\beta}\hat{\u}_t.
  \end{align*}
\end{lemma}
\begin{proof}
  For $t = 1$, we have that
  \begin{align*}
    \hat{\y}_1 - \hat{\w}_1 &= \left(\w_0 - \frac{1}{1-\beta}\sum_{k \in [K]}\hat{\eta}_{0,k} \Delta \h_0^k\right) - \left(\w_0 - \sum_{k \in [K]}\hat{\eta}_{0,k} \Delta \w_0^k\right)\\
    &=  - \frac{\beta}{1-\beta}\sum_{k \in [K]}\hat{\eta}_{0,k} \Delta \u_0^k=  - \frac{\beta}{1-\beta}\hat{\u}_{1}.
  \end{align*}
  
  Assuming that $\hat{\y}_t - \hat{\w}_t = -\frac{\beta}{1-\beta}\hat{\u}_t$ is true for $t=t'$ for some arbitrary integer $t'\geq 1$, we will prove that $\hat{\y}_t - \hat{\w}_t = -\frac{\beta}{1-\beta}\hat{\u}_{t}$ is true for $t=t'+1$. 
  
  We divide our discussion into two cases based on whether $t'-1$ is divisible by $K$.

  \begin{itemize}
      \item $K \mid (t'-1)$ 
  \begin{align*}
    \hat{\y}_{t'+ 1} &= \hat{\y}_{t'} - \frac{\hat{\eta}_{t', k_{t'-1}}}{1-\beta}\Delta \h_{t'}^{k_{t'-1}} \\
    \hat{\w}_{t'+ 1} &= \hat{\w}_{t'} - \beta \hat{\u}_{t'} - \hat{\eta}_{t', k_{t'-1}}{\Delta \w}_{t'}^{k_{t'-1}}\\
    \hat{\y}_{t'+ 1} - \hat{\w}_{t'+ 1} 
    &= \left(\hat{\y}_{t'} - \frac{\hat{\eta}_{t', k_{t'-1}}}{1-\beta}\Delta \h_{t'}^{k_{t'-1}}\right) - \left(\hat{\w}_{t'} - \beta \hat{\u}_{t'} - \hat{\eta}_{t', k_{t'-1}}{\Delta \w}_{t'}^{k_{t'-1}}\right) \\
    &= \hat{\y}_{t'} - \hat{\w}_{t'}  - \frac{\hat{\eta}_{t', k_{t'-1}}}{1-\beta}\Delta \h_{t'}^{k_{t'-1}}+ \beta \hat{\u}_{t'} + \hat{\eta}_{t', k_{t'-1}}{\Delta \w}_{t'}^{k_{t'-1}}\\
    &= -\frac{\beta^2}{1-\beta}\hat{\u}_{t'} - \frac{\beta}{1-\beta}\hat{\eta}_{t', k_{t'-1}}{\Delta \u}_{t'}^{k_{t'-1}} \\
    &= -\frac{\beta}{1-\beta}\hat{\u}_{t'+1}
\end{align*}
\item $K \nmid (t'-1)$ 
  \begin{align*}
    \hat{\y}_{t'+ 1} &= \hat{\y}_{t'} - \frac{\hat{\eta}_{t', k_{t'-1}}}{1-\beta}\Delta \h_{t'}^{k_{t'-1}} \\
    \hat{\w}_{t'+ 1} &= \hat{\w}_{t'} - \hat{\eta}_{t', k_{t'-1}}{\Delta \w}_{t'}^{k_{t'-1}} \\
    \hat{\y}_{t'+ 1} - \hat{\w}_{t'+ 1} 
    &= \left(\hat{\y}_{t'} - \frac{\hat{\eta}_{t', k_{t'-1}}}{1-\beta}\Delta \h_{t'}^{k_{t'-1}}\right) - \left(\hat{\w}_{t'} - \hat{\eta}_{t', k_{t'-1}}{\Delta \w}_{t'}^{k_{t'-1}}\right) \\
    &= \hat{\y}_{t'} - \hat{\w}_{t'}  - \frac{\hat{\eta}_{t', k_{t'-1}}}{1-\beta}\Delta \h_{t'}^{k_{t'-1}} + \hat{\eta}_{t', k_{t'-1}}{\Delta \w}_{t'}^{k_{t'-1}} \\
    &= -\frac{\beta}{1-\beta}\hat{\u}_{t'} - \frac{\beta}{1-\beta}\hat{\eta}_{t', k_{t'-1}}{\Delta \u}_{t'}^{k_{t'-1}} \\
    &= -\frac{\beta}{1-\beta}\hat{\u}_{t'+1}
\end{align*}
  \end{itemize}

  We can conclude that $\hat{\y}_t - \hat{\w}_t = -\frac{\beta}{1-\beta}\hat{\u}_{t}$ is true for any $t \geq 1$.           
  \end{proof}

    \begin{lemma}
    With Assumption~\ref{assu:gradient}, the gap between the local and global parameters can be bounded for $t \geq 1$:
  \begin{align*}
  \sum_{s=0}^{S-1}\mathbb{E}\left\|\tilde{\w}_{t,s}^{k_{t-1}} - \w_t\right\|^2 \leq \frac{\gamma^2S(S-1)\sigma^2}{(1-\beta)^2} +\frac{\gamma^2S(S-1)}{(1-\beta)^2}\sum_{s=0}^{S-1}\mathbb{E}\left\| \nabla F(\tilde{\w}_{t,s}^{k_{t-1}})\right\|^2
  \end{align*}
  \end{lemma}
  \begin{proof}
  For $s=0$, we have $\mathbb{E}\left\|\tilde{\w}_{t,s}^{k_{t-1}} - \w_t\right\|^2=0$.
  
  For $s \geq 1$, we have that
 \begin{align*}
  \mathbb{E}\left\|\tilde{\w}_{t,s}^{k_{t-1}} - \w_t\right\|^2  &= \mathbb{E}\left\| - \gamma\sum_{i=0}^{s-1}\frac{1-\beta^{s-i}}{1-\beta}\nabla f(\tilde{\w}_{t,i}^{k_{t-1}}; \xi)\right\|^2 \\
  & \leq 2\mathbb{E}\left\| \gamma\sum_{i=0}^{s-1}\frac{1-\beta^{s-i}}{1-\beta}\left[\nabla f(\tilde{\w}_{t,i}^{k_{t-1}}; \xi)-\nabla F(\tilde{\w}_{t,i}^{k_{t-1}})\right]\right\|^2 \\
  &~~~~+2\mathbb{E}\left\| \gamma\sum_{i=0}^{s-1}\frac{1-\beta^{s-i}}{1-\beta}\nabla F(\tilde{\w}_{t,i}^{k_{t-1}})\right\|^2\\
  & \leq \frac{2\gamma^2s\sigma^2}{(1-\beta)^2} +\frac{2\gamma^2s}{(1-\beta)^2}\sum_{i=0}^{s-1}\mathbb{E}\left\| \nabla F(\tilde{\w}_{t,i}^{k_{t-1}})\right\|^2
  \end{align*}
  Summing up the above equations from $s=0$ to $S-1$, we get that
  \begin{align*}
  \sum_{s=0}^{S-1}\mathbb{E}\left\|\tilde{\w}_{t,s}^{k_{t-1}} - \w_t\right\|^2 &\leq \frac{\gamma^2S(S-1)\sigma^2}{(1-\beta)^2} +\frac{\gamma^2S(S-1)}{(1-\beta)^2}\sum_{s=0}^{S-1}\mathbb{E}\left\| \nabla F(\tilde{\w}_{t,s}^{k_{t-1}})\right\|^2
  \end{align*}
  \end{proof}

    \begin{lemma} 
    For any $t \geq 1$, $\hat{\u}_{t}$ can be formulated as follows:    
    \begin{align*}
      \hat{\u}_t &= \beta^{\lfloor \frac{t+K-2}{K} \rfloor}\left(\sum_{k \in [K]} \hat{\eta}_{0,k}{\Delta \u}_{0}^{k}\right) + \sum_{q=1}^{\lfloor \frac{t+K-2}{K} \rfloor}\beta^{\lfloor \frac{t+K-2}{K} \rfloor-q}\left(\sum_{j=(q-1)K+1}^{\min\{qK,t-1\}}\hat{\eta}_{j, k_{j-1}}{\Delta \u}_{j}^{k_{j-1}}\right).
    \end{align*}
  \end{lemma}
    \begin{proof}
    It's straightforward to get this conclusion from the definition of the sequence $\hat{\u}_t$.
    \end{proof}

\begin{lemma} 
  With Assumption \ref{assu:gradient}, letting
\begin{equation} 
\eta_t = 
\begin{cases} 
\frac{1}{K} & \tau_t \leq 2K, \\ 
\frac{1}{\tau_t} & \tau_t > 2K, 
\end{cases} 
\end{equation} the difference between $\hat{\y}_t$ and $\hat{\w}_t$  can be bounded as follows:
  \begin{align*}
    \sum_{t=1}^{T-1} \mathbb{E}\left(\hat{\eta}_{t, k_{t-1}} \left\|\hat{\y}_t - \hat{\w}_t \right\|^2\right) 
    &\leq  \frac{2\gamma^2\beta^2\sigma^2}{(1-\beta)^4K}\left(\sum_{k \in [K]} \hat{\eta}_{0,k} +\sum_{t=1}^{T-1} \hat{\eta}_{t,k_{t-1}} \right) \\
    &~~~~+ \frac{2\gamma^2\beta^2S}{(1-\beta)^4}\left[\sum_{k \in [K]}\left( \hat{\eta}_{0,k} \sum_{s=0}^{S-1}\mathbb{E}\left\|\nabla F(\tilde{\w}_{0,s}^k)\right\|^2\right)+\sum_{t=1}^{T-1} \left(\hat{\eta}_{t,k_{t-1}}\sum_{s=0}^{S-1}\mathbb{E}\left\|\nabla F(\tilde{\w}_{t, s}^{k_{t-1}})\right\|^2\right)\right].
\end{align*}
  \end{lemma}
  \begin{proof}
      \begin{align*}
      \hat{\u}_t 
      &= \sum_{k \in [K]}\beta^{\lfloor \frac{t+K-2}{K} \rfloor} \hat{\eta}_{0,k}\left({\Delta \u}_{0}^{k}-\gamma \sum_{s=0}^{S-1}\beta^{S-1-s}\nabla F(\tilde{\w}_{0,s}^k)\right) \\
      &~~~~+ \sum_{q=1}^{\lfloor \frac{t+K-2}{K} \rfloor}\sum_{j=(q-1)K+1}^{\min\{qK,t-1\}}\beta^{\lfloor \frac{t+K-2}{K} \rfloor-q}\hat{\eta}_{j, k_{j-1}}\left({\Delta \u}_{j}^{k_{j-1}}  -\gamma \sum_{s=0}^{S-1}\beta^{S-1-s}\nabla F(\tilde{\w}_{j, s}^{k_{j-1}})\right)\\
      &~~~~+ \gamma\sum_{k \in [K]} \beta^{\lfloor \frac{t+K-2}{K} \rfloor}\hat{\eta}_{0,k}\sum_{s=0}^{S-1}\beta^{S-1-s}\nabla F(\tilde{\w}_{0,s}^k) + \gamma\sum_{q=1}^{\lfloor \frac{t+K-2}{K} \rfloor} \sum_{j=(q-1)K+1}^{\min\{qK,t-1\}}\beta^{\lfloor \frac{t+K-2}{K} \rfloor-q}\hat{\eta}_{j, k_{j-1}}\sum_{s=0}^{S-1}\beta^{S-1-s}\nabla F(\tilde{\w}_{j, s}^{k_{j-1}})
    \end{align*}

        \begin{align*}
      \mathbb{E}\left\|\hat{\u}_t \right\|^2&\leq 2\mathbb{E}\left\|\sum_{k \in [K]}\beta^{\lfloor \frac{t+K-2}{K} \rfloor} \hat{\eta}_{0,k}\left({\Delta \u}_{0}^{k}-\gamma \sum_{s=0}^{S-1}\beta^{S-1-s}\nabla F(\tilde{\w}_{0,s}^k)\right) \right.\\
      &~~~~\left.+ \sum_{q=1}^{\lfloor \frac{t+K-2}{K} \rfloor}\sum_{j=(q-1)K+1}^{\min\{qK,t-1\}}\beta^{\lfloor \frac{t+K-2}{K} \rfloor-q}\hat{\eta}_{j, k_{j-1}}\left({\Delta \u}_{j}^{k_{j-1}}  -\gamma \sum_{s=0}^{S-1}\beta^{S-1-s}\nabla F(\tilde{\w}_{j, s}^{k_{j-1}})\right)\right\|^2\\
      &~~~~+ 2\gamma^2\mathbb{E}\left\|\sum_{k \in [K]} \beta^{\lfloor \frac{t+K-2}{K} \rfloor}\hat{\eta}_{0,k}\sum_{s=0}^{S-1}\beta^{S-1-s}\nabla F(\tilde{\w}_{0,s}^k) \right. \\
      &~~~~\left. + \sum_{q=1}^{\lfloor \frac{t+K-2}{K} \rfloor} \sum_{j=(q-1)K+1}^{\min\{qK,t-1\}}\beta^{\lfloor \frac{t+K-2}{K} \rfloor-q}\hat{\eta}_{j, k_{j-1}}\sum_{s=0}^{S-1}\beta^{S-1-s}\nabla F(\tilde{\w}_{j, s}^{k_{j-1}})\right\|^2
    \end{align*}
    \begin{align*}
    &\mathbb{E}\left\|\sum_{k \in [K]}\beta^{\lfloor \frac{t+K-2}{K} \rfloor} \hat{\eta}_{0,k}\left({\Delta \u}_{0}^{k}-\gamma \sum_{s=0}^{S-1}\beta^{S-1-s}\nabla F(\tilde{\w}_{0,s}^k)\right) \right.\\
      &~~~~\left.+ \sum_{q=1}^{\lfloor \frac{t+K-2}{K} \rfloor}\sum_{j=(q-1)K+1}^{\min\{qK,t-1\}}\beta^{\lfloor \frac{t+K-2}{K} \rfloor-q}\hat{\eta}_{j, k_{j-1}}\left({\Delta \u}_{j}^{k_{j-1}}  -\gamma \sum_{s=0}^{S-1}\beta^{S-1-s}\nabla F(\tilde{\w}_{j, s}^{k_{j-1}})\right)\right\|^2\\
      &\leq \sum_{k \in [K]}\beta^{2\lfloor \frac{t+K-2}{K} \rfloor} \hat{\eta}_{0,k}^2\mathbb{E}\left\|{\Delta \u}_{0}^{k}-\gamma \sum_{s=0}^{S-1}\beta^{S-1-s}\nabla F(\tilde{\w}_{0,s}^k)\right\|^2 \\
      &~~~~+ \sum_{q=1}^{\lfloor \frac{t+K-2}{K} \rfloor}\sum_{j=(q-1)K+1}^{\min\{qK,t-1\}}\beta^{2\lfloor \frac{t+K-2}{K} \rfloor-2q}\hat{\eta}_{j, k_{j-1}}^2\mathbb{E}\left\|{\Delta \u}_{j}^{k_{j-1}}  -\gamma \sum_{s=0}^{S-1}\beta^{S-1-s}\nabla F(\tilde{\w}_{j, s}^{k_{j-1}})\right\|^2\\
      &\leq \frac{\gamma^2\sigma^2}{1-\beta}\left[\sum_{k \in [K]}\beta^{2\lfloor \frac{t+K-2}{K} \rfloor} \hat{\eta}_{0,k}^2 + \sum_{q=1}^{\lfloor \frac{t+K-2}{K} \rfloor}\sum_{j=(q-1)K+1}^{\min\{qK,t-1\}}\beta^{2\lfloor \frac{t+K-2}{K} \rfloor-2q}\hat{\eta}_{j, k_{j-1}}^2\right]\\
    \end{align*}
    
  Let $z_t = \sum_{k \in [K]}\beta^{\lfloor \frac{t+K-2}{K} \rfloor} + \sum_{q=1}^{\lfloor \frac{t+K-2}{K} \rfloor}\sum_{j=(q-1)K+1}^{\min\{qK,t-1\}}\beta^{\lfloor \frac{t+K-2}{K} \rfloor-q}$, then we have

      \begin{align*}
  &\mathbb{E}\left\|\sum_{k \in [K]} \beta^{\lfloor \frac{t+K-2}{K} \rfloor}\hat{\eta}_{0,k}\sum_{s=0}^{S-1}\beta^{S-1-s}\nabla F(\tilde{\w}_{0,s}^k) + \sum_{q=1}^{\lfloor \frac{t+K-2}{K} \rfloor} \sum_{j=(q-1)K+1}^{\min\{qK,t-1\}}\beta^{\lfloor \frac{t+K-2}{K} \rfloor-q}\hat{\eta}_{j, k_{j-1}}\sum_{s=0}^{S-1}\beta^{S-1-s}\nabla F(\tilde{\w}_{j, s}^{k_{j-1}})\right\|^2\\
  &=z_t^2\mathbb{E}\left\|\sum_{k \in [K]} \frac{\beta^{\lfloor \frac{t+K-2}{K} \rfloor}}{z_t}\hat{\eta}_{0,k}\sum_{s=0}^{S-1}\beta^{S-1-s}\nabla F(\tilde{\w}_{0,s}^k)+ \sum_{q=1}^{\lfloor \frac{t+K-2}{K} \rfloor} \sum_{j=(q-1)K+1}^{\min\{qK,t-1\}}\frac{\beta^{\lfloor \frac{t+K-2}{K} \rfloor-q}}{z_t}\hat{\eta}_{j, k_{j-1}}\sum_{s=0}^{S-1}\beta^{S-1-s}\nabla F(\tilde{\w}_{j, s}^{k_{j-1}})\right\|^2\\
  &\leq z_t\sum_{k \in [K]} \beta^{\lfloor \frac{t+K-2}{K} \rfloor}\hat{\eta}_{0,k}^2\mathbb{E}\left\|\sum_{s=0}^{S-1}\beta^{S-1-s}\nabla F(\tilde{\w}_{0,s}^k)\right\|^2 \\
  &~~~~+z_t \sum_{q=1}^{\lfloor \frac{t+K-2}{K} \rfloor} \sum_{j=(q-1)K+1}^{\min\{qK,t-1\}}\beta^{\lfloor \frac{t+K-2}{K} \rfloor-q}\hat{\eta}_{j, k_{j-1}}^2\mathbb{E}\left\|\sum_{s=0}^{S-1}\beta^{S-1-s}\nabla F(\tilde{\w}_{j, s}^{k_{j-1}})\right\|^2\\
  &\leq \frac{KS}{1-\beta}\left[\sum_{k \in [K]} \beta^{\lfloor \frac{t+K-2}{K} \rfloor}\hat{\eta}_{0,k}^2\sum_{s=0}^{S-1}\mathbb{E}\left\|\nabla F(\tilde{\w}_{0,s}^k)\right\|^2 +\sum_{q=1}^{\lfloor \frac{t+K-2}{K} \rfloor} \sum_{j=(q-1)K+1}^{\min\{qK,t-1\}}\beta^{\lfloor \frac{t+K-2}{K} \rfloor-q}\hat{\eta}_{j, k_{j-1}}^2\sum_{s=0}^{S-1}\mathbb{E}\left\|\nabla F(\tilde{\w}_{j, s}^{k_{j-1}})\right\|^2\right]\\
  \end{align*}

    \begin{align*}
      \mathbb{E}\left\|\hat{\u}_t \right\|^2 &\leq 
      \frac{2\gamma^2\sigma^2}{1-\beta}\left[\sum_{k \in [K]}\beta^{2\lfloor \frac{t+K-2}{K} \rfloor} \hat{\eta}_{0,k}^2 + \sum_{q=1}^{\lfloor \frac{t+K-2}{K} \rfloor}\sum_{j=(q-1)K+1}^{\min\{qK,t-1\}}\beta^{2\lfloor \frac{t+K-2}{K} \rfloor-2q}\hat{\eta}_{j, k_{j-1}}^2\right] \\
      &~~~~+\frac{2\gamma^2KS}{1-\beta}\sum_{k \in [K]} \beta^{\lfloor \frac{t+K-2}{K} \rfloor}\hat{\eta}_{0,k}^2\sum_{s=0}^{S-1}\mathbb{E}\left\|\nabla F(\tilde{\w}_{0,s}^k)\right\|^2 \\
      &~~~~+\frac{2\gamma^2KS}{1-\beta}\sum_{q=1}^{\lfloor \frac{t+K-2}{K} \rfloor} \sum_{j=(q-1)K+1}^{\min\{qK,t-1\}}\beta^{\lfloor \frac{t+K-2}{K} \rfloor-q}\hat{\eta}_{j, k_{j-1}}^2\sum_{s=0}^{S-1}\mathbb{E}\left\|\nabla F(\tilde{\w}_{j, s}^{k_{j-1}})\right\|^2 \\
    \sum_{t=1}^{T-1} \mathbb{E}\left(\hat{\eta}_{t, k_{t-1}} \left\|\hat{\y}_t - \hat{\w}_t \right\|^2\right) &\leq \frac{\beta^2}{{(1-\beta)}^2 K}\sum_{t=1}^{T-1}\mathbb{E}\left\|\hat{\u}_{t} \right\|^2 \\
    & \leq      \frac{2\gamma^2\sigma^2\beta^2}{(1-\beta)^3K}\sum_{t=1}^{T-1}\left[\sum_{k \in [K]}\beta^{2\lfloor \frac{t+K-2}{K} \rfloor} \hat{\eta}_{0,k}^2 + \sum_{q=1}^{\lfloor \frac{t+K-2}{K} \rfloor}\sum_{j=(q-1)K+1}^{\min\{qK,t-1\}}\beta^{2\lfloor \frac{t+K-2}{K} \rfloor-2q}\hat{\eta}_{j, k_{j-1}}^2\right] \\
    &~~~~+ \frac{2\gamma^2\beta^2 S}{(1-\beta)^3}\sum_{t=1}^{T-1}\left[\sum_{k \in [K]} \beta^{\lfloor \frac{t+K-2}{K} \rfloor}\hat{\eta}_{0,k}^2\sum_{s=0}^{S-1}\mathbb{E}\left\|\nabla F(\tilde{\w}_{0,s}^k)\right\|^2 \right] \\
    &~~~~+ \frac{2\gamma^2\beta^2 S}{(1-\beta)^3}\sum_{t=1}^{T-1}\left[\sum_{q=1}^{\lfloor \frac{t+K-2}{K} \rfloor} \sum_{j=(q-1)K+1}^{\min\{qK,t-1\}}\beta^{\lfloor \frac{t+K-2}{K} \rfloor-q}\hat{\eta}_{j, k_{j-1}}^2\sum_{s=0}^{S-1}\mathbb{E}\left\|\nabla F(\tilde{\w}_{j, s}^{k_{j-1}})\right\|^2\right] \\
    & \leq      \frac{2\gamma^2\beta^2\sigma^2}{(1-\beta)^4}\left(\sum_{k \in [K]} \hat{\eta}_{0,k}^2 +\sum_{t=1}^{T-1} \hat{\eta}_{t,k_{t-1}}^2 \right) \\
    &~~~~+ \frac{2\gamma^2\beta^2KS}{(1-\beta)^4}\left[\sum_{k \in [K]} \hat{\eta}_{0,k}^2 \sum_{s=0}^{S-1}\mathbb{E}\left\|\nabla F(\tilde{\w}_{0,s}^k)\right\|^2+\sum_{t=1}^{T-1} \hat{\eta}_{t,k_{t-1}}^2\sum_{s=0}^{S-1}\mathbb{E}\left\|\nabla F(\tilde{\w}_{t, s}^{k_{t-1}})\right\|^2\right] \\
    & \leq      \frac{2\gamma^2\beta^2\sigma^2}{(1-\beta)^4K}\left(\sum_{k \in [K]} \hat{\eta}_{0,k} +\sum_{t=1}^{T-1} \hat{\eta}_{t,k_{t-1}} \right) \\
    &~~~~+ \frac{2\gamma^2\beta^2S}{(1-\beta)^4}\left[\sum_{k \in [K]} \hat{\eta}_{0,k} \sum_{s=0}^{S-1}\mathbb{E}\left\|\nabla F(\tilde{\w}_{0,s}^k)\right\|^2+\sum_{t=1}^{T-1} \hat{\eta}_{t,k_{t-1}}\sum_{s=0}^{S-1}\mathbb{E}\left\|\nabla F(\tilde{\w}_{t, s}^{k_{t-1}})\right\|^2\right] 
\end{align*}
  \end{proof}
  \begin{lemma} 
  With Assumption~\ref{assu:gradient}, letting
\begin{equation}
\eta_t = 
\begin{cases} 
\frac{1}{K} & \tau_t \leq 2K, \\ 
\frac{1}{\tau_t} & \tau_t > 2K, 
\end{cases} 
\end{equation} the difference between $\hat{\w}_t$ and $\w_t$  can be bounded as follows:
  \begin{align*}
  \sum_{t=1}^{T-1} \mathbb{E}\left(\hat{\eta}_{t, k_{t-1}} \left\|\hat{\w}_t - \w_t \right\|^2\right)&\leq \frac{4\gamma^2S\sigma^2}{(1-\beta)^2K}\left(\sum_{k \in [K]} \hat{\eta}_{0,k} + \sum_{t=1}^{T-2} {\hat{\eta}}_{t,k_{t-1}}\right)\\
  &~~~~+\frac{4\gamma^2S}{(1-\beta)^2}\left[\sum_{k \in [K]} \left(\hat{\eta}_{0,k}\sum_{s=0}^{S-1}\mathbb{E}\left\|\nabla F(\tilde{\w}_{0,s}^k)\right\|^2\right) + \sum_{t=1}^{T-2} \left({\hat{\eta}}_{t,k_{t-1}}\sum_{s=0}^{S-1}\mathbb{E}\left\|\nabla F(\tilde{\w}_{t, s}^{k_{t-1}})\right\|^2\right)\right]
  \end{align*}
  \end{lemma}
  \begin{proof}
  \begin{align*}
  \hat{\w}_t - \w_t &=-\sum_{k \in [K], k \neq k_{t-1}}{\hat{\eta}}_{ite(t-1, k), k} \left[{\Delta \w}_{ite(t-1, k)}^k + \frac{\beta(1-\beta^{\lceil \frac{t-1}{K} \rceil - \lceil \frac{ite(t-1, k)}{K}\rceil})}{1-\beta}{\Delta \u}_{ite(t-1, k)}^k\right] \\
  &=-\gamma \sum_{k \in [K], k \neq k_{t-1}}{\hat{\eta}}_{ite(t-1, k), k} \sum_{s=0}^{S-1}\frac{1-\beta^{S-s+\lceil \frac{t-1}{K} \rceil - \lceil \frac{ite(t-1, k)}{K}\rceil}}{1-\beta}\nabla f(\tilde{\w}_{ite(t-1, k),s}^k;\xi) 
  \end{align*}
    \begin{align*}
  &\mathbb{E}\left\|\hat{\w}_t - \w_t\right\|^2\\
  &\leq 2\gamma^2 \underbrace{\mathbb{E}\left\|\sum_{k \in [K], k \neq k_{t-1}}{\hat{\eta}}_{ite(t-1, k), k} \sum_{s=0}^{S-1}\frac{1-\beta^{S-s+\lceil \frac{t-1}{K} \rceil - \lceil \frac{ite(t-1, k)}{K}\rceil}}{1-\beta}\left[\nabla f(\tilde{\w}_{ite(t-1, k),s}^k;\xi)-\nabla F(\tilde{\w}_{ite(t-1, k),s}^k)\right] \right\|^2}_{\text{\textcircled{1}}} \\
  &~~~~+2\gamma^2 \underbrace{\mathbb{E}\left\| \sum_{k \in [K], k \neq k_{t-1}}{\hat{\eta}}_{ite(t-1, k), k} \sum_{s=0}^{S-1}\frac{1-\beta^{S-s+\lceil \frac{t-1}{K} \rceil - \lceil \frac{ite(t-1, k)}{K}\rceil}}{1-\beta}\nabla F(\tilde{\w}_{ite(t-1, k),s}^k) \right\|^2}_{\text{\textcircled{2}}} 
  \end{align*}
  \begin{align*}
  \text{\textcircled{1}}&\leq \frac{S}{(1-\beta)^2}\sum_{k \in [K],k\neq k_{t-1}}{\hat{\eta}}_{ite(t-1, k), k}^2\sigma^2 \\
  \text{\textcircled{2}}&\leq \frac{KS}{(1-\beta)^2}\sum_{k \in [K],k\neq k_{t-1}}{\hat{\eta}}_{ite(t-1, k), k}^2\sum_{s=0}^{S-1}\mathbb{E}\left\|\nabla F(\tilde{\w}_{ite(t-1, k),s}^k)\right\|^2 \\
  \mathbb{E}\left\|\hat{\w}_t - \w_t\right\|^2 &\leq \frac{2\gamma^2S}{(1-\beta)^2} \sum_{k \in [K],k\neq k_{t-1}}{\hat{\eta}}_{ite(t-1, k), k}^2\left[\sigma^2+K\sum_{s=0}^{S-1}\mathbb{E}\left\|\nabla F(\tilde{\w}_{ite(t-1, k),s}^k)\right\|^2\right]
  \end{align*}
    \begin{align*}
      &\sum_{t=1}^{T-1} \mathbb{E}\left(\hat{\eta}_{t, k_{t-1}} \left\|\hat{\w}_t - \w_t \right\|^2\right)  \\
      & = \frac{2\gamma^2S}{(1-\beta)^2K} \sum_{t=1}^{T-1} \sum_{k \in [K],k\neq k_{t-1}}{\hat{\eta}}_{ite(t-1, k), k}^2\left(\sigma^2+K\sum_{s=0}^{S-1}\mathbb{E}\left\|\nabla F(\tilde{\w}_{ite(t-1, k),s}^k)\right\|^2\right) \\
      & = \frac{2\gamma^2S}{(1-\beta)^2K} \sum_{t=0}^{T-2} \sum_{k \in [K],k\neq k_t}{\hat{\eta}}_{ite(t, k),k}^2\left(\sigma^2+K\sum_{s=0}^{S-1}\mathbb{E}\left\|\nabla F(\tilde{\w}_{ite(t,k), s}^k)\right\|^2\right) \\
      & = \frac{2\gamma^2S}{(1-\beta)^2K} \sum_{j=0}^{T-2}\sum_{t=0}^{T-2} \sum_{k \in [K],k\neq k_t}{\hat{\eta}}_{ite(t, k),k}^2\left[\sigma^2+K\sum_{s=0}^{S-1}\mathbb{E}\left\|\nabla F(\tilde{\w}_{ite(t,k), s}^k)\right\|^2\right]\mathds{1}\left(k \neq k_t\right)\mathds{1}\left(j = ite(t,k)\right) \\
      & = \frac{2\gamma^2S}{(1-\beta)^2K} \sum_{j=1}^{T-2}\sum_{t=0}^{T-2} \sum_{k \in [K],k\neq k_t}{\hat{\eta}}_{ite(t, k),k}^2\left(\sigma^2+K\sum_{s=0}^{S-1}\mathbb{E}\left\|\nabla F(\tilde{\w}_{ite(t,k), s}^k)\right\|^2\right)\mathds{1}\left(k \neq k_t\right)\mathds{1}\left(j = ite(t,k)\right) \\
      &~~~~+ \frac{2\gamma^2S}{(1-\beta)^2K}\sum_{t=0}^{T-2} \sum_{k \in [K],k\neq k_t}{\hat{\eta}}_{ite(t, k),k}^2\left(\sigma^2+K\sum_{s=0}^{S-1}\mathbb{E}\left\|\nabla F(\tilde{\w}_{ite(t,k), s}^k)\right\|^2\right)\mathds{1}\left(k \neq k_t\right)\mathds{1}\left(0 = ite(t,k)\right) 
      \end{align*}
      \begin{align*}
      & \leq \frac{2\gamma^2S}{(1-\beta)^2K} \sum_{k \in [K]} {\hat{\eta}}_{0,k}^2\left(\sigma^2+K\sum_{s=0}^{S-1}\mathbb{E}\left\|\nabla F(\tilde{\w}_{0,s}^k)\right\|^2\right)\left(next(0,k)-0\right) \\      
      &~~~~+ \frac{2\gamma^2S}{(1-\beta)^2K} \sum_{j=1}^{T-2} {\hat{\eta}}_{j, k_{j-1}}^2\left(\sigma^2+K\sum_{s=0}^{S-1}\mathbb{E}\left\|\nabla F(\tilde{\w}_{j, s}^{k_{j-1}})\right\|^2\right)\left(next(j,k_{j-1})-j\right) \\
      & \leq \frac{4\gamma^2S}{(1-\beta)^2K} \sum_{k \in [K]} {\hat{\eta}}_{0,k}\left(\sigma^2+K\sum_{s=0}^{S-1}\mathbb{E}\left\|\nabla F(\tilde{\w}_{0,s}^k)\right\|^2\right)+\frac{4\gamma^2S}{(1-\beta)^2K} \sum_{j=1}^{T-2} {\hat{\eta}}_{j, k_{j-1}}\left(\sigma^2+K\sum_{s=0}^{S-1}\mathbb{E}\left\|\nabla F(\tilde{\w}_{j, s}^{k_{j-1}})\right\|^2\right) \\
      & \leq \frac{4\gamma^2S\sigma^2}{(1-\beta)^2K}\left(\sum_{k \in [K]} \hat{\eta}_{0,k}+\sum_{t=1}^{T-2} {\hat{\eta}}_{t,k_{t-1}}\right)+\frac{4\gamma^2S}{(1-\beta)^2}\left(\sum_{k \in [K]} \hat{\eta}_{0,k}\sum_{s=0}^{S-1}\mathbb{E}\left\|\nabla F(\tilde{\w}_{0,s}^k)\right\|^2+\sum_{t=1}^{T-2} {\hat{\eta}}_{t,k_{t-1}}\sum_{s=0}^{S-1}\mathbb{E}\left\|\nabla F(\tilde{\w}_{t, s}^{k_{t-1}})\right\|^2 \right)
  \end{align*}
    \end{proof}

    \begin{theorem}
  With Assumptions~\ref{assu:gradient},~\ref{assu:smooth} and~\ref{assu:lowerbounded}, letting
\begin{equation}
\eta_t = 
\begin{cases} 
\frac{1}{K} & \tau_t \leq 2K, \\ 
\frac{1}{\tau_t} & \tau_t > 2K, 
\end{cases} 
\end{equation}
and $16LS\gamma \leq (1-\beta)^2$, OrLoMo has the following convergence rate:
\begin{align*}
\mathbb{E}\left\|\nabla F(\bar{\w}_T)\right\|^2 \leq \frac{4K(1-\beta)\left(F(\w_0)- F^*\right)}{\gamma ST}+\frac{4\gamma L\sigma^2}{K(1-\beta)^2}+\frac{\gamma^2 L^2(S-1)\sigma^2}{(1-\beta)^2}, 
  \end{align*}
where $ \mathbb{E}\left\|\nabla F(\bar{\w}_T)\right\|^2 = \frac{1}{\sum_{k \in [K]}\hat{\eta}_{0,k}+\sum_{t=1}^{T-1}\hat{\eta}_{t,k_{t-1}}}\left[\left(\sum_{k \in [K]} \hat{\eta}_{0,k}\right)\left\|\nabla F(\w_0)\right\|^2+\sum_{t=1}^{T-1}\left(\hat{\eta}_{t,k_{t-1}}\mathbb{E}\left\|\nabla F(\w_t)\right\|^2\right)\right]$.
\end{theorem}
  \begin{proof}

\begin{align*}
  \mathbb{E}F(\hat{\y}_{t+1}) \leq \mathbb{E} F(\hat{\y}_t) - \frac{\hat{\eta}_{t,k_{t-1}}}{1-\beta}\mathbb{E}\langle \nabla F(\hat{\y}_t), \Delta \h_{t}^{k_{t-1}} \rangle + \frac{L\hat{\eta}_{t,k_{t-1}}^2}{2(1-\beta)^2}\mathbb{E}\left\|\Delta \h_{t}^{k_{t-1}}\right\|^2
  \end{align*}
  \begin{align*}
  \frac{L\hat{\eta}_{t,k_{t-1}}^2}{2(1-\beta)^2}\mathbb{E}\left\|\Delta \h_{t}^{k_{t-1}}\right\|^2 
  &= \frac{\gamma^2L\hat{\eta}_{t,k_{t-1}}^2}{2(1-\beta)^2} \mathbb{E} \left\|\sum_{s=0}^{S-1}\nabla f(\tilde{\w}_{t, s}^{k_{t-1}}; \xi)\right\|^2 \\
  &\leq \frac{\gamma^2L\hat{\eta}_{t,k_{t-1}}^2}{(1-\beta)^2} \mathbb{E} \left\| \sum_{s=0}^{S-1}\left(\nabla f(\tilde{\w}_{t, s}^{k_{t-1}}; \xi)-\nabla F(\tilde{\w}_{t, s}^{k_{t-1}})\right)\right\|^2 \\
  &~~~~+ \frac{\gamma^2L\hat{\eta}_{t,k_{t-1}}^2}{(1-\beta)^2}\mathbb{E} \left\|\sum_{s=0}^{S-1}\nabla F(\tilde{\w}_{t, s}^{k_{t-1}})\right\|^2\\
  &\leq \frac{\gamma^2LS\hat{\eta}_{t,k_{t-1}}^2\sigma^2}{(1-\beta)^2}  + \frac{\gamma^2L\hat{\eta}_{t,k_{t-1}}^2S}{(1-\beta)^2}\sum_{s=0}^{S-1}\mathbb{E} \left\|\nabla F(\tilde{\w}_{t, s}^{k_{t-1}})\right\|^2\\
  - \frac{\hat{\eta}_{t,k_{t-1}}}{1-\beta}\mathbb{E}\langle \nabla F(\hat{\y}_t), \Delta \h_{t}^{k_{t-1}} \rangle  
  &= - \frac{\hat{\eta}_{t,k_{t-1}}}{1-\beta}\mathbb{E}\langle \nabla F(\hat{\y}_t), \gamma\sum_{s=0}^{S-1}\nabla F(\tilde{\w}_{t, s}^{k_{t-1}}) \rangle \\
  & = \underbrace{- \frac{\hat{\eta}_{t,k_{t-1}}}{1-\beta}\sum_{s=0}^{S-1}\mathbb{E}\langle \nabla F(\hat{\w}_t), \gamma\nabla F(\tilde{\w}_{t, s}^{k_{t-1}}) \rangle}_{\text{\textcircled{1}}} \\
  &~~~~  \underbrace{- \frac{\hat{\eta}_{t,k_{t-1}}}{1-\beta}\sum_{s=0}^{S-1}\mathbb{E}\langle \nabla F(\hat{\y}_t)-\nabla F(\hat{\w}_t), \gamma\nabla F(\tilde{\w}_{t, s}^{k_{t-1}}) \rangle}_{\text{\textcircled{2}}} 
  \end{align*}
  \begin{align*}
\text{\textcircled{1}}& =-\frac{\hat{\eta}_{t,k_{t-1}}\gamma}{1-\beta}\sum_{s=0}^{S-1}\mathbb{E}\langle \nabla F(\w_t), \nabla F(\tilde{\w}_{t, s}^{k_{t-1}}) - \nabla F(\w_t)\rangle - \frac{\hat{\eta}_{t,k_{t-1}}\gamma S}{1-\beta}\mathbb{E}\left\|\nabla F(\w_t)\right\|^2\\
& \leq  \frac{\gamma S \hat{\eta}_{t,k_{t-1}}}{2(1-\beta)}\mathbb{E}\left\|\nabla F(\w_t)\right\|^2 + \frac{ \hat{\eta}_{t,k_{t-1}}\gamma}{2(1-\beta)} \sum_{s=0}^{S-1}\mathbb{E}\left\|\nabla F(\tilde{\w}_{t, s}^{k_{t-1}}) - \nabla F(\w_t) \right\|^2- \frac{ \gamma\hat{\eta}_{t,k_{t-1}}}{2(1-\beta)} \sum_{s=0}^{S-1}\mathbb{E}\left\|\nabla F(\tilde{\w}_{t, s}^{k_{t-1}}) \right\|^2 \\
&~~~~ - \frac{ \gamma S \hat{\eta}_{t,k_{t-1}}}{1-\beta}\mathbb{E}\left\|\nabla F(\w_t)\right\|^2\\
& \leq  - \frac{ \gamma S \hat{\eta}_{t,k_{t-1}}}{2(1-\beta)}\mathbb{E}\left\|\nabla F(\w_t)\right\|^2 + \frac{ \hat{\eta}_{t,k_{t-1}}\gamma}{2(1-\beta)} \sum_{s=0}^{S-1}\mathbb{E}\left\|\nabla F(\tilde{\w}_{t, s}^{k_{t-1}}) - \nabla F(\w_t) \right\|^2 - \frac{ \gamma\hat{\eta}_{t,k_{t-1}}}{2(1-\beta)} \sum_{s=0}^{S-1}\mathbb{E}\left\|\nabla F(\tilde{\w}_{t, s}^{k_{t-1}}) \right\|^2 \\
& \leq  - \frac{ \gamma S \hat{\eta}_{t,k_{t-1}}}{2(1-\beta)}\mathbb{E}\left\|\nabla F(\w_t)\right\|^2 + \frac{ \hat{\eta}_{t,k_{t-1}}\gamma L^2}{2(1-\beta)} \sum_{s=0}^{S-1}\mathbb{E}\left\|\tilde{\w}_{t, s}^{k_{t-1}} - \w_t \right\|^2 - \frac{ \gamma\hat{\eta}_{t,k_{t-1}}}{2(1-\beta)} \sum_{s=0}^{S-1}\mathbb{E}\left\|\nabla F(\tilde{\w}_{t, s}^{k_{t-1}}) \right\|^2 
\end{align*}
  \begin{align*}
  \text{\textcircled{2}} &\leq \frac{\hat{\eta}_{t,k_{t-1}}\gamma S}{1-\beta} \mathbb{E}\left\|\nabla F(\hat{\y}_t) -\nabla F(\w_t) \right\|^2 + \frac{\hat{\eta}_{t,k_{t-1}}\gamma}{4(1-\beta)}\sum_{s=0}^{S-1}\mathbb{E}\left\| \nabla F(\tilde{\w}_{t, s}^{k_{t-1}})\right\|^2 \\
  &\leq \frac{L^2 \hat{\eta}_{t,k_{t-1}}\gamma S}{1-\beta} \mathbb{E}\left\|\hat{\y}_t-\w_t \right\|^2 + \frac{\hat{\eta}_{t,k_{t-1}}\gamma}{4(1-\beta)}\sum_{s=0}^{S-1}\mathbb{E}\left\| \nabla F(\tilde{\w}_{t, s}^{k_{t-1}})\right\|^2
  \end{align*}

  \begin{align*}
  \mathbb{E}F(\hat{\y}_{t+1}) 
&\leq \mathbb{E} F(\hat{\y}_t) - \frac{ \gamma S \hat{\eta}_{t,k_{t-1}}}{2(1-\beta)}\mathbb{E}\left\|\nabla F(\w_t)\right\|^2 + \frac{\gamma L^2\hat{\eta}_{t,k_{t-1}}}{2(1-\beta)} \sum_{s=0}^{S-1}\mathbb{E}\left\|\tilde{\w}_{t, s}^{k_{t-1}} - \w_t \right\|^2 \\
  &~~~~- \frac{ \gamma\hat{\eta}_{t,k_{t-1}}}{2(1-\beta)} \sum_{s=0}^{S-1}\mathbb{E}\left\|\nabla F(\tilde{\w}_{t, s}^{k_{t-1}}) \right\|^2  + \frac{L^2 \gamma S\hat{\eta}_{t,k_{t-1}}}{1-\beta} \mathbb{E}\left\|\hat{\y}_t-\w_t \right\|^2 + \frac{\gamma\hat{\eta}_{t,k_{t-1}}}{4(1-\beta)}\sum_{s=0}^{S-1}\mathbb{E}\left\| \nabla F(\tilde{\w}_{t, s}^{k_{t-1}})\right\|^2 \\
  &~~~~+\frac{\gamma^2LS\hat{\eta}_{t,k_{t-1}}^2\sigma^2}{(1-\beta)^2}  + \frac{\gamma^2LS\hat{\eta}_{t,k_{t-1}}^2}{(1-\beta)^2}\sum_{s=0}^{S-1}\mathbb{E} \left\|\nabla F(\tilde{\w}_{t, s}^{k_{t-1}})\right\|^2 \\
  &\leq \mathbb{E} F(\hat{\y}_t) - \frac{ \gamma S \hat{\eta}_{t,k_{t-1}}}{2(1-\beta)}\mathbb{E}\left\|\nabla F(\w_t)\right\|^2 + \frac{ \gamma L^2\hat{\eta}_{t,k_{t-1}}}{2(1-\beta)} \sum_{s=0}^{S-1}\mathbb{E}\left\|\tilde{\w}_{t, s}^{k_{t-1}} - \w_t \right\|^2 +\frac{\gamma^2LS\hat{\eta}_{t,k_{t-1}}^2\sigma^2}{(1-\beta)^2}\\
  &~~~~+ \frac{L^2 \gamma S\hat{\eta}_{t,k_{t-1}}}{1-\beta} \mathbb{E}\left\|\hat{\y}_t-\w_t \right\|^2 +\left[\frac{\gamma^2LS\hat{\eta}_{t,k_{t-1}}^2}{(1-\beta)^2}-\frac{\gamma\hat{\eta}_{t,k_{t-1}}}{4(1-\beta)}\right]\sum_{s=0}^{S-1}\mathbb{E}\left\| \nabla F(\tilde{\w}_{t, s}^{k_{t-1}})\right\|^2  
  \end{align*}
Letting $S \leq \frac{1-\beta}{8\gamma L}$, then we have that
  \begin{align*} 
  \mathbb{E}F(\hat{\y}_{t+1}) 
  &\leq \mathbb{E} F(\hat{\y}_t) - \frac{ \gamma S \hat{\eta}_{t,k_{t-1}}}{2(1-\beta)}\mathbb{E}\left\|\nabla F(\w_t)\right\|^2 + \frac{ \gamma L^2\hat{\eta}_{t,k_{t-1}}}{2(1-\beta)} \sum_{s=0}^{S-1}\mathbb{E}\left\|\tilde{\w}_{t, s}^{k_{t-1}} - \w_t \right\|^2 +\frac{\gamma^2LS\hat{\eta}_{t,k_{t-1}}^2\sigma^2}{(1-\beta)^2}\\
  &~~~~+ \frac{2L^2 \gamma S\hat{\eta}_{t,k_{t-1}}}{1-\beta} \mathbb{E}\left\|\hat{\y}_t-\hat{\w}_t \right\|^2+ \frac{2L^2 \gamma S\hat{\eta}_{t,k_{t-1}}}{1-\beta} \mathbb{E}\left\|\hat{\w}_t-\w_t \right\|^2 -\frac{\gamma\hat{\eta}_{t,k_{t-1}}}{8(1-\beta)}\sum_{s=0}^{S-1}\mathbb{E}\left\| \nabla F(\tilde{\w}_{t, s}^{k_{t-1}})\right\|^2  
  \end{align*}
    Summing up the above equation from $t=1$ to $T-1$, we can get that
    \begin{equation} \label{eq:descent1}
    \begin{aligned}
  \mathbb{E}F(\hat{\y}_{T}) 
  &\leq \mathbb{E} F(\hat{\y}_1) - \frac{\gamma S }{2(1-\beta)}\sum_{t=1}^{T-1}\left(\hat{\eta}_{t,k_{t-1}}\mathbb{E}\left\|\nabla F(\w_t)\right\|^2\right) + \frac{\gamma L^2}{2(1-\beta)} \sum_{t=1}^{T-1}\left(\hat{\eta}_{t,k_{t-1}}\sum_{s=0}^{S-1}\mathbb{E}\left\|\tilde{\w}_{t, s}^{k_{t-1}} - \w_t \right\|^2 \right)\\
  &~~~~+\frac{\gamma^2LS\sigma^2}{(1-\beta)^2}\sum_{t=1}^{T-1}\hat{\eta}_{t,k_{t-1}}^2 -\frac{\gamma}{8(1-\beta)}\sum_{t=1}^{T-1}\left(\hat{\eta}_{t,k_{t-1}}\sum_{s=0}^{S-1}\mathbb{E}\left\| \nabla F(\tilde{\w}_{t, s}^{k_{t-1}})\right\|^2  \right)\\
  &~~~~+ \frac{2L^2 \gamma S}{1-\beta} \sum_{t=1}^{T-1}\left(\hat{\eta}_{t,k_{t-1}}\mathbb{E}\left\|\hat{\y}_t-\hat{\w}_t \right\|^2\right)+ \frac{2L^2 \gamma S}{1-\beta} \sum_{t=1}^{T-1}\left( \hat{\eta}_{t,k_{t-1}}\mathbb{E}\left\|\hat{\w}_t-\w_t \right\|^2\right)
    \end{aligned}
    \end{equation}

    \begin{align*}
    \mathbb{E} F(\hat{\y}_1) &\leq  F(\w_0) - \frac{1}{1-\beta}\mathbb{E}\langle\nabla F(\w_0), \sum_{k \in [K]} \hat{\eta}_{0,k}\Delta \h_0^k\rangle + \frac{L}{2{(1-\beta)}^2} \mathbb{E} \left\|\sum_{k \in [K]} \hat{\eta}_{0,k}\Delta \h_0^k\right\|^2 
  \end{align*}
    \begin{align*}
  \frac{L}{2(1-\beta)^2}\mathbb{E} \left\|\sum_{k \in [K]} \hat{\eta}_{0,k}\Delta \h_0^k\right\|^2
  &= \frac{\gamma^2L}{2(1-\beta)^2} \mathbb{E} \left\|\sum_{k \in [K]} \hat{\eta}_{0,k}\sum_{s=0}^{S-1}\nabla f(\tilde{\w}_{0,s}^k; \xi)\right\|^2 \\
  &\leq \frac{\gamma^2L}{(1-\beta)^2} \mathbb{E} \left\|\sum_{k \in [K]} \hat{\eta}_{0,k}\sum_{s=0}^{S-1}\left(\nabla f(\tilde{\w}_{0,s}^k; \xi)-\nabla F(\tilde{\w}_{0,s}^k)\right)\right\|^2 \\
  &~~~~+\frac{\gamma^2L}{(1-\beta)^2} \mathbb{E}\left\|\sum_{k \in [K]} \hat{\eta}_{0,k}\sum_{s=0}^{S-1}\nabla F(\tilde{\w}_{0,s}^k)\right\|^2  \\
  &\leq \frac{\gamma^2LS\sigma^2}{(1-\beta)^2}\sum_{k \in [K]}\hat{\eta}_{0,k}^2  + 
   \frac{\gamma^2LKS}{(1-\beta)^2} \sum_{k \in [K]} \left(\hat{\eta}_{0,k}^2\sum_{s=0}^{S-1}\mathbb{E}\left\|\nabla F(\tilde{\w}_{0,s}^k)\right\|^2 \right)\\
  &\leq \frac{\gamma^2LS\sigma^2}{(1-\beta)^2}\sum_{k \in [K]}\hat{\eta}_{0,k}^2  + 
   \frac{\gamma^2LS}{(1-\beta)^2} \sum_{k \in [K]} \left(\hat{\eta}_{0,k}\sum_{s=0}^{S-1}\mathbb{E}\left\|\nabla F(\tilde{\w}_{0,s}^k)\right\|^2 \right)
  \end{align*}
  \begin{align*}
  - \frac{1}{1-\beta}\mathbb{E}\langle\nabla F(\w_0), \sum_{k \in [K]} \hat{\eta}_{0,k}\Delta \h_0^k\rangle 
  &=- \frac{\gamma}{1-\beta}\sum_{k \in [K]} \hat{\eta}_{0,k}\sum_{s=0}^{S-1}\mathbb{E}\langle\nabla F(\w_0), \nabla F(\tilde{\w}_{0,s}^k)\rangle  \\
  &=- \frac{\gamma}{1-\beta}\sum_{k \in [K]} \hat{\eta}_{0,k}\sum_{s=0}^{S-1}\mathbb{E}\langle\nabla F(\w_0), \nabla F(\tilde{\w}_{0,s}^k)-\nabla F(\w_0)\rangle \\
  &~~~~- \frac{\gamma S}{1-\beta}\sum_{k \in [K]} \hat{\eta}_{0,k}\left\|\nabla F(\w_0) \right\|^2\\
  &= -\frac{\gamma S}{2(1-\beta)}\sum_{k \in [K]} \hat{\eta}_{0,k}\left\|\nabla F(\w_0)\right\|^2-\frac{\gamma}{2(1-\beta)}\sum_{k \in [K]} \hat{\eta}_{0,k}\sum_{s=0}^{S-1}\mathbb{E}\left\|\nabla F(\tilde{\w}_{0,s}^k)\right\|^2 \\
  &~~~~+\frac{\gamma L^2}{2(1-\beta)}\sum_{k \in [K]} \hat{\eta}_{0,k}\sum_{s=0}^{S-1}\mathbb{E}\left\|\tilde{\w}_{0,s}^k-\w_0\right\|^2
  \end{align*}
  
    \begin{align*}
    \mathbb{E} F(\hat{\y}_1) &\leq  F(\w_0)  -\frac{\gamma S}{2(1-\beta)}\sum_{k \in [K]} \hat{\eta}_{0,k}\left\|\nabla F(\w_0)\right\|^2+\left[\frac{\gamma^2LS}{(1-\beta)^2}-\frac{\gamma}{2(1-\beta)}\right]\sum_{k \in [K]} \hat{\eta}_{0,k}\sum_{s=0}^{S-1}\mathbb{E}\left\|\nabla F(\tilde{\w}_{0,s}^k)\right\|^2 \\
  &~~~~+\frac{\gamma L^2}{2(1-\beta)}\sum_{k \in [K]} \hat{\eta}_{0,k}\sum_{s=0}^{S-1}\mathbb{E}\left\|\tilde{\w}_{0,s}^k-\w_0\right\|^2 +\frac{\gamma^2LS\sigma^2}{(1-\beta)^2}\sum_{k \in [K]}\hat{\eta}_{0,k}^2 
  \end{align*}
  Letting $8\gamma L S\leq 1-\beta$, then we have that
  \begin{equation}\label{eq:descent2}
  \begin{aligned}
  \mathbb{E} F(\hat{\y}_1) &\leq  F(\w_0)  -\frac{\gamma S}{2(1-\beta)}\sum_{k \in [K]} \hat{\eta}_{0,k}\left\|\nabla F(\w_0)\right\|^2-\frac{\gamma}{8(1-\beta)}\sum_{k \in [K]} \hat{\eta}_{0,k}\sum_{s=0}^{S-1}\mathbb{E}\left\|\nabla F(\tilde{\w}_{0,s}^k)\right\|^2 \\
  &~~~~+\frac{\gamma L^2}{2(1-\beta)}\sum_{k \in [K]} \hat{\eta}_{0,k}\sum_{s=0}^{S-1}\mathbb{E}\left\|\tilde{\w}_{0,s}^k-\w_0\right\|^2 +\frac{\gamma^2LS\sigma^2}{(1-\beta)^2}\sum_{k \in [K]}\hat{\eta}_{0,k}^2
  \end{aligned}
  \end{equation}
Summing up (\ref{eq:descent1}) and (\ref{eq:descent2}), we can get that 
   \begin{align*}
  \mathbb{E}F(\hat{\y}_{T}) 
  &\leq \mathbb{E} F(\w_0) - \frac{\gamma S }{2(1-\beta)}\left[\sum_{k \in [K]} \hat{\eta}_{0,k}\left\|\nabla F(\w_0)\right\|^2+\sum_{t=1}^{T-1}\left(\hat{\eta}_{t,k_{t-1}}\mathbb{E}\left\|\nabla F(\w_t)\right\|^2\right) \right]\\
  &~~~~+ \frac{\gamma L^2}{2(1-\beta)} \left[\sum_{k \in [K]} \hat{\eta}_{0,k}\sum_{s=0}^{S-1}\mathbb{E}\left\|\tilde{\w}_{0,s}^k-\w_0\right\|^2+\sum_{t=1}^{T-1}\left(\hat{\eta}_{t,k_{t-1}}\sum_{s=0}^{S-1}\mathbb{E}\left\|\tilde{\w}_{t, s}^{k_{t-1}} - \w_t \right\|^2 \right)\right]\\
  &~~~~+\frac{\gamma^2LS\sigma^2}{(1-\beta)^2}\left(\sum_{k \in [K]}\hat{\eta}_{0,k}^2+\sum_{t=1}^{T-1}\hat{\eta}_{t,k_{t-1}}^2\right) \\
  &~~~~-\frac{\gamma}{8(1-\beta)}\left[\sum_{k \in [K]} \hat{\eta}_{0,k}\sum_{s=0}^{S-1}\mathbb{E}\left\|\nabla F(\tilde{\w}_{0,s}^k)\right\|^2+\sum_{t=1}^{T-1}\left(\hat{\eta}_{t,k_{t-1}}\sum_{s=0}^{S-1}\mathbb{E}\left\| \nabla F(\tilde{\w}_{t, s}^{k_{t-1}})\right\|^2  \right)\right]\\
  &~~~~+ \frac{2L^2 \gamma S}{1-\beta} \sum_{t=1}^{T-1}\left(\hat{\eta}_{t,k_{t-1}}\mathbb{E}\left\|\hat{\y}_t-\hat{\w}_t \right\|^2\right)+ \frac{2L^2 \gamma S}{1-\beta} \sum_{t=1}^{T-1}\left( \hat{\eta}_{t,k_{t-1}}\mathbb{E}\left\|\hat{\w}_t-\w_t \right\|^2\right) \\
  &\leq \mathbb{E} F(\w_0) - \frac{\gamma S }{2(1-\beta)}\left[\sum_{k \in [K]} \hat{\eta}_{0,k}\left\|\nabla F(\w_0)\right\|^2+\sum_{t=1}^{T-1}\left(\hat{\eta}_{t,k_{t-1}}\mathbb{E}\left\|\nabla F(\w_t)\right\|^2\right) \right]\\
  &~~~~+ \frac{\gamma^3S(S-1) L^2\sigma^2}{2(1-\beta)^3} \left(\sum_{k \in [K]} \hat{\eta}_{0,k}+\sum_{t=1}^{T-1}\hat{\eta}_{t,k_{t-1}}\right)+\frac{\gamma^2LS\sigma^2}{(1-\beta)^2}\left(\sum_{k \in [K]}\hat{\eta}_{0,k}^2+\sum_{t=1}^{T-1}\hat{\eta}_{t,k_{t-1}}^2\right)\\
  &~~~~+ \frac{\gamma^3S(S-1) L^2}{2(1-\beta)^3} \left[\sum_{k \in [K]} \left(\hat{\eta}_{0,k}\sum_{s=0}^{S-1}\mathbb{E}\left\| \nabla F(\tilde{\w}_{0,s}^k)\right\|^2\right)+\sum_{t=1}^{T-1}\left(\hat{\eta}_{t,k_{t-1}}\sum_{s=0}^{S-1}\mathbb{E}\left\| \nabla F(\tilde{\w}_{t,s}^{k_{t-1}})\right\|^2 \right)\right]\\
  &~~~~-\frac{\gamma}{8(1-\beta)}\left[\sum_{k \in [K]} \hat{\eta}_{0,k}\sum_{s=0}^{S-1}\mathbb{E}\left\|\nabla F(\tilde{\w}_{0,s}^k)\right\|^2+\sum_{t=1}^{T-1}\left(\hat{\eta}_{t,k_{t-1}}\sum_{s=0}^{S-1}\mathbb{E}\left\| \nabla F(\tilde{\w}_{t, s}^{k_{t-1}})\right\|^2  \right)\right]\\
  &~~~~+ \frac{4L^2 \gamma^3 S \beta^2\sigma^2}{(1-\beta)^5K} \left(\sum_{k \in [K]} \hat{\eta}_{0,k} +\sum_{t=1}^{T-1} \hat{\eta}_{t,k_{t-1}} \right)+ \frac{8L^2 \gamma^3 S^2 \sigma^2}{(1-\beta)^3K}\left(\sum_{k \in [K]} \hat{\eta}_{0,k}+\sum_{t=1}^{T-2} {\hat{\eta}}_{t,k_{t-1}}\right)\\
    &~~~~+ \frac{4L^2 \gamma^3 \beta^2 S^2}{(1-\beta)^5}\left[\sum_{k \in [K]} \left(\hat{\eta}_{0,k}\sum_{s=0}^{S-1}\mathbb{E}\left\| \nabla F(\tilde{\w}_{0,s}^k)\right\|^2\right)+\sum_{t=1}^{T-1}\left(\hat{\eta}_{t,k_{t-1}}\sum_{s=0}^{S-1}\mathbb{E}\left\| \nabla F(\tilde{\w}_{t,s}^{k_{t-1}})\right\|^2 \right)\right]\\
  &~~~~+\frac{8L^2 \gamma^3 S^2}{(1-\beta)^3}\left[\sum_{k \in [K]} \left(\hat{\eta}_{0,k}\sum_{s=0}^{S-1}\mathbb{E}\left\| \nabla F(\tilde{\w}_{0,s}^k)\right\|^2\right)+\sum_{t=1}^{T-1}\left(\hat{\eta}_{t,k_{t-1}}\sum_{s=0}^{S-1}\mathbb{E}\left\| \nabla F(\tilde{\w}_{t,s}^{k_{t-1}})\right\|^2 \right)\right]
    \end{align*}

 Letting $\gamma \leq \frac{(1-\beta)^2}{16LS}$, then we have that
     \begin{align*}
  F(\hat{\y}_T) &\leq F(\w_0)- \frac{S \gamma }{2(1-\beta)}\left[\left(\sum_{k \in [K]} \hat{\eta}_{0,k}\right)\left\|\nabla F(\w_0)\right\|^2+\sum_{t=1}^{T-1}\left(\hat{\eta}_{t,k_{t-1}}\mathbb{E}\left\|\nabla F(\w_t)\right\|^2\right)\right]\\
  &~~~~+ \frac{\gamma^3S(S-1) L^2\sigma^2}{2(1-\beta)^3} \left(\sum_{k \in [K]} \hat{\eta}_{0,k}+\sum_{t=1}^{T-1}\hat{\eta}_{t,k_{t-1}}\right)+\frac{\gamma^2LS\sigma^2}{(1-\beta)^2}\left(\sum_{k \in [K]}\hat{\eta}_{0,k}^2+\sum_{t=1}^{T-1}\hat{\eta}_{t,k_{t-1}}^2\right)\\
  &~~~~+ \frac{4L^2\gamma^3S\beta^2\sigma^2}{(1-\beta)^5K}\left(\sum_{k \in [K]} \hat{\eta}_{0,k} +\sum_{t=1}^{T-1} \hat{\eta}_{t,k_{t-1}} \right) + \frac{8L^2\gamma^3S^2\sigma^2}{(1-\beta)^3K}\left( \sum_{k \in [K]} \hat{\eta}_{0,k}+\sum_{t=1}^{T-1} {\hat{\eta}}_{t,k_{t-1}}\right)\\
  &\leq F(\w_0)- \frac{S \gamma }{2(1-\beta)}\left[\left(\sum_{k \in [K]} \hat{\eta}_{0,k}\right)\left\|\nabla F(\w_0)\right\|^2+\sum_{t=1}^{T-1}\left(\hat{\eta}_{t,k_{t-1}}\mathbb{E}\left\|\nabla F(\w_t)\right\|^2\right)\right]\\
  &~~~~+\frac{2\gamma^2LS\sigma^2}{(1-\beta)^3K}\left(\sum_{k \in [K]}\hat{\eta}_{0,k}+\sum_{t=1}^{T-1}\hat{\eta}_{t,k_{t-1}}\right) + \frac{\gamma^3 S (S-1) L^2\sigma^2}{2(1-\beta)^3} \left(\sum_{k \in [K]}\hat{\eta}_{0,k} + \sum_{t=1}^{T-1}\hat{\eta}_{t,k_{t-1}} \right)
  \end{align*}
  
 It's easy to verify that 
\begin{align*}
  & \hat{\tau}_{t+1,k} \triangleq t+1-ite(t+1,k)= \left\{
  \begin{aligned}
    & t-ite(t,k)+1=\hat{\tau}_{t,k}+1 & k \neq k_t,\\
    & 0 = \hat{\tau}_{t,k}- \tau_t & k = k_t.
  \end{aligned}
  \right . 
 \end{align*}
Since $\sum_{k \in [K]}\hat{\tau}_{0,k} = 0$, we have $\sum_{k \in [K]}\hat{\tau}_{T,k}+\sum_{t=0}^{T-1}\tau_t=(K-1)T$. 
We get that $\sum_{t=0}^{T-1}\tau_t \leq (K-1)T$.
Thus, at least $\frac{T}{2}$ delays in $\{\tau_t\}_{t \in [T]}$ are smaller than $2K$. $\sum_{t=1}^{T-1}\hat{\eta}_{t, k_{t-1}}+\sum_{k \in [K]}\hat{\eta}_{0, k} \geq \sum_{t \in [T]} \eta_t \geq \frac{T}{2K}$.
    \begin{align*}
\mathbb{E}\left\|\nabla F(\bar{\w}_T)\right\|^2&\leq \frac{2(1-\beta)\left(F(\w_0)- F^*\right)}{S\gamma (\sum_{k \in [K]}\hat{\eta}_{0,k}+\sum_{t=1}^{T-1}\hat{\eta}_{t,k_{t-1}})}+\frac{4\gamma L\sigma^2}{K(1-\beta)^2}+\frac{\gamma^2 L^2(S-1)\sigma^2}{(1-\beta)^2} \\
  &\leq \frac{4K(1-\beta)\left(F(\w_0)- F^*\right)}{\gamma ST}+\frac{4\gamma L\sigma^2}{K(1-\beta)^2}+\frac{\gamma^2 L^2(S-1)\sigma^2}{(1-\beta)^2}, 
  \end{align*} where $ \mathbb{E}\left\|\nabla F(\bar{\w}_T)\right\|^2 =  \frac{\sum_{k \in [K]}\hat{\eta}_{0,k}\left\|\nabla F(\w_0)\right\|^2}{\sum_{k \in [K]}\hat{\eta}_{0,k}+\sum_{t=1}^{T-1}\hat{\eta}_{t,k_{t-1}}}+\frac{\sum_{t=1}^{T-1}\hat{\eta}_{t,k_{t-1}}\mathbb{E}\left\|\nabla F(\w_t)\right\|^2}{\sum_{k \in [K]}\hat{\eta}_{0,k}+\sum_{t=1}^{T-1}\hat{\eta}_{t,k_{t-1}}}$.
  \end{proof}
 
\subsection{Additional Experimental Results}

We evaluate OrLoMo by training SqueezeNet and ResNet20 models on CIFAR10 and CIFAR100 datasets. Experimental details are provided in the Experiments section. Table~\ref{table: loss-resnet20} and Table~\ref{table: loss-squeezenet} show the training loss results. Figure~\ref{fig:resnet20-loss} and Figure~\ref{fig:squeezenet-loss} show the corresponding training loss curves  with respect to wall-clock time.

Additional results for ResNet32 model on Tiny-ImageNet dataset are provided in Table~\ref{table: resnet32-imagenet} and Figure~\ref{fig:resnet32-loss}. In this heterogeneous setting, 25\% of all workers are designated as slow workers, whose average gradient computation time is four times that of the other workers.
  \begin{table*}[!t] \small
  \centering
  \setlength{\tabcolsep}{0.9mm}{  
  \begin{tabular}{c|c|c|cccc|cccc}
    \toprule   
      \multicolumn{3}{c|}{Datasets} & \multicolumn{4}{c|}{CIFAR10}& \multicolumn{4}{c}{CIFAR100}  \\ 
    \midrule
     & Workers & Local Iterations & PRSGDm & AL-SGD &  local OrMo-DA & OrLoMo & PRSGDm & AL-SGD &  local OrMo-DA & OrLoMo \\
    \midrule
     \multirow{6}{*}{\rotatebox{90}{homogeneous}}&\multirow{3}{*}{8}& 8 & \textbf{0.04{\scriptsize$\pm$0.00}} & 0.10{\scriptsize$\pm$0.00} & \textbf{0.04{\scriptsize$\pm$0.00}} & \textbf{0.04{\scriptsize$\pm$0.00}} & 0.64{\scriptsize$\pm$0.01}  & 0.87{\scriptsize$\pm$0.00} & 0.60{\scriptsize$\pm$0.00} & \textbf{0.62{\scriptsize$\pm$0.01}}   \\ 
     & & 16 & \textbf{0.06{\scriptsize$\pm$0.00}} & 0.11{\scriptsize$\pm$0.01} & 0.07{\scriptsize$\pm$0.00} & \textbf{0.06{\scriptsize$\pm$0.00}} & \textbf{0.74{\scriptsize$\pm$0.00}}  & 0.91{\scriptsize$\pm$0.00} & 0.72{\scriptsize$\pm$0.01} & \textbf{0.74{\scriptsize$\pm$0.00}}  \\ 
     & & 32 & 0.09{\scriptsize$\pm$0.00} & 0.12{\scriptsize$\pm$0.00} & 0.19{\scriptsize$\pm$0.01} & \textbf{0.07{\scriptsize$\pm$0.00}} & 0.86{\scriptsize$\pm$0.01}  & 0.90{\scriptsize$\pm$0.00} & 0.95{\scriptsize$\pm$0.02} & \textbf{0.80{\scriptsize$\pm$0.00}}  \\  \cmidrule(r){2-11} 
     &\multirow{3}{*}{16}& 8 & 0.07{\scriptsize$\pm$0.00} & 0.18{\scriptsize$\pm$0.00} & 0.08{\scriptsize$\pm$0.01} & \textbf{0.06{\scriptsize$\pm$0.00}} &  0.79{\scriptsize$\pm$0.01} & 1.06{\scriptsize$\pm$0.01} & 0.76{\scriptsize$\pm$0.00} & \textbf{0.74{\scriptsize$\pm$0.01}}  \\ 
     & & 16 & 0.11{\scriptsize$\pm$0.00} & 0.19{\scriptsize$\pm$0.00} & 0.19{\scriptsize$\pm$0.01} & \textbf{0.09{\scriptsize$\pm$0.00}} & 0.93{\scriptsize$\pm$0.01} & 1.14{\scriptsize$\pm$0.01} & 1.04{\scriptsize$\pm$0.02} & \textbf{0.86{\scriptsize$\pm$0.02}} \\ 
     & & 32 & 0.18{\scriptsize$\pm$0.00} & 0.23{\scriptsize$\pm$0.00} & 0.64{\scriptsize$\pm$0.09} & \textbf{0.16{\scriptsize$\pm$0.01}} & 1.10{\scriptsize$\pm$0.01}  & 1.22{\scriptsize$\pm$0.01} & 1.81{\scriptsize$\pm$0.03} & \textbf{1.08{\scriptsize$\pm$0.01}}  \\ 
    \midrule 
     \multirow{6}{*}{\rotatebox{90}{heterogeneous}}&\multirow{3}{*}{8}& 8 & \textbf{0.04{\scriptsize$\pm$0.00}} & 0.09{\scriptsize$\pm$0.00} & \textbf{0.04{\scriptsize$\pm$0.00}} & \textbf{0.04{\scriptsize$\pm$0.00}} & 0.65{\scriptsize$\pm$0.00} & 0.87{\scriptsize$\pm$0.00} & \textbf{0.61{\scriptsize$\pm$0.00}} & 0.62{\scriptsize$\pm$0.01}    \\
     && 16 & \textbf{0.06{\scriptsize$\pm$0.00}} & 0.11{\scriptsize$\pm$0.00} & 0.07{\scriptsize$\pm$0.00} & \textbf{0.06{\scriptsize$\pm$0.00}}  & 0.74{\scriptsize$\pm$0.01} & 0.89{\scriptsize$\pm$0.01} & \textbf{0.72{\scriptsize$\pm$0.01}} & 0.73{\scriptsize$\pm$0.00} \\
     && 32 & 0.10{\scriptsize$\pm$0.00} & 0.11{\scriptsize$\pm$0.01} & 0.15{\scriptsize$\pm$0.01} & \textbf{0.08{\scriptsize$\pm$0.00}} & 0.86{\scriptsize$\pm$0.01} & 0.90{\scriptsize$\pm$0.00} & 0.95{\scriptsize$\pm$0.02} & \textbf{0.80{\scriptsize$\pm$0.00}} \\ 
    \cmidrule(r){2-11}
     &\multirow{3}{*}{16}& 8 & 0.07{\scriptsize$\pm$0.00} & 0.18{\scriptsize$\pm$0.01} & 0.07{\scriptsize$\pm$0.00} & \textbf{0.06{\scriptsize$\pm$0.00}} & 0.79{\scriptsize$\pm$0.01} &  1.06{\scriptsize$\pm$0.01} & 0.76{\scriptsize$\pm$0.00} & \textbf{0.74{\scriptsize$\pm$0.01}} \\ 
     && 16 & 0.12{\scriptsize$\pm$0.00} & 0.19{\scriptsize$\pm$0.00} & 0.15{\scriptsize$\pm$0.00} & \textbf{0.09{\scriptsize$\pm$0.00}} & 0.93{\scriptsize$\pm$0.01} & 1.13{\scriptsize$\pm$0.00} & 1.01{\scriptsize$\pm$0.01} & \textbf{0.84{\scriptsize$\pm$0.00}} \\ 
     && 32 & 0.18{\scriptsize$\pm$0.01} & 0.22{\scriptsize$\pm$0.01} & 0.47{\scriptsize$\pm$0.02} & \textbf{0.15{\scriptsize$\pm$0.01}} & 1.10{\scriptsize$\pm$0.00} & 1.22{\scriptsize$\pm$0.00} & 1.65{\scriptsize$\pm$0.04} & \textbf{1.07{\scriptsize$\pm$0.00}}  \\ 
    \bottomrule 
  \end{tabular}}
    \caption{Training loss results of the ResNet20 model.} \label{table: loss-resnet20}
\end{table*}

  \begin{table*}[!t] \small
  \centering
  \setlength{\tabcolsep}{0.9mm}{  
  \begin{tabular}{c|c|c|cccc|cccc}
    \toprule   
      \multicolumn{3}{c|}{Datasets} & \multicolumn{4}{c|}{CIFAR10}& \multicolumn{4}{c}{CIFAR100}  \\ 
    \midrule
     & Workers & Local Iterations & PRSGDm & AL-SGD &  local OrMo-DA & OrLoMo & PRSGDm & AL-SGD &  local OrMo-DA & OrLoMo \\
    \midrule
     \multirow{6}{*}{\rotatebox{90}{homogeneous}}&\multirow{3}{*}{8}& 8 & \textbf{0.02{\scriptsize$\pm$0.00}} & 0.06{\scriptsize$\pm$0.00} & \textbf{0.02{\scriptsize$\pm$0.00}} & \textbf{0.02{\scriptsize$\pm$0.00}} & 0.27{\scriptsize$\pm$0.01} & 0.52{\scriptsize$\pm$0.01} & \textbf{0.23{\scriptsize$\pm$0.00}} & 0.24{\scriptsize$\pm$0.01}  \\ 
     & & 16 & \textbf{0.03{\scriptsize$\pm$0.00}} & 0.07{\scriptsize$\pm$0.00} & 0.04{\scriptsize$\pm$0.00} & \textbf{0.03{\scriptsize$\pm$0.00}} & 0.38{\scriptsize$\pm$0.01} & 0.57{\scriptsize$\pm$0.00} & 0.34{\scriptsize$\pm$0.01} & \textbf{0.32{\scriptsize$\pm$0.01}}  \\
     & & 32 & 0.06{\scriptsize$\pm$0.00} & 0.07{\scriptsize$\pm$0.00} & 0.11{\scriptsize$\pm$0.01} & \textbf{0.04{\scriptsize$\pm$0.00}} & 0.53{\scriptsize$\pm$0.01} & 0.60{\scriptsize$\pm$0.01} & 0.68{\scriptsize$\pm$0.02} & \textbf{0.42{\scriptsize$\pm$0.00}}  \\ \cmidrule(r){2-11} 
     &\multirow{3}{*}{16}& 8 & \textbf{0.03{\scriptsize$\pm$0.00}} & 0.13{\scriptsize$\pm$0.00} & 0.04{\scriptsize$\pm$0.00} & \textbf{0.03{\scriptsize$\pm$0.00}} & \textbf{0.37{\scriptsize$\pm$0.00}} & 0.83{\scriptsize$\pm$0.01} & 0.40{\scriptsize$\pm$0.01} & 0.38{\scriptsize$\pm$0.01}  \\
     & & 16 & 0.07{\scriptsize$\pm$0.00} & 0.15{\scriptsize$\pm$0.00} & 0.14{\scriptsize$\pm$0.01} & \textbf{0.05{\scriptsize$\pm$0.00}} & 0.60{\scriptsize$\pm$0.01} & 0.89{\scriptsize$\pm$0.01} & 0.70{\scriptsize$\pm$0.02} & \textbf{0.50{\scriptsize$\pm$0.00}}  \\
     & & 32 & 0.14{\scriptsize$\pm$0.00} & 0.20{\scriptsize$\pm$0.00} & 0.59{\scriptsize$\pm$0.03} & \textbf{0.12{\scriptsize$\pm$0.00}} & 0.87{\scriptsize$\pm$0.01} & 1.05{\scriptsize$\pm$0.02} & 1.76{\scriptsize$\pm$0.02} & \textbf{0.79{\scriptsize$\pm$0.01}}  \\ 
    \midrule 
     \multirow{6}{*}{\rotatebox{90}{heterogeneous}}&\multirow{3}{*}{8}& 8 & \textbf{0.02{\scriptsize$\pm$0.00}} & 0.05{\scriptsize$\pm$0.00} &  \textbf{0.02{\scriptsize$\pm$0.00}} & \textbf{0.02{\scriptsize$\pm$0.00}} & 0.27{\scriptsize$\pm$0.00} & 0.51{\scriptsize$\pm$0.00} & \textbf{0.23{\scriptsize$\pm$0.00}} & 0.24{\scriptsize$\pm$0.00}  \\
     && 16 & \textbf{0.03{\scriptsize$\pm$0.00}} & 0.07{\scriptsize$\pm$0.00} & \textbf{0.03{\scriptsize$\pm$0.00}} & \textbf{0.03{\scriptsize$\pm$0.00}} & 0.38{\scriptsize$\pm$0.00} & 0.58{\scriptsize$\pm$0.00} & \textbf{0.31{\scriptsize$\pm$0.01}} & 0.32{\scriptsize$\pm$0.00}  \\
     && 32 & 0.06{\scriptsize$\pm$0.00} & 0.07{\scriptsize$\pm$0.00} & 0.07{\scriptsize$\pm$0.01} & \textbf{0.04{\scriptsize$\pm$0.00}} & 0.53{\scriptsize$\pm$0.00} & 0.60{\scriptsize$\pm$0.01} & 0.52{\scriptsize$\pm$0.03} & \textbf{0.42{\scriptsize$\pm$0.00}}  \\
    \cmidrule(r){2-11}
     &\multirow{3}{*}{16}& 8 & \textbf{0.03{\scriptsize$\pm$0.00}} & 0.13{\scriptsize$\pm$0.00} & 0.04{\scriptsize$\pm$0.00} & \textbf{0.03{\scriptsize$\pm$0.00}} & 0.37{\scriptsize$\pm$0.00} & 0.84{\scriptsize$\pm$0.01} & \textbf{0.36{\scriptsize$\pm$0.00}} & \textbf{0.36{\scriptsize$\pm$0.00}} \\
     && 16 & 0.07{\scriptsize$\pm$0.00} & 0.15{\scriptsize$\pm$0.00} & 0.11{\scriptsize$\pm$0.00} & \textbf{0.05{\scriptsize$\pm$0.00}} & 0.60{\scriptsize$\pm$0.01} & 0.89{\scriptsize$\pm$0.01} & 0.63{\scriptsize$\pm$0.01} & \textbf{0.49{\scriptsize$\pm$0.00}}  \\
     && 32 & 0.14{\scriptsize$\pm$0.00} & 0.20{\scriptsize$\pm$0.00} & 0.41{\scriptsize$\pm$0.04} & \textbf{0.12{\scriptsize$\pm$0.00}} & 0.88{\scriptsize$\pm$0.01} & 1.04{\scriptsize$\pm$0.01} & 1.46{\scriptsize$\pm$0.03} & \textbf{0.77{\scriptsize$\pm$0.01}}  \\
    \bottomrule 
  \end{tabular}}
    \caption{Training loss results of SqueezeNet model.} \label{table: loss-squeezenet}
\end{table*}
\begin{figure*}[!t]
  \centering
  \subfigure[homogeneous (CIFAR10)]{
    \begin{minipage}[b]{0.24\textwidth}
      \includegraphics[width=1\linewidth]{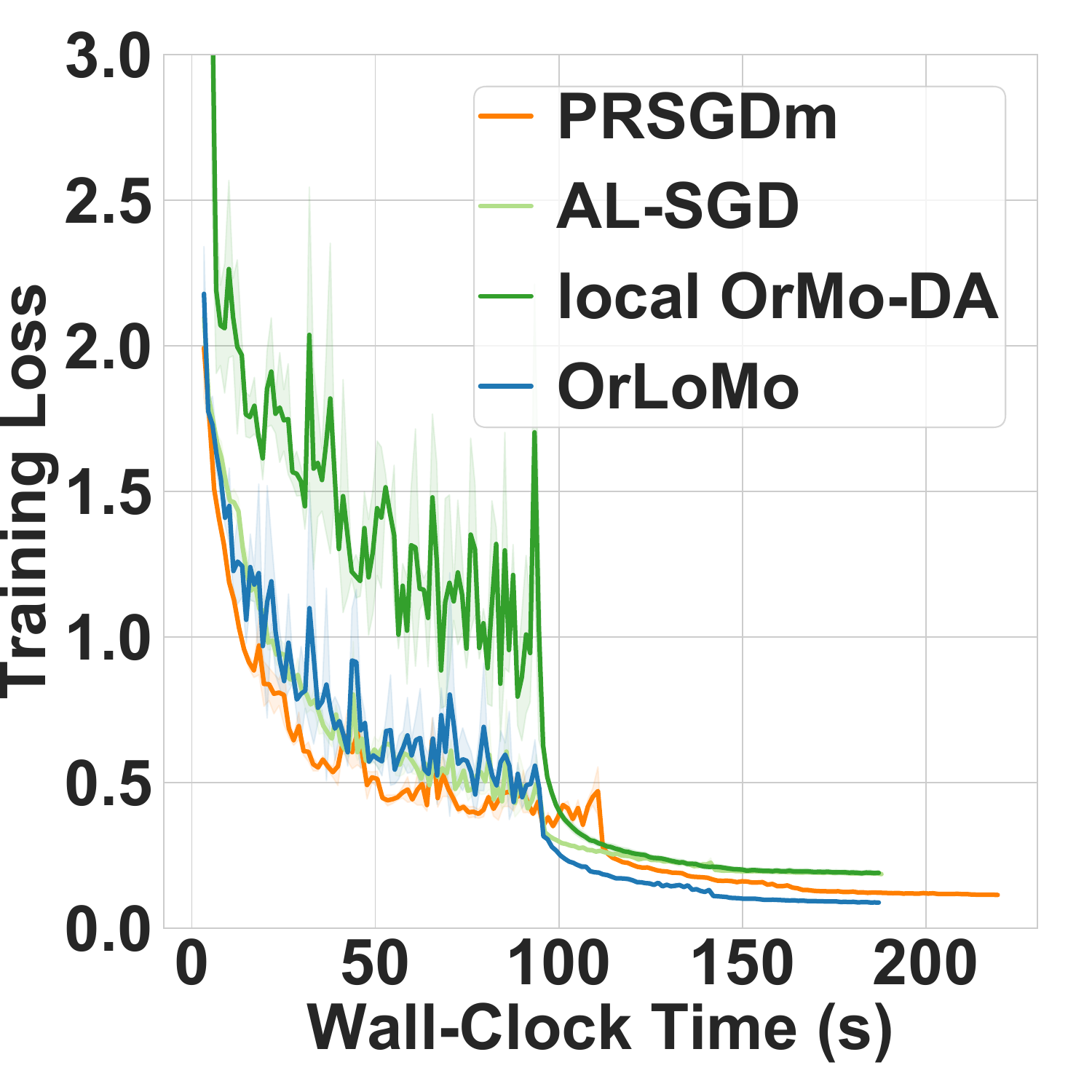}
      \end{minipage}}
  \subfigure[heterogeneous (CIFAR10)]{
    \begin{minipage}[b]{0.24\textwidth}
      \includegraphics[width=1\linewidth]{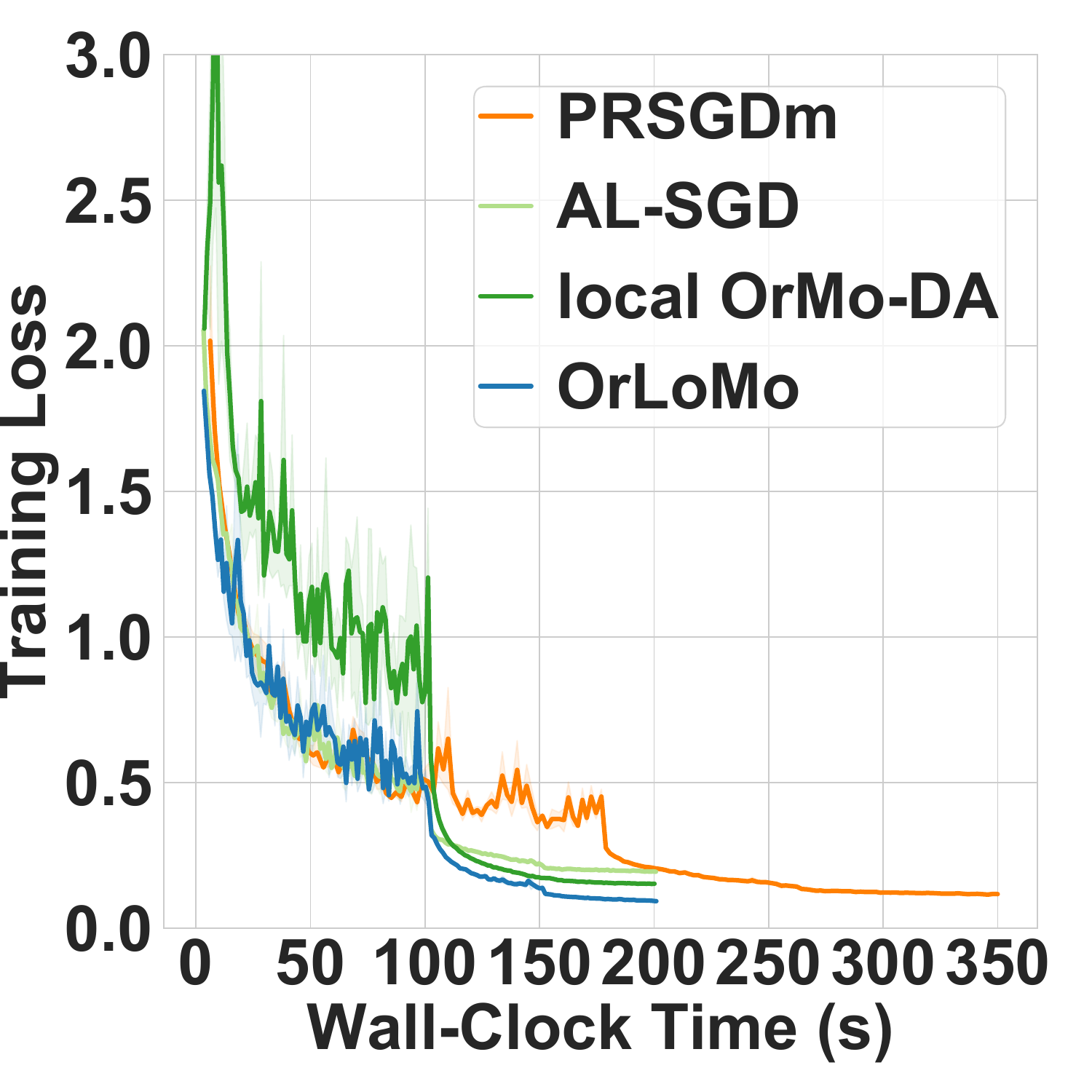}
      \end{minipage}}
  \subfigure[homogeneous (CIFAR100)]{
    \begin{minipage}[b]{0.24\textwidth}
      \includegraphics[width=1\linewidth]{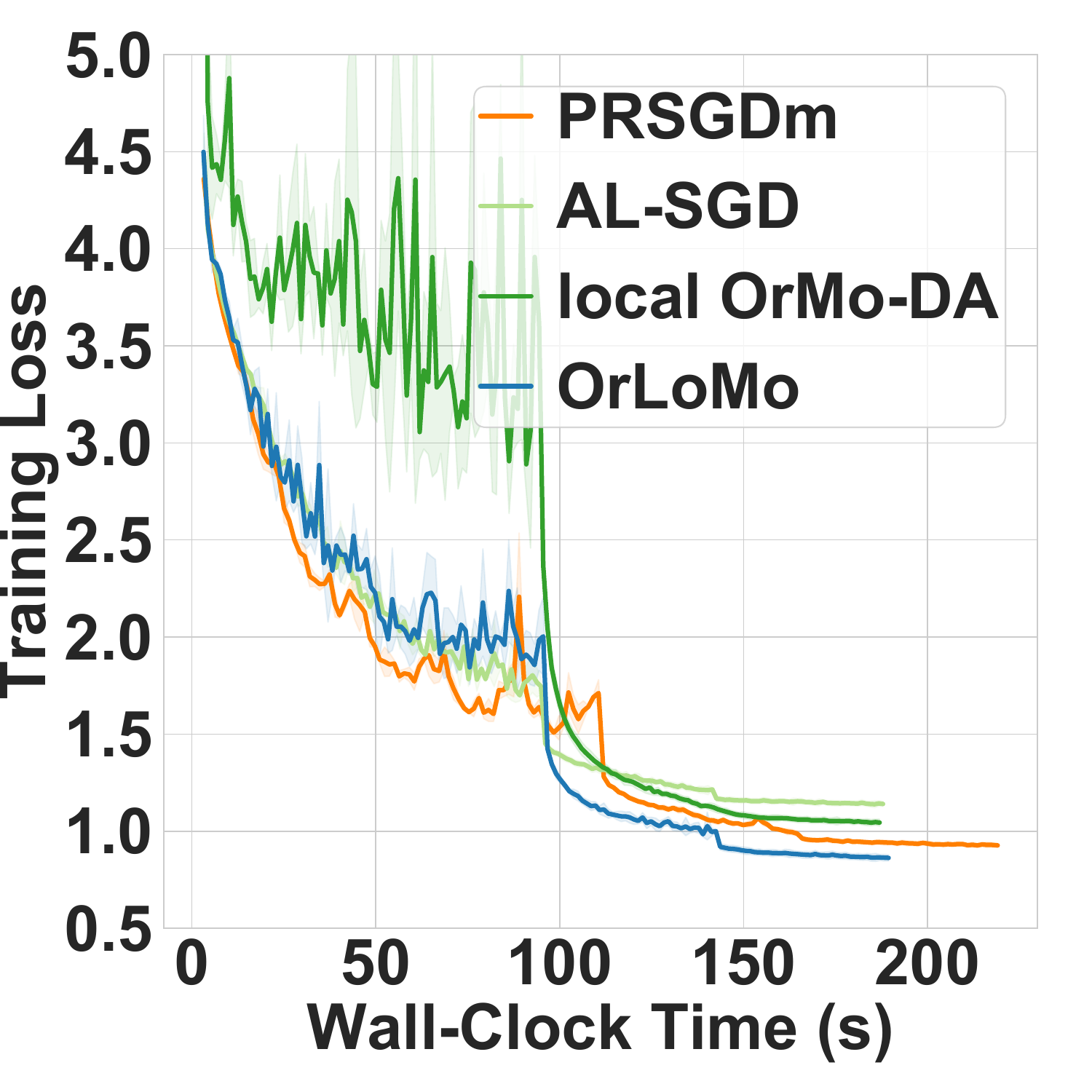}
      \end{minipage}}
  \subfigure[heterogeneous (CIFAR100)]{
    \begin{minipage}[b]{0.24\textwidth}
      \includegraphics[width=1\linewidth]{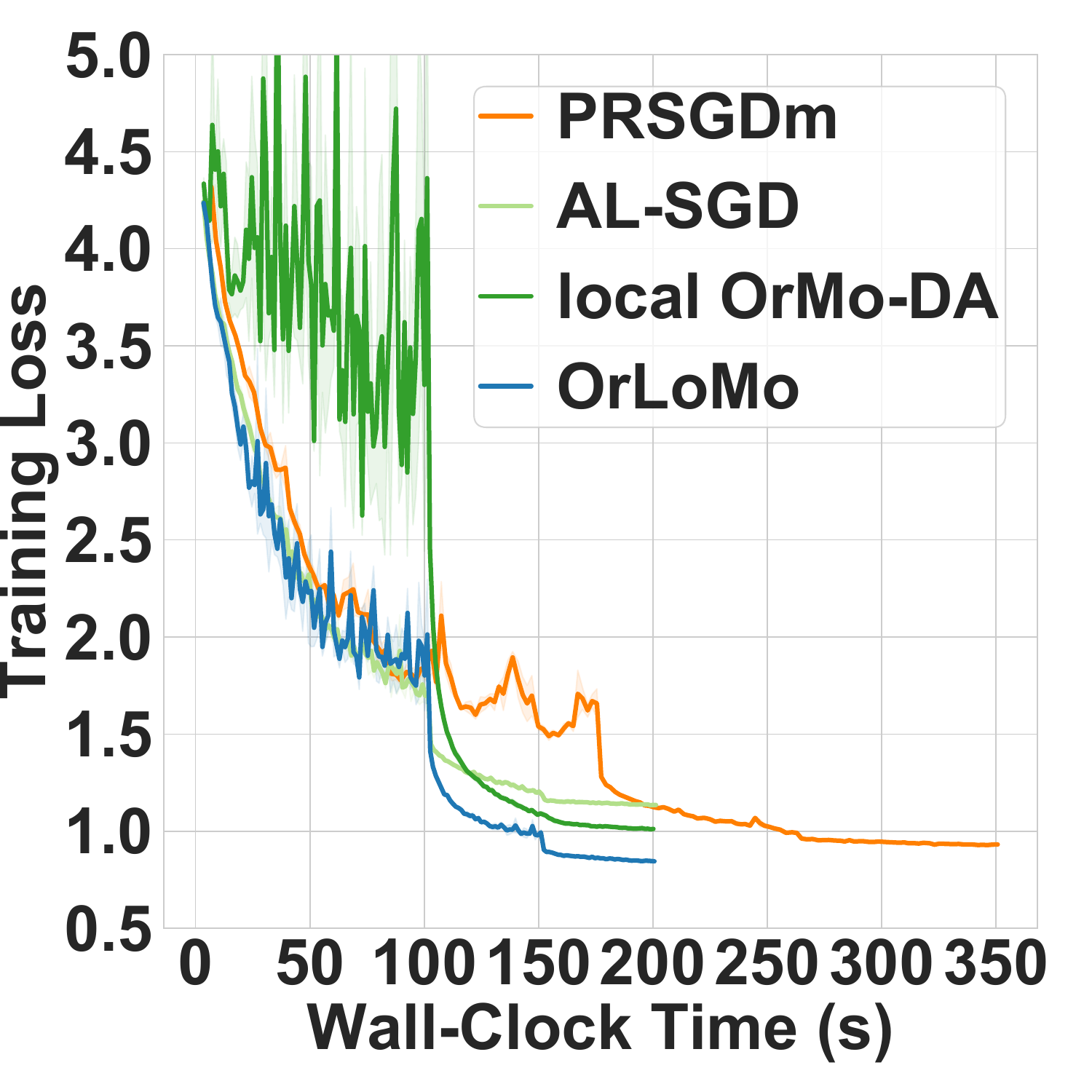}
      \end{minipage}}
\caption{Training loss curves of ResNet20 model when $K=16, S=16$.}\label{fig:resnet20-loss}
\end{figure*}   

\begin{figure*}[!t]
  \centering
  \subfigure[homogeneous (CIFAR10)]{
    \begin{minipage}[b]{0.24\textwidth}
      \includegraphics[width=1\linewidth]{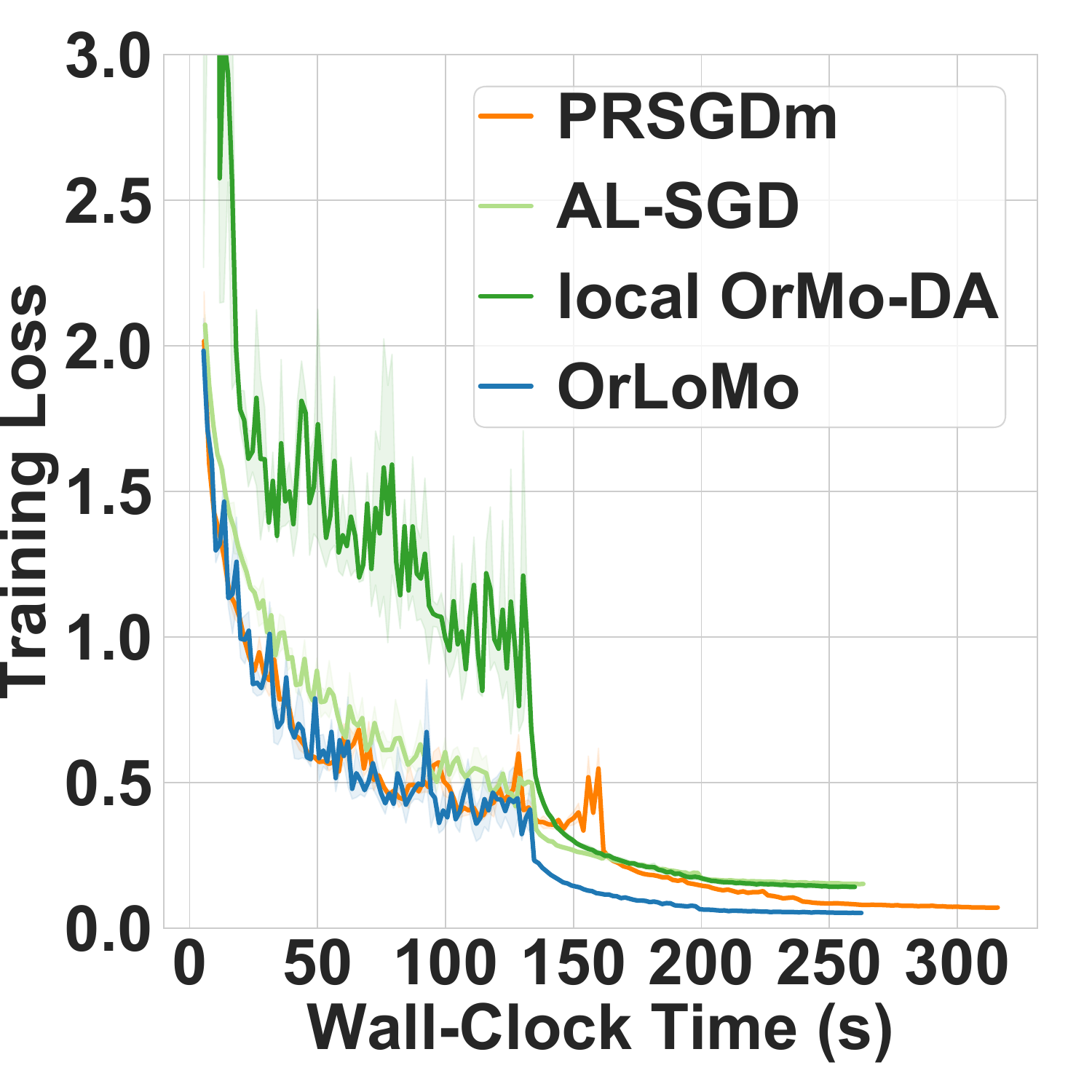}
      \end{minipage}}
  \subfigure[heterogeneous (CIFAR10)]{
    \begin{minipage}[b]{0.24\textwidth}
      \includegraphics[width=1\linewidth]{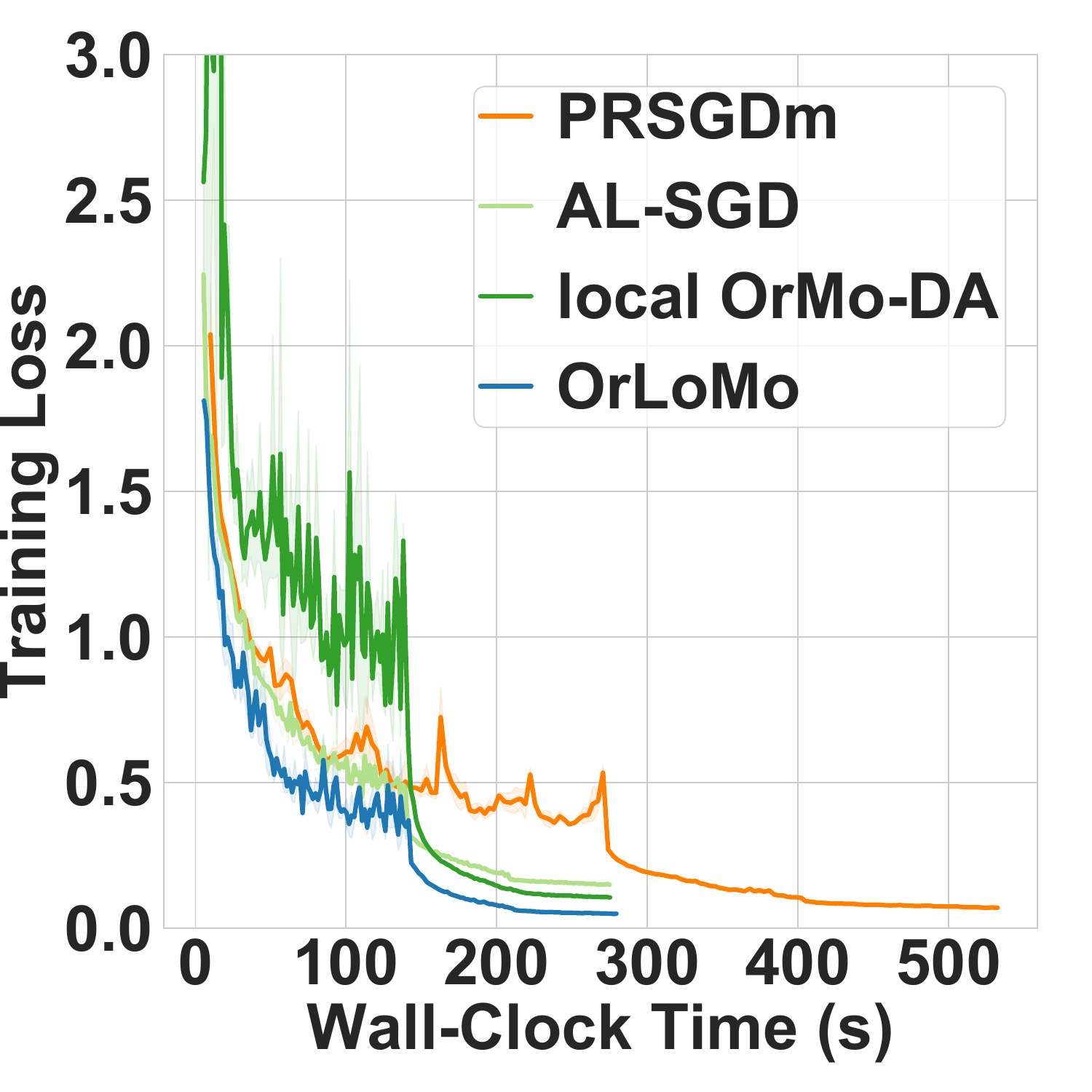}
      \end{minipage}}
  \subfigure[homogeneous (CIFAR100)]{
    \begin{minipage}[b]{0.24\textwidth}
      \includegraphics[width=1\linewidth]{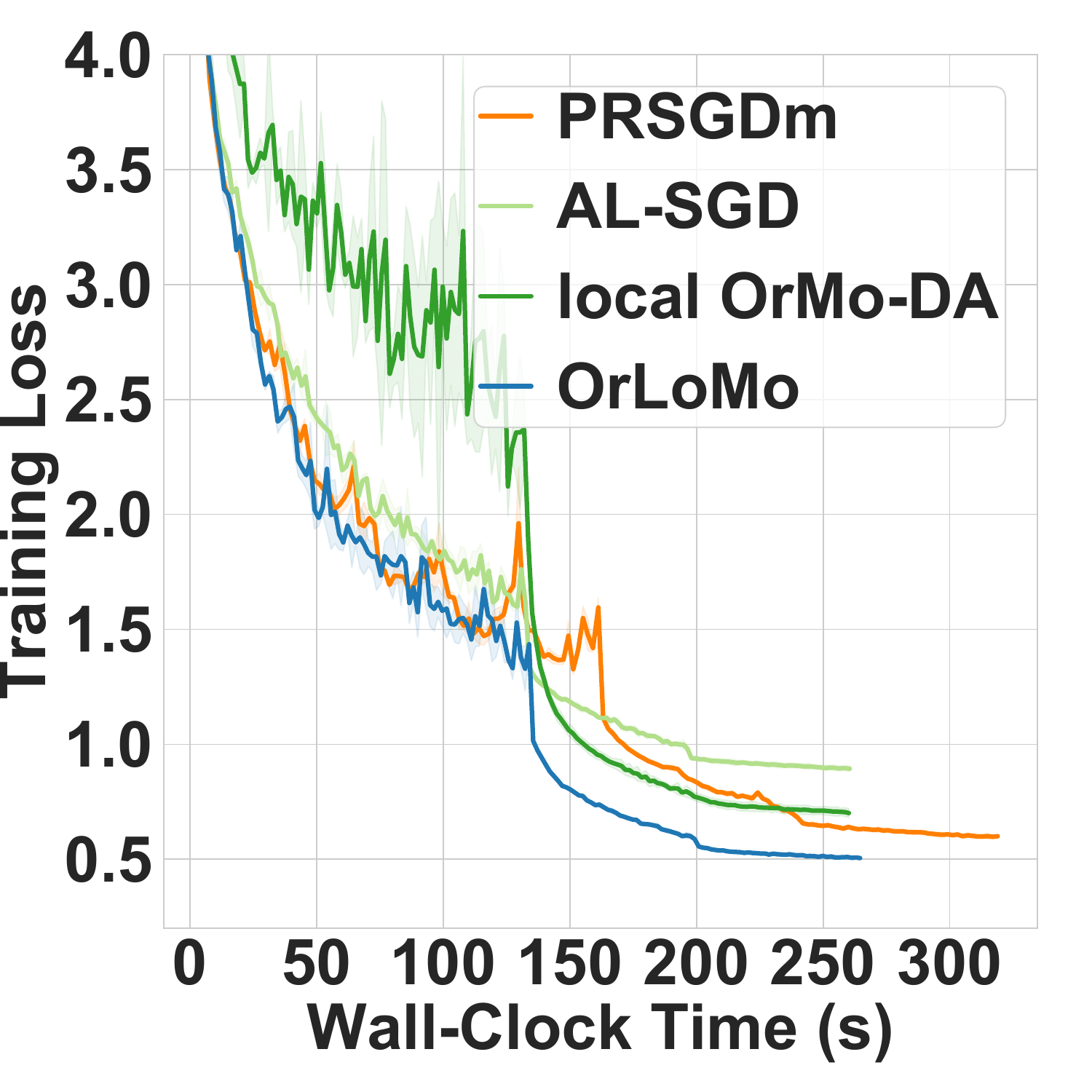}
      \end{minipage}}
  \subfigure[heterogeneous (CIFAR100)]{
    \begin{minipage}[b]{0.24\textwidth}
      \includegraphics[width=1\linewidth]{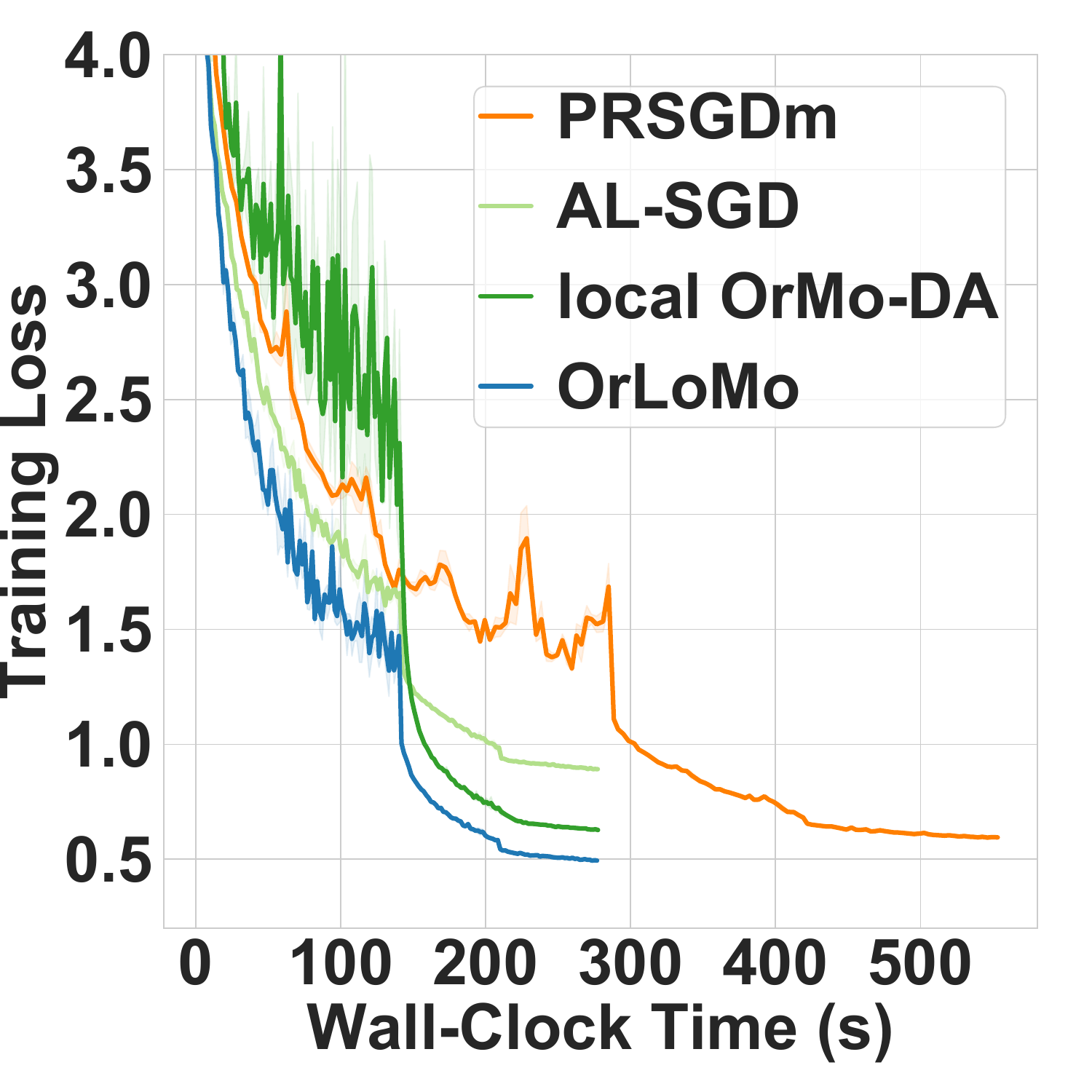}
      \end{minipage}}
\caption{Training loss curves of SqueezeNet model when $K=16, S=16$.}\label{fig:squeezenet-loss}
\end{figure*}

  \begin{table*}[!t] \small
  \centering
  \setlength{\tabcolsep}{1mm}{  
  \begin{tabular}{c|c|c|cccc|cccc}
    \toprule   
      \multicolumn{3}{c|}{Datasets} & \multicolumn{4}{c|}{Test Accuracy}& \multicolumn{4}{c}{Training Loss}  \\ 
    \midrule
     & Workers & Local Iterations & PRSGDm & AL-SGD &  local OrMo-DA & OrLoMo & PRSGDm & AL-SGD &  local OrMo-DA & OrLoMo \\
    \midrule
     \multirow{6}{*}{\rotatebox{90}{homogeneous}}&\multirow{3}{*}{8}& 8 & 53.64{\tiny$\pm$0.40} & 51.45{\tiny$\pm$0.23} & 53.59{\tiny$\pm$0.41} &\textbf{53.69{\tiny$\pm$0.19}} & 1.34{\tiny$\pm$0.01} & 1.60{\tiny$\pm$0.01}  & \textbf{1.30{\tiny$\pm$0.02}} & 1.33{\tiny$\pm$0.01}   \\ 
     & & 16 & 52.89{\tiny$\pm$0.22} & 50.88{\tiny$\pm$0.21} & \textbf{53.13{\tiny$\pm$0.13}} & 52.90{\tiny$\pm$0.13} & 1.46{\tiny$\pm$0.01} & 1.65{\tiny$\pm$0.01} & \textbf{1.38{\tiny$\pm$0.01}} & 1.45{\tiny$\pm$0.01}  \\
     & & 32 & 51.72{\tiny$\pm$0.30} & 50.99{\tiny$\pm$0.21} & 50.47{\tiny$\pm$0.14} & \textbf{51.89{\tiny$\pm$0.30}} & 1.59{\tiny$\pm$0.00} & 1.68{\tiny$\pm$0.01} & 1.64{\tiny$\pm$0.01} & \textbf{1.55{\tiny$\pm$0.01}} \\ \cmidrule(r){2-11} 
     &\multirow{3}{*}{16}& 8 & \textbf{52.75{\tiny$\pm$0.32}} & 47.43{\tiny$\pm$0.02} & 52.29{\tiny$\pm$0.43} & 52.66{\tiny$\pm$0.40} &  \textbf{1.46{\tiny$\pm$0.00}} & 1.89{\tiny$\pm$0.01} & 1.50{\tiny$\pm$0.00} & 1.46{\tiny$\pm$0.01} \\ 
     & & 16 & 50.71{\tiny$\pm$0.21} & 46.90{\tiny$\pm$0.24} & 49.23{\tiny$\pm$0.11} & \textbf{51.16{\tiny$\pm$0.20}} &  1.65{\tiny$\pm$0.01} & 1.94{\tiny$\pm$0.00} & 1.78{\tiny$\pm$0.02} & \textbf{1.63{\tiny$\pm$0.00}}  \\ 
     & & 32 & 48.26{\tiny$\pm$0.37} & 47.26{\tiny$\pm$0.08} & 40.47{\tiny$\pm$0.33} & \textbf{49.32{\tiny$\pm$0.41}} & 1.85{\tiny$\pm$0.02} & 1.94{\tiny$\pm$0.00} & 2.34{\tiny$\pm$0.02} & \textbf{1.78{\tiny$\pm$0.01}} \\ 
    \midrule 
     \multirow{6}{*}{\rotatebox{90}{heterogeneous}}&\multirow{3}{*}{8}& 8 & 53.65{\tiny$\pm$0.30} & 51.84{\tiny$\pm$0.15} &  \textbf{54.25{\tiny$\pm$0.15}} & 53.88{\tiny$\pm$0.04} & 1.34{\tiny$\pm$0.01} & 1.59{\tiny$\pm$0.01} & \textbf{1.26{\tiny$\pm$0.00}} & 1.32{\tiny$\pm$0.01}  \\ 
     && 16 & 52.12{\tiny$\pm$0.15} & 50.90{\tiny$\pm$0.29} & \textbf{53.25{\tiny$\pm$0.13}} & 52.89{\tiny$\pm$0.11}  & 1.49{\tiny$\pm$0.01} & 1.66{\tiny$\pm$0.00} & \textbf{1.35{\tiny$\pm$0.01}} & 1.44{\tiny$\pm$0.00}  \\
     && 32 & 51.92{\tiny$\pm$0.20} & 51.04{\tiny$\pm$0.31} & 51.92{\tiny$\pm$0.14} & \textbf{51.99{\tiny$\pm$0.22}} & 1.59{\tiny$\pm$0.01} & 1.68{\tiny$\pm$0.01} & 1.56{\tiny$\pm$0.02} & \textbf{1.56{\tiny$\pm$0.01}} \\ 
    \cmidrule(r){2-11}
     &\multirow{3}{*}{16}& 8 & \textbf{52.86{\tiny$\pm$0.29}} & 47.61{\tiny$\pm$0.24} & 52.05{\tiny$\pm$0.21} & 52.72{\tiny$\pm$0.21} & 1.47{\tiny$\pm$0.01} & 1.90{\tiny$\pm$0.01} & 1.49{\tiny$\pm$0.01} & \textbf{1.46{\tiny$\pm$0.01}} \\ 
     && 16 & 50.57{\tiny$\pm$0.19} & 47.56{\tiny$\pm$0.19} & 49.81{\tiny$\pm$0.27} & \textbf{51.18{\tiny$\pm$0.14}} & 1.65{\tiny$\pm$0.01} & 1.93{\tiny$\pm$0.01} & 1.74{\tiny$\pm$0.01} & \textbf{1.63{\tiny$\pm$0.01}} \\ 
     && 32 & 48.22{\tiny$\pm$0.17} & 46.88{\tiny$\pm$0.39} & 42.20{\tiny$\pm$0.69} & \textbf{49.35{\tiny$\pm$0.10}} & 1.86{\tiny$\pm$0.02} & 1.94{\tiny$\pm$0.01} & 2.23{\tiny$\pm$0.02} & \textbf{1.79{\tiny$\pm$0.01}} \\ 
    \bottomrule 
  \end{tabular}}
  \caption{Experimental results of ResNet32 model on Tiny-ImageNet dataset.} \label{table: resnet32-imagenet}
\end{table*}

\begin{figure*}[!t]
  \centering
  \subfigure[homogeneous]{
  \label{fig:aa}
      \includegraphics[width=0.23\linewidth]{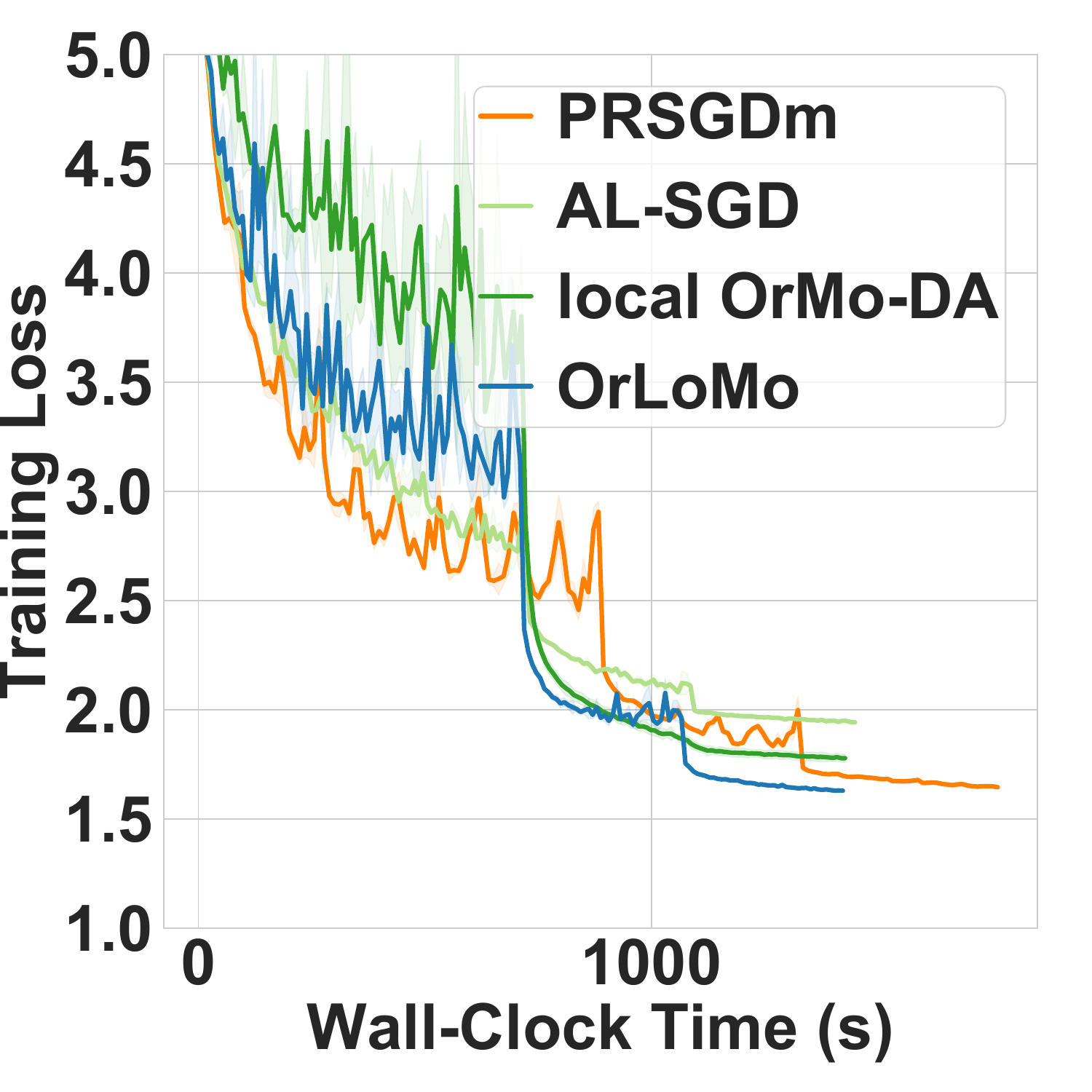}
      \includegraphics[width=0.23\linewidth]{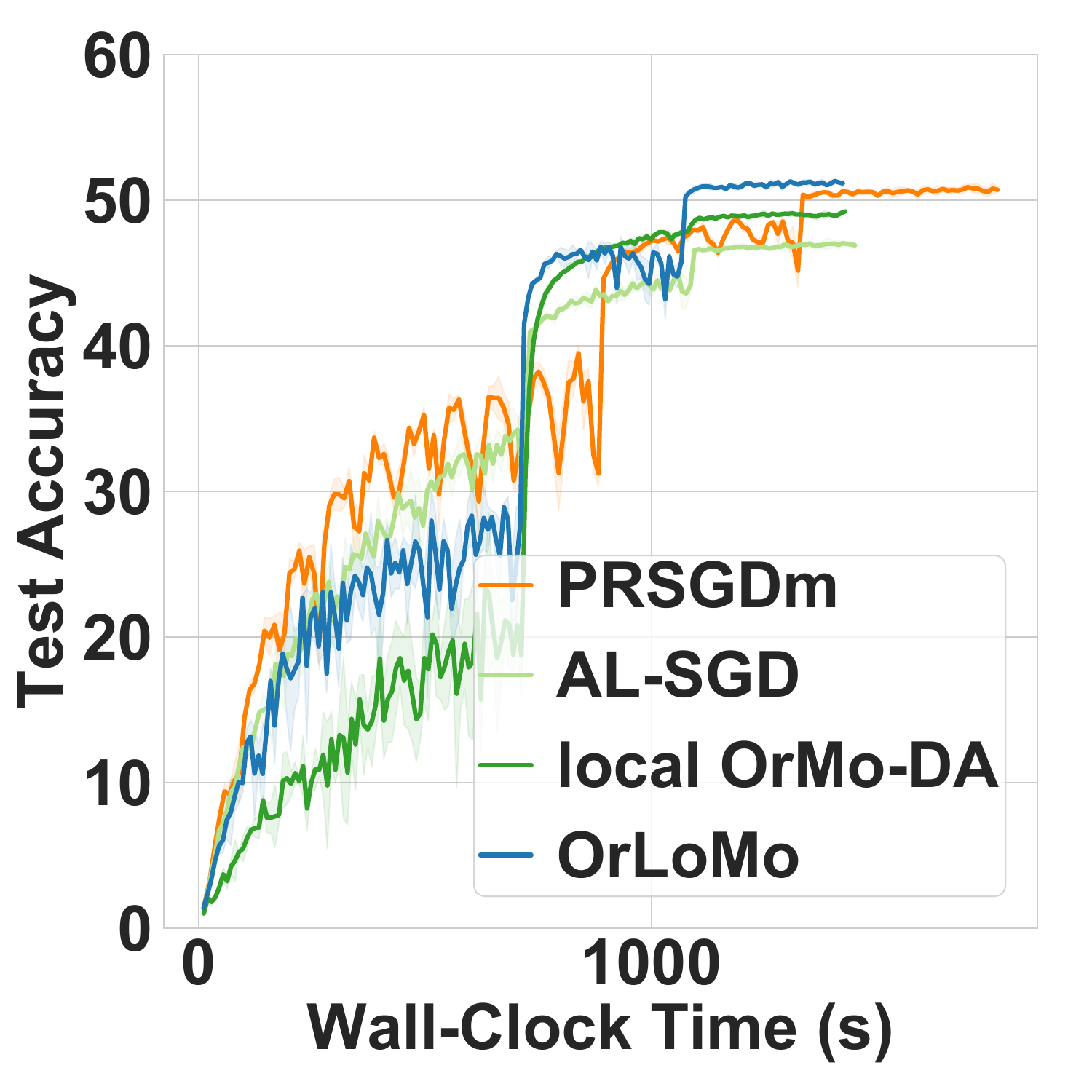}}
  \subfigure[heterogeneous]{
  \label{fig:bb}
      \includegraphics[width=0.23\linewidth]{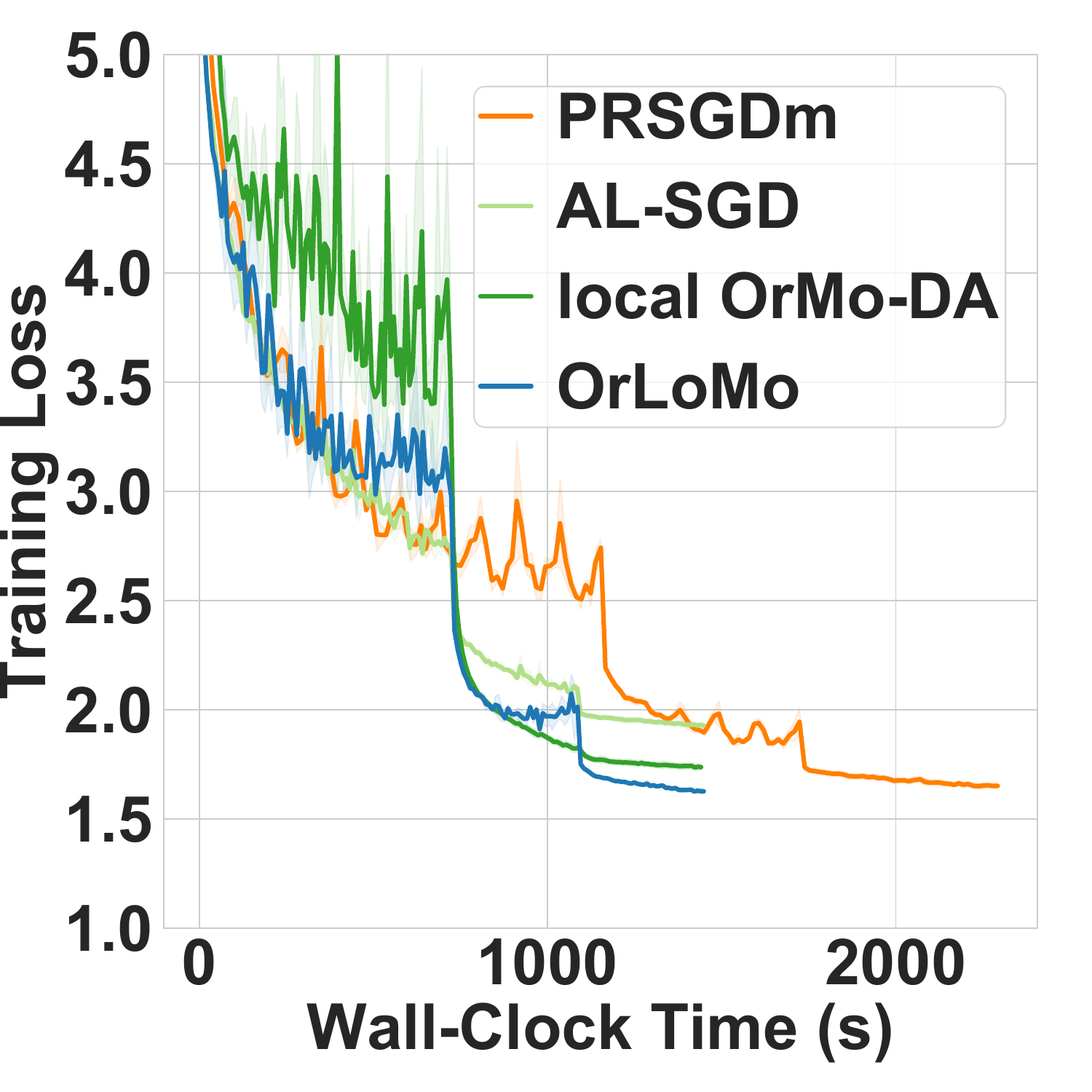}
      \includegraphics[width=0.23\linewidth]{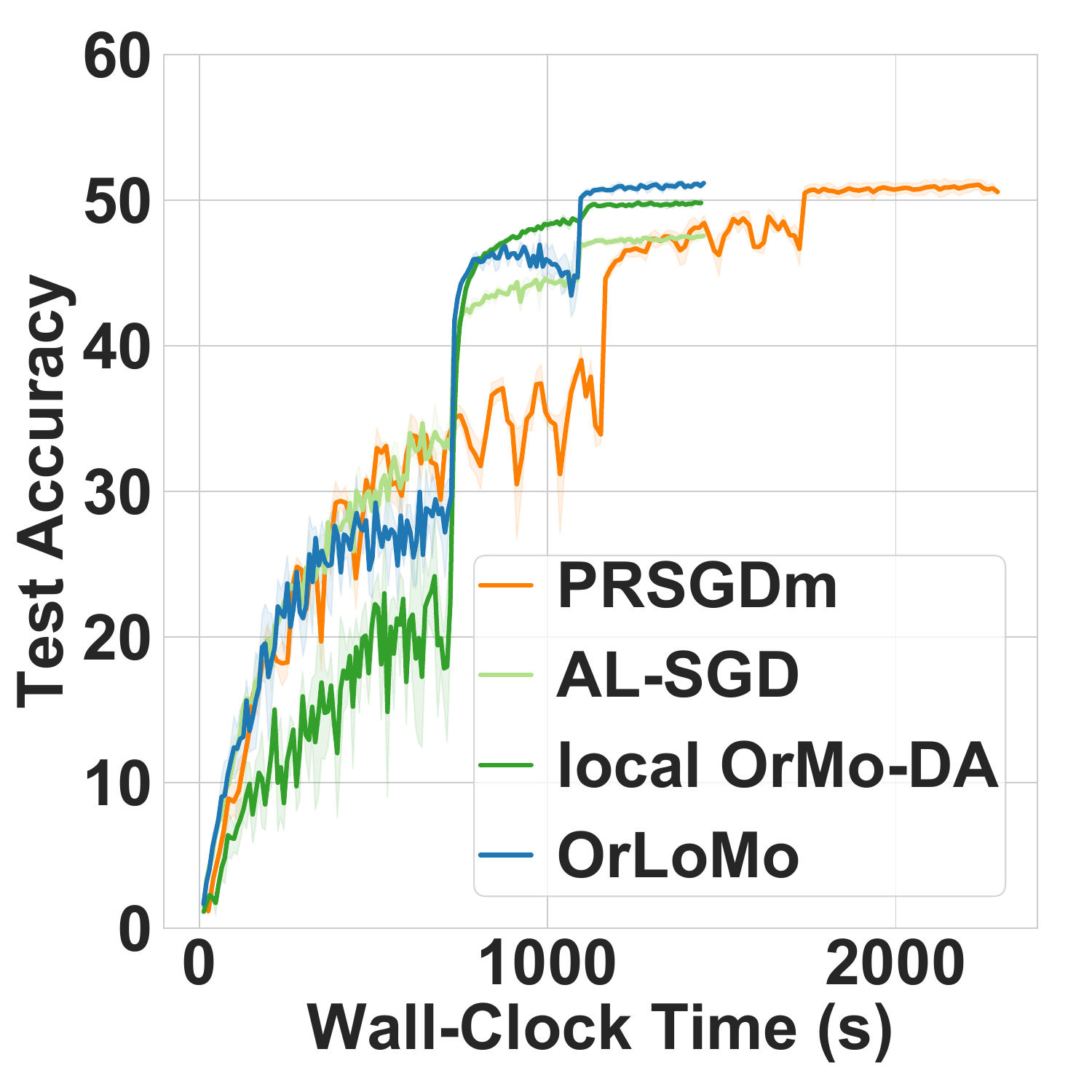}}
\caption{Experimental curves of ResNet32 model on Tiny-ImageNet dataset when $K=16, S=16$.}\label{fig:resnet32-loss}
\end{figure*}  
\subsection{Algorithm Details}
The details of local OrMo-DA are shown in Algorithm \ref{alg:local OrMo-DA}.
\begin{algorithm}[!t]
    \caption{local OrMo-DA}
    \begin{algorithmic}[1]
 \label{alg:local OrMo-DA}   
    \STATE \textbf{Input}: number of workers $K$, number of global iterations $T$, number of local iterations $S$, learning rate $\gamma$, momentum coefficient $\beta \in [0,1)$;
    \STATE \underline{\textbf{Server:}}  
    \STATE \textbf{Initialization}: parameter $\w_0$, momentum $\u_0 = \0$;  
    \STATE Send $\w_{0}$ to all workers;
    \FOR {$t=0$  \textbf{to}  $T-1$}
    \IF{$\lceil \frac{t}{K}\rceil > \lceil \frac{t-1}{K}\rceil$}
    \STATE $\w_{t+\frac{1}{2}} = \w_{t}-\beta \u_{t}$, $\u_{t+\frac{1}{2}}=\beta \u_t$;
    \ELSE 
    \STATE $\w_{t+\frac{1}{2}} = \w_t$, $\u_{t+\frac{1}{2}}= \u_t$; 
    \ENDIF    
    \STATE  Receive $\Delta \w_{ite (t,k_t)}^{k_t}$ from some worker $k_t$;
    \STATE $\u_{t+1} = \u_{t+\frac{1}{2}} + \beta^{\lceil \frac{t}{K}\rceil - \lceil \frac{ite (t,k_t)}{K}\rceil}\eta_t{\Delta \w_{ite (t,k_t)}^{k_t}}$;
    \STATE $\w_{t+1} = \w_{t+\frac{1}{2}} - \frac{1-\beta^{\lceil \frac{t}{K}\rceil - \lceil \frac{ite (t,k_t)}{K}\rceil+1}}{1-\beta} \eta_t {\Delta \w_{ite (t,k_t)}^{k_t}}$;
    \STATE Send $\w_{t+1}$ to the worker;
    \ENDFOR
    \STATE Notify all workers to stop;
    \STATE \underline{\textbf{Worker $k:$}} $(k \in [K])$  
    \REPEAT
    \STATE Wait until receiving $\w_{t'}$ from the server;
    \STATE Initialize local variables $\tilde{\w}_{t',0}^k=\w_{t'}$;  
    \FOR {$s=0$  \textbf{to}  $S-1$}
    \STATE Randomly sample $\xi \sim \DM$ and then compute the stochastic gradient $\g_{t',s}^{k} = \nabla f(\tilde{\w}_{t',s}^{k}; \xi)$;
    \STATE $\tilde{\w}_{t',s+1}^{k} = \tilde{\w}_{t',s}^{k} - \gamma \g_{t',s}^{k}$;   
    \ENDFOR    
    \STATE $\Delta \w_{t'}^k = \tilde{\w}_{t',0}^k- \tilde{\w}_{t',S}^k$;    
    \STATE Send $\Delta \w_{t'}^k$ to the server;  
    \UNTIL{receive server's notification to stop}
    \end{algorithmic}
    \end{algorithm}
\end{document}